\documentclass{article} 
\usepackage{iclr2026_conference,times}
\usepackage{macros}
\usepackage{hyperref}
\usepackage{url}
\usepackage[toc,page,header]{appendix}

\AddToHook{cmd/appendix/before}{
  \crefalias{section}{appendix}
  \crefalias{subsection}{appendix}
  \crefalias{subsubsection}{appendix}
}

\usepackage{minitoc}
\iclrfinalcopy

\title{Cross-Modal Redundancy and the Geometry of Vision–Language Embeddings}

\newcommand{\aut}[1]{\textbf{#1}}
\newcommand{\af}[1]{{\small #1}}

\newcommand{\afn}[1]{\textcolor{primary}{$^{#1}$}}

\author{
\aut{Grégoire Dhimoïla}{\afn{a,b,c}} \quad
\aut{Thomas Fel}{\afn{d}} \quad 
\aut{Victor Boutin}{\afn{e}} \quad 
\aut{Agustin Picard}\afn{c}
\vspace{4pt}\\
\afn{a}\af{Brown University} \quad 
\afn{b}\af{ENS Paris Saclay} \quad 
\afn{c}\af{IRT Saint Exupéry} \\ 
\afn{d}\af{Kempner Institute, Harvard University} \quad
\afn{e}\af{CNRS} \quad 
\vspace{3mm}\\
{
\small 
\href{lelien.fr}{\raisebox{-2.8pt}{\includegraphics[height=10pt]{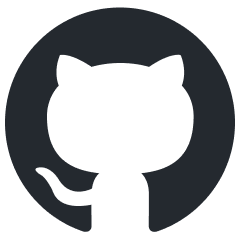}}
\texttt{\textcolor{secondary}{https://github.com/Parabrele/IsoEnergy}}}
}
\vspace{-4mm}
}

\begin{document}

\doparttoc %
\faketableofcontents %

\maketitle

\begin{abstract}

Vision–language models (VLMs) align images and text with remarkable success, yet the geometry of their shared embedding space remains poorly understood. 
To probe this geometry, we begin from the Iso-Energy Assumption, which exploits cross-modal redundancy: a concept that is truly shared should exhibit the same average energy across modalities.
We operationalize this assumption with an Aligned Sparse Autoencoder (SAE) that encourages energy consistency during training while preserving reconstruction.
We find that this inductive bias changes the SAE solution without harming reconstruction, giving us a representation that serves as a tool for geometric analysis.
Sanity checks on controlled data with known ground truth confirm that alignment improves when Iso-Energy holds and remains neutral when it does not.
Applied to foundational VLMs, our framework reveals a clear structure with practical consequences: 
\textbf{(\textit{i})} sparse \bimodal atoms carry the entire \crossmodal alignment signal; 
\textbf{(\textit{ii})} \unimodal atoms act as \modalityspecific biases and fully explain the modality gap; 
\textbf{(\textit{iii})} removing \unimodal atoms collapses the gap without harming performance; 
\textbf{(\textit{iv})} restricting vector arithmetic to the \bimodal subspace yields in-distribution edits and improved retrieval. 
These findings suggest that the right inductive bias can both preserve model fidelity and render the latent geometry interpretable and actionable.
\end{abstract}

\section{Introduction}

\let\thefootnote\relax\footnote{\textit{Correspondence to:} \textcolor{secondary}{\texttt{gregoire.dhimoila@ens-paris-saclay.fr}}}

Vision-language models (VLMs)~\citep{radford2021,zhai2023sigmoid,fini2024multimodal,tschannen2025siglip2} have become central to applications from visual question answering~\citep{chen2024spatialvlm} and medical imaging~\citep{singhal2023large} to autonomous driving~\citep{zhou2024vision} and embodied AI~\citep{shukor2025smolvla}, creating shared embedding spaces where visual and textual representations of similar concepts align. Despite their empirical success, we lack a principled understanding of how these models internally organize and align semantic content across modalities.
This work investigates the geometric structure of cross-modal embeddings: \emph{How do VLMs organize the alignment between visual and textual semantics, and what principles govern this shared representational space?} Understanding these mechanisms is crucial for designing more robust and interpretable vision-language architectures.

\paragraph{From Attributions to Concepts.} To address this, the interpretability community has developed an array of tools~\citep{gilpin2018explaining,bau2017network} aimed at dissecting learned representations. Early efforts centered on attribution methods~\citep{zeiler2013visualizing,sundararajan2017axiomatic,smilkov2017smoothgrad,petsiuk2018rise,fel2021sobol} that highlight ``where'' a model focuses its attention, but these approaches often fall short~\citep{nguyen2021effectiveness,kim2021hive,fel2021cannot,hase2020evaluating,sixt2022users} of explaining ``what'' abstractions and on which basis the model operates.
More recent concept-based methods~\citep{ghorbani2019towards,zhang2021invertible,fel2023craft,elhage2022superposition,fel2023holistic} have emerged to answer this, seeking to extract meaningful features (concepts) that models implicitly compute over. 
Concept extraction is typically framed as a dictionary learning problem~\citep{tovsic2011dictionary,rubinstein2010dictionaries,elad2010sparse,mairal2014sparse,dumitrescu2018dictionary}: identifying an overcomplete set of basis vectors that explain internal activations via sparse coding~\citep{olshausen1996emergence, olshausen1997sparse, lee2006efficient,foldiak2008sparse,rentzeperis2023beyond}. This approach is a direct response to an underlying phenomenology of the activation space: the Linear Representation Hypothesis (LRH)~\citep{elhage2022superposition,wattenberg2024relational} which posits that model activations can be viewed as sparse combinations of latent directions drawn from a high-dimensional concept basis. 
This is motivated not only by empirical findings, but also by geometric arguments: in high-dimensional spaces, sparse sets of nearly orthogonal directions can represent exponentially many distinct concepts with minimal interference, a principle reminiscent of Johnson-Lindenstrauss-style~\citep{johnson1984extensions,larsen2017optimality} embeddings.
Sparse autoencoders (SAEs)~\citep{makhzani2013k,elhage2022superposition} directly operationalize this phenomenology by learning an approximate inverse mapping from model activations to latent conceptual directions. They recover a sparse code that identifies which overcomplete basis elements (concepts) are active in a given representation, and have been effective at uncovering semantically meaningful structure in both vision~\citep{gorton2024missing,fel2025archetypal,dreyer2025mechanistic,rao2024discover} and language models~\citep{cunningham2023sparse, bricken2023monosemanticity, rajamanoharan2024jumping, gao2024scaling, surkov2024unpacking}.
\vspace{-2mm}
\paragraph{Multimodal Interpretability.} However, applying SAEs to VLMs~\citep{bhalla2024interpreting,pach2025sparse} reveals a puzzling behavior: concept dictionaries often segregate by modality~\citep{papadimitriou2025}. Although VLMs are trained for cross-modal alignment, the extracted concepts tend to activate exclusively for either image or text inputs. This observation is the concept-based view of a now well-documented phenomenon called the {\it modality gap}. 
Prior work has described this separation geometrically, attributing it to a conical structure in the embedding space \citep{liang2022mind, ethayarajh2019}, or through the lens of training dynamics induced by the contrastive loss \citep{fahim2024s,shi2023,yaras2024}. Yet these accounts do not explain what such separation means at the level of shared concepts.

In this work, we introduce a framework for analyzing multimodal representations grounded in an explicit generative model. Central to this framework is the \textit{Iso-Energy Assumption} -- if a concept is truly shared across modalities, it should exhibit invariant energy, defined as the average squared activation, regardless of input domain. This assumption provides a concrete, testable criterion for identifying \bimodal concepts.
We then operationalize this simple assumption with a natural method: an alignment-penalized Matching Pursuit Sparse Autoencoder (Aligned SAE), which encourages energy consistency across modalities during training. This approach allows us to diagnose whether extracted concepts genuinely support cross-modal alignment or merely reflect modality-specific patterns.
This work makes the following contributions:  

\begin{itemize}[leftmargin=*, nosep]
  \item We introduce the \emph{Iso-Energy Assumption}, which exploits cross-modal redundancy by requiring that shared concepts exhibit the same average activation energy in image and text.  
  \item We operationalize this assumption with an Aligned Sparse Autoencoder that enforces energy consistency during training while preserving reconstruction, and validate it with ground-truth sanity checks where classical SAEs fail.  
  \item Applied to dual-encoder vision–language foundation models, this inductive bias reveals a geometric decomposition invisible to classical SAEs: (\textbf{\textit{i}}) sparse \bimodal atoms carry the entire \crossmodal alignment signal, while (\textbf{\textit{ii}}) \unimodal atoms carry \modalityspecific information and fully explain the modality gap, with a few high energy atoms acting as \modalityspecific biases.
  \item Moreover, our work reveals that cross-modal information is carried by shared atoms, as opposed to idiosyncratic ones described by \citet{papadimitriou2025}.
  \item This structure enables actionable interventions without loss of performance: (\textbf{\textit{iii}}) removing \unimodal atoms collapses the modality gap, and (\textbf{\textit{iv}}) restricting vector arithmetic to the \bimodal subspace yields in-distribution edits.%
\end{itemize}

At the core lies a single theoretical premise: if concepts are genuinely shared across modalities, they must imprint redundant statistical traces in each domain. We formalize this intuition by modeling multimodal representations as partial inverses of a shared generative process, and introducing the \emph{Iso-Energy Assumption} as the inductive bias that makes concept recovery feasible.

\paragraph{Nomenclature.} We distinguish the following terms when referring to concepts based on their behavior and based on the type of information they carry. \textbf{\textit{{(i) Activation patterns:}}} \textit{(a) unimodal} concepts activate exclusively on a single modality, whereas \textit{(b) bimodal} concepts activate on both. This distinction is made by thresholding the modality score $\mu$, a comparison of domain-wise energy (see \cref{app:modality_metrics}). \textbf{\textit{{(ii) Information carried:}}} \textit{(c) modality-specific} concepts carry information specific to a given modality, e.g. that an image has cropping artifacts, while \textit{(d) cross-modal} concepts carry cross-modal information and participate in the contrastive geometric alignment between modalities.

\vspace{-2mm}
\section{Related Work}
\label{sec:RW}

\vspace{-3mm}
Previous work extensively describes salient phenomena in the shared space of multimodal dual-encoders \citep{schrodi2025two,levi2025the,jiang2023understanding,udandarao2022understanding}. Most notably, \citet{liang2022mind} describes what is now commonly known as the \emph{modality-gap}: that image and text embeddings reside in disjoint cones in the latent space. This modality gap has been attributed to a cone effect \citep{liang2022mind,ethayarajh2019,schrodi2025two} and training dynamics induced by the contrastive loss and by mismatched pairs of data \citep{fahim2024s,shi2023,yaras2024}. \citep{levi2025the} additionally shows that embeddings are contained near the surface of ellipsoid empty shells centered near the mean of the two distributions $\mu_I$ and $\mu_T$.

The \emph{cone effect} naturally comes with a salient difference in modality wise means $\Delta\coloneqq\mu_I-\mu_T$ \citep{levi2025the,liang2022mind,fahim2024s}, or a perfect linear separability between image and text embeddings \citep{schrodi2025two,levi2025the,fahim2024s,shi2023}. Furthermore, \citet{schrodi2025two} shows that, surprisingly, a small subset of coordinates in the canonical latent basis accounts for most of the norm of $\Delta$.
Previous work tries to get rid of the modality gap by shifting modality-wise means \citep{liang2022mind,levi2025the}, or by projecting out the few canonical directions mentioned above \citep{schrodi2025two}. In all of these works, the proposed intervention decreases cross-modal performance of embeddings, with the notable exception of \citet{zhang2023diagnosing}, whose Proposition 1 precedes our more general \cref{prop:modality_removal}.
\Cref{fig:removing_gap_main} shows that this difference in means, while explaining the bulk of the modality gap and linear separability between the two sets of embeddings, is not enough to explain the full distributional mismatch. For that, we need to account for modality-specific information.
 \vspace{-1mm}

\citet{jiang2023understanding} takes an information-theoretic angle to show that modality specific variability in representations is necessary to preserve unimodal capabilities. They use this insight to design a new architecture that would explicitely regularise for the representation of such information. We show that (\emph{i}) unregularised foundational VLM encoders possess such representations, and that (\emph{ii}) the underlying mechanism is to organise these types of information linearly in distinct subspaces---let us call $\Omega_{I}$ and $\Omega_{T}$ the subspaces containing modality specific information, and $\Gamma$ the one containing shared information. The projection by \citet{schrodi2025two} described above, though not introduced for this purpose, is a first attempt at identifying and intervening on $\Omega_I \oplus\Omega_T$, characterised by the span of the canonical basis vector selected, while $\Gamma$ would be the orthogonal complement. However, their results are negative, as their intervention is detrimental to model performance even on cross-modal focused tasks---indicating significant unintended alteration of $\Gamma$---while leaving the gap wide open (see \cref{app:ilestsinuljpp}), thus showing that their identified directions can't describe $\Omega_I \oplus\Omega_T$.

Our novelty lies not in the description of most phenomena discussed above, but rather in the characterisation of these structures in foundational VLM encoders.
Through the use of a concept-based approach, we are able to identify (\emph{i}) high-energy unimodal features to unimodal biases, (\emph{ii}) $\Gamma = \mathrm{Cone}(\bm{\delta} \cdot \bm{D})$ and $\Omega_{I/T} = \mathrm{Cone}(\bm{\delta}_{I/T}\cdot \bm{D})$. Here, $\mathrm{Cone}(\bm{e}_1,\ldots,\bm{e}_k)=\{\sum_{i=1}^k\lambda_i \bm{e}_i \mid \lambda_i \geq 0\}$, $\bm{D} \in \R^{K\times d}$ is the set of $K$ dictionary atoms of latent dimension $d$, and $\bm{\delta}$ (resp. $\bm{\delta}_I$, $\bm{\delta}_T$) $\in \{0,1\}^d$ is the binary mask selecting bimodal (resp. image-, text-only) atoms. To claim this characterisation, we carefully analyse four complementary aspects of the dictionary through novel metrics. Once validated, we further test its practical value through targeted interventions on these structures.

\vspace{-2mm}
\section{Exploiting Cross-Modal Redundancy for Concept Recovery}

\begin{figure}[t]
  \centering
  \vspace{-12mm}
  \includegraphics[width=0.9\textwidth]{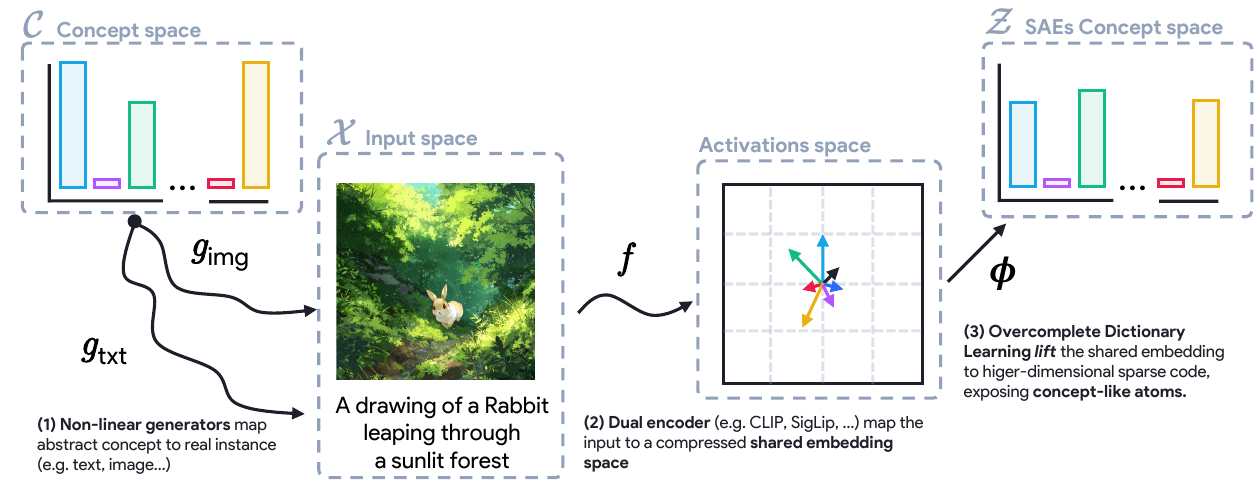}
  \caption{
  \textbf{Multimodal data-generating process.}  
    A latent concept vector $\cvec \in \Cspace$ (e.g., \emph{rabbit}, \emph{forest}, \emph{light}, \emph{running}) is sampled as a sparse combination of abstract concepts and rendered through domain-specific generators $\bm{g}(\cdot)$ (e.g., image or text). Dual-encoder models (e.g., CLIP, SigLIP) map these observations to a shared activation space, which sparse autoencoders (or other overcomplete dictionary learning methods) then attempt to \textit{lift} back to concept-like atoms. 
    However, without additional inductive bias, encoder-decoder pairs $(\bm{f}, \bm{\phi})$ are not uniquely determined, a well-known identifiability problem in nonlinear ICA. Here we leverage cross-modal redundancy as a useful inductive bias, nudging the solution toward recovering \bimodal concepts.}
  \label{fig:dgp}
  \vspace{-2mm}
\end{figure}

We formalize the setting by modeling each datum (image, caption, \dots) as generated from a latent concept vector and a domain-specific generator. From this perspective, encoders serve as partial inverses of a shared multimodal generative process. %
However, recovering the underlying concepts is in general an ill-posed inverse problem: without additional inductive biases, nonlinear ICA is provably unidentifiable~\citep{hyvarinen2016unsupervised,locatello2019challenging,khemakhem2020variational}---infinitely many mappings can account for the observed data in the absence of cross-domain constraints.
Introducing Iso-Energy provides a selection principle among these solutions and yields a representation that is useful for studying the geometry observed in practice.
We begin by formally introducing our data generative process: each datum arises by (\textbf{\textit{i}}) sampling a sparse concept vector and (\textbf{\textit{ii}}) rendering it through a domain-specific generator.
Formally, let $\K\!\in\!\mathbb{N}$ be the number of latent concepts and $\Cspace\!\subseteq\!\mathbb{R}^{\K}$ the concept space.
Let $\Domains=\{\domain_1,\ldots,\domain_{|\Domains|}\}$ index domains with observation spaces $\Xd$.
\begin{definition}[Multimodal Concept Generative Process]
\label{def:generative_process}
Sample $\cvec \sim \prior=\prod_{k=1}^{\K}\marginal$ with $|\mathrm{supp}(\cvec)|\!\ll\!\K$.
For each $\domain\!\in\!\Domains$, a deterministic generator $\generator:\Cspace\!\to\!\Xd$ yields $\xvec^{(\domain)}=\generator(\cvec)$.
We admit that each $\generator$ is $C^1$ and locally invertible.\footnote{for any $\cvec$ there exists a neighborhood $U(\cvec)$ on which $\generator$ is injective with continuous inverse.}
\end{definition}
A VLM encoder $\bm{f}$ maps each observation $\bm{x}^{(\domain)}$ to a shared embedding, and a sparse autoencoder $\bm{\phi}$ attempts to \emph{disentangle} this embedding back into concept coordinates, so that $\bm{\phi} \circ \bm{f}$ approximates the inverse of \Cref{def:generative_process}. 
However, recovering latent concepts from nonlinear generators is not identifiable in general: without additional structure, many different dictionaries can explain the same data~\citep{hyvarinen2016unsupervised,khemakhem2020variational,locatello2019challenging}. 
This ambiguity is visible in practice, where SAE dictionaries often vary substantially across random seeds~\citep{fel2025archetypal}. 
Domain membership provides auxiliary information, but by itself it is insufficient to ensure consistent recovery. 
We therefore introduce a cross-domain simple inductive bias, the \emph{Iso-Energy Assumption}, which states that genuinely multimodal concepts should maintain consistent average energy across modalities. 
Because such concepts manifest in parallel across domains, their observations contain redundant signals that can be exploited to guide dictionary recovery toward more stable and plausible solutions. Formally, 
\begin{definition}[Iso-Energy Assumption]\label{def:isoenergy}
Let $\encoder:\bigcup_{\domain\in\Domains}\Xd\to\mathbb{R}^{\K}$ denote the learned encoder (VLM $\bm{f}$ composed with SAE $\bm{\phi}$).
We say $\encoder$ satisfies \emph{Iso-Energy} if the second moment of each coordinate is domain-invariant:
\[
\mathop{\mathbb{E}}\limits_{\Xvec \in \Xd}\!\left(\encoder(\Xvec)_k^2\right)
=
\mathop{\mathbb{E}}\limits_{\Xvec \in \mathcal{X}^{(\domain')}}\!\left(\encoder(\Xvec)_k^2 \right),
~~
\text{with}~ k \in \llbracket 1,K \rrbracket,~~(\domain,\domain')\in \Domains^2.
\]
\end{definition}
Iso-Energy provides a testable, domain-agnostic inductive bias: \crossmodal features should correspond to \bimodal concepts and maintain comparable energy across domains, whereas \modalityspecific factors need not.
This constraint narrows the solution set without requiring instance-level matching.
It is reminiscent of the rosetta neurons \citep{dravid2023rosetta,gresele2020incomplete} and the platonic representation hypothesis \citep{pmlr-v235-huh24a}, suggesting that independent models (in our case, the vision and language encoders) converge on shared features.

\paragraph{Operationalization.}  
To operationalize this assumption, we adopt as our base the Matching Pursuit (MP) sparse autoencoder, recently introduced by \citet{costa2025flat} and rooted in the original framework of~\citet{mallat1993matching}. MP enforces $\ell_0$ sparsity through sequential residual updates, which aligns with the sparse generative model introduced above and has shown strong empirical reconstruction performance in vision compared to ReLU~\citep{bricken2023monosemanticity}, JumpReLU~\citep{rajamanoharan2024jumping}, or BatchTopK SAEs~\citep{gao2024scaling,bussmann2024batchtopk} (see \cref{app:SAE_architecture} for their formal definition).
We then incorporate Iso-Energy into the sparse autoencoder recipe as a soft regularizer: a small penalty encouraging the activations of the same atom to maintain similar strength across domains. Given $\ell_2$-normalized codes $\Z^{(\domain)},\Z^{(\domain')}\in\mathbb{R}^{b\times K}$, our training loss becomes:
\begin{equation}
\label{eq:alignment_penalty}
\mathcal{L}_{\mathrm{\SAEA{SAE-A}}}= \mathcal{L}_{\mathrm{\SAE{SAE}}}+ \beta \cdot\SAEA{\mathcal{L}_{\mathrm{align}}} 
\quad \text{with} \quad 
\SAEA{\mathcal{L}_{\mathrm{align}}} = -\tfrac{1}{b}\,\mathrm{Tr}\!\left(\Z^{(\domain)} \,\Z^{(\domain')\top}\right)
\end{equation}
with $\mathcal{L}_{\mathrm{\SAE{SAE}}}$ being the training loss of a standard \SAE{SAE}, typically an $\ell_2$ norm measuring reconstruction error with a sparsity constraint on the encoder's output. $\SAEA{\mathcal{L}_{\mathrm{align}}}$ is the soft inductive bias,
where $b$ is the batch size and $K$ the number of latent atoms. With a small weight ($\beta \approx 10^{-4}$, see \cref{app:beta_tuning}), this regularizer gently biases the dictionary toward \bimodal features while preserving reconstruction performance. %
In Section~\ref{sec:results}, we demonstrate that this bias does not force the model to create multimodal concepts that don't exist in the input data. It is important to note that the minimum of this loss function is consistent with the iso-energy principle for cross-modal concepts formalized in \cref{def:isoenergy}: maximizing the cosine similarity of codes coming from the aligned samples from the different modalities leads to codes with the same energy across modalities.
From this point onward, we denote the unregularized model as \SAE{SAE}, and its alignment-augmented counterpart as \SAEA{SAE-A}. %

\section{Recovering Multimodal Structure with Aligned SAEs}
\label{sec:results}

\paragraph{Sanity check.}  
Before turning to large-scale embeddings, we validate our approach on controlled toy data with known ground truth. We construct two synthetic data-generating processes that mimic CLIP-like cosine similarity statistics (\cref{fig:toy_cosim}) while giving us exact control over which atoms are \unimodal and which are \bimodal (\cref{app:DGP_def}). Each sample is generated by drawing a sparse code \(\bm{z}\) with \(\|\bm{z}\|_{0}=L=20\) and producing normalized embeddings
\[
\bm{x}^{(\mu)} = \bm{z}^{\top} \bm{D}^{(\mu)}, \qquad 
\bm{x}^{(\nu)} = \bm{z}^{\top} \bm{D}^{(\nu)}, \qquad 
\|\bm{x}^{(\mu)}\|_{2}=\|\bm{x}^{(\nu)}\|_{2}=1 .
\]
Two parameters govern the process. First, \(\tau_{1}\) sets the cross-modal alignment of each shared (\bimodal) atom \(k\in B\) by fixing the cosine between its per-modality components. Second, \(\tau_{2}\) fixes the average paired image–text similarity at the embedding level:
\[
\cos\!\angle\!\big(\bm{d}_{k}^{(\mu)},\bm{d}_{k}^{(\nu)}\big)
=\frac{\langle \bm{d}_{k}^{(\mu)},\bm{d}_{k}^{(\nu)}\rangle}{\|\bm{d}_{k}^{(\mu)}\|_{2}\,\|\bm{d}_{k}^{(\nu)}\|_{2}}
=\tau_{1}, \quad \text{and} \quad \mathbb{E}\!\left[\left\langle \bm{x}^{(\mu)}, \bm{x}^{(\nu)} \right\rangle\right]=\tau_{2}.
\]
with the Iso-Energy case corresponding to \(\tau_{1}=1\) (the bimodal atom is identical across modalities up to scale).  

Using this approach, we can generate a distribution that matches the alignment statistics found in CLIP embeddings. We then compare a standard Sparse Autoencoder (SAE) to its "aligned" variant trained with the loss function in Eq.~\ref{eq:alignment_penalty}. This new loss encourages energy consistency without compromising reconstruction quality, as indicated by the high R-squared value (\(R^{2}\geq 0.99\)) we observed in all our experiments.
Recovery is evaluated against \((\bm{D}^{\star}, \bm{Z}^{\star})\) using the Wasserstein distance \(\mathcal{W}\) between learned and true atoms (lower is better) and mean matching accuracy (mma) of usage patterns after optimal bipartite matching (higher is better). When Iso-Energy is violated (\(\tau_{1}\neq 1\)), both \SAE{SAE} and \SAEA{SAE-A} recover the dictionary equally well (\(\mathcal{W}\approx 0.19,\ \mathrm{mma}\approx 0.82\)), confirming that the regularizer does not hallucinate \bimodal atoms. When Iso-Energy holds (\(\tau_{1}=1\)), the standard \SAE{SAE} fails (\(\mathcal{W}=0.396,\ \mathrm{mma}=0.29\)) while the aligned \SAEA{SAE-A} succeeds (\(\mathcal{W}=0.184,\ \mathrm{mma}=0.52\)), showing that the inductive bias is neutral when unnecessary and decisive when appropriate. Full experimental details are provided in \cref{app:DGP_experiments}. Having validated the principle in controlled settings, we next study its behavior on embeddings from foundation-scale VLMs.

\subsection{Empirical Evaluation on Vision–Language Foundations}

\newcommand{\best}[1]{\textbf{#1}}
\newcommand{\separator}{{~}}

\begin{table*}[t]
  \centering
  \vspace{-8mm}
  \scalebox{0.9}{
  \begin{tabular}{lcccccc}
    \toprule
    Metric & CLIP & CLIP-L & OpenCLIP & OpenCLIP-L & SigLIP & SigLIP2 \\
    \midrule
    MSE ($\downarrow$)   
      & \best{0.141} \separator 0.163 
      & \best{0.207} \separator 0.213 
      & \best{0.246} \separator 0.257 
      & \best{0.244} \separator 0.253 
      & \best{0.212} \separator 0.214 
      & \best{0.115} \separator 0.115 \\
    $R^2$ ($\uparrow$)   
      & \best{0.859} \separator 0.837 
      & \best{0.793} \separator 0.787 
      & \best{0.754} \separator 0.742 
      & \best{0.755} \separator 0.747 
      & \best{0.788} \separator 0.784 
      & 0.884 \separator \best{0.885} \\
    \midrule 
    $p_{\mathrm{acc}}$ ($\uparrow$) 
      & 0.847 \separator \best{0.915} 
      & 0.843 \separator \best{0.868} 
      & 0.849 \separator \best{0.880} 
      & 0.845 \separator \best{0.873} 
      & 0.897 \separator \best{0.899} 
      & 0.886 \separator \best{0.903} \\
    $\rho$ ($\uparrow$)  
      & 0.327 \separator \best{4.232} 
      & 1.566 \separator \best{4.086} 
      & 4.072 \separator \best{16.02} 
      & 8.737 \separator \best{16.58} 
      & 1.370 \separator \best{2.182} 
      & 0.713 \separator \best{1.475} \\
    FDA ($\uparrow$)     
      & 2.630 \separator \best{4.559} 
      & 3.914 \separator \best{4.800} 
      & 4.369 \separator \best{8.160} 
      & 9.787 \separator \best{16.49} 
      & 8.831 \separator \best{34.95} 
      & 8.246 \separator \best{18.24} \\
    $\delta_{\mathrm{r}}$ ($\downarrow$) 
      & 0.224 \separator \best{0.125} 
      & 0.039 \separator \best{0.021} 
      & 0.037 \separator \best{0.018} 
      & 0.001 \separator \best{-0.000} 
      & 0.023 \separator \best{0.006} 
      & 0.007 \separator \best{-0.006} \\
    \bottomrule
  \end{tabular}}
  \caption{\textbf{Comparison of unregularized vs.\ aligned sparse autoencoders across six VLMs.}  
We report reconstruction fidelity (MSE, $R^2$) and multimodality-sensitive metrics: probing accuracy ($p_{\mathrm{acc}}$), functional alignment ($\rho$), Functional and Distributional Agreement (FDA), and interventional robustness ($\delta_{\mathrm{r}}$). Left values correspond to \SAE{SAE}, right values to \SAEA{SAE-A}. While reconstruction is nearly identical, the aligned variant consistently improves on all multimodality metrics : $\rho$ increases by more than an order of magnitude, FDA doubles (or triples), and $\delta_{\mathrm{r}}$ remains near zero, demonstrating that bimodal atoms alone sustain cross-modal alignment while unimodal ones contribute little beyond modality specific bias.}
  \label{tab:metrics}
  \vspace{-4mm}
\end{table*}

\paragraph{Setup.} We train both \SAE{SAE} and its aligned counterpart \SAEA{SAE-A} on activations from six representative models: CLIP (ViT-B/32, ViT-L/14) \citep{radford2021}, OpenCLIP, OpenCLIP-L \citep{cherti2023reproducible}, SigLIP \citep{zhai2023sigmoid}, and SigLIP2 \citep{tschannen2025siglip2} (see \cref{app:models} for more details on these models). All SAE models are trained with identical hyperparameters (expansion ratio 8, target $\ell_0 = 20$) using a subset of 1 million LAION embeddings chosen at random, ensuring that observed differences are attributable to the Iso-Energy regularizer rather than training artifacts.  
Classical reconstruction metrics (MSE, $R^2$) reveal little difference between the two methods, confirming that the alignment penalty does not compromise fidelity. Yet reconstruction alone is uninformative: two dictionaries with identical $R^2$ may encode radically different concept structures. To capture these distinctions, we introduce a suite of multimodality-sensitive metrics, each designed to probe a complementary aspect of the recovered dictionary. For the full definitions of these metrics, we refer the reader to \cref{app:metrics}.  
\vspace{-3mm}
\paragraph{Tuning $\beta$.}
In \Cref{eq:alignment_penalty}, we select $\beta$  via a log-sweep over $\{10^{-6},\ldots,10^{-1}\}$ and pick the largest value such that the difference in explained variance compared to non regularised \SAE{SAE} is less than $0.05$. This rule reproduces our settings without hand-tuning and, under this criterion, we also observe no degenerate (always-on) features in practice; see \cref{app:beta_tuning} for details.

\paragraph{Metrics:}\hfill

\textit{(i) Probing accuracy $p_{\mathrm{acc}}$} (\cref{metric:pacc}). This metric tests whether the geometry of dictionary atoms reflects the modality structure of the embedding space. \textit{Unimodal} atoms should act as strong linear classifiers for domain membership, while \bimodal atoms should remain domain-agnostic. $p_{\mathrm{acc}}$ gathers all these classifiers' performance in a single scalar. High $p_{\mathrm{acc}}$ therefore indicates that the dictionary correctly distinguishes modality-specific and shared information. %

\textit{(ii) Functional alignment $\rho$} (\cref{metric:rho}). Beyond geometry, we ask which features actually drive cross-modal alignment. $\rho$ measures the ratio of alignment explained by \bimodal versus \unimodal features through instance-level co-activation patterns. Values of $\rho > 1$ indicate that alignment is predominantly supported by \bimodal concepts, consistent with the Iso-Energy Assumption. 

\textit{(iii) Functional and Distributional Agreement (FDA)} (\cref{metric:FDA}). At the population level, FDA measures whether the functional role of features is consistent with their geometric organization. While $\rho$ is local and instance-based, FDA checks distributional alignment across large batches. Intuitively, high FDA confirms that \bimodal atoms globally sustain the match between modalities. Neither $\rho$ nor FDA considers the contrastive aspect of alignment, which is done by our final metric.

\textit{(iv) Interventional robustness $\delta_{\mathrm{r}}$} (\cref{metric:deltar}). To test for causality, we measure how performance in retrieval tasks changes when we remove \unimodal features. We use $\delta_{\mathrm{r}}$ to quantify the change in recall after ablating these \unimodal atoms. A small $\delta_{\mathrm{r}}$ suggests that these atoms aren't essential, while a large drop indicates they are necessary for the contrastive aspect of the embeddings.%
\paragraph{Results.}As shown in Table~\ref{tab:metrics}, \SAEA{SAE-A} matches \SAE{SAE}'s MSE and $R^2$ scores but consistently outperforms it on all multimodality-sensitive metrics. In particular, $\rho$ increases by more than an order of magnitude across models, indicating that alignment is almost entirely driven by \bimodal atoms. Likewise, FDA improves substantially, confirming consistency at the distributional level. Probing accuracy improves modestly, reflecting a more distinct geometric separation, while $\delta_{\mathrm{r}}$ remains small, demonstrating that \unimodal atoms can be safely ablated without harming retrieval performance and therefore do not contain cross-modal information.
Taken together, these results offer converging evidence that, in \SAEA{SAE-A}, \unimodal atoms encode modality-specific information, while \bimodal atoms, identified via Iso-Energy, constitute the principal basis for cross-modal alignment. In \SAE{SAE}, however, the picture is less clear, and some \unimodal atoms carry \crossmodal information.

These findings indicate that the Iso-Energy Assumption reveals a qualitatively distinct structure: a compact \crossmodal subspace, entirely spanned by \bimodal concepts, that retains the contrastive power of the original embeddings while fully supporting cross-modal alignment.
As we will see in \cref{sec:actionability}, making this structure explicit also makes it actionable: it allows us to manipulate representations directly, from closing the modality gap to performing in-distribution semantic arithmetic. But before making this structure actionable, we propose an in-depth analysis of the solution to understand some aspects of the geometry.

\subsection{Concept Geometry Under Iso-Energy}  
\begin{figure}[t]
  \centering
  \includegraphics[width=\linewidth]{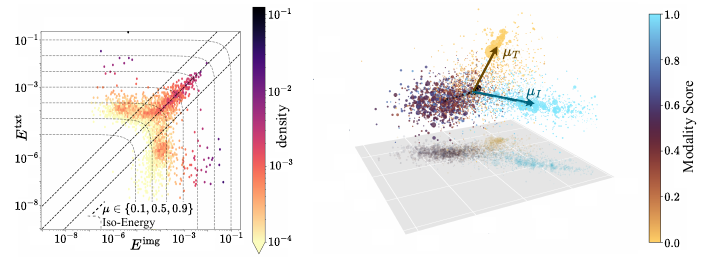}
  \caption{\textbf{(Left) Energy distribution across learned atoms.} The majority of features are \bimodal and medium-energy (inside diagonals defined by constant modality score $\mu$-\cref{app:modality_metrics}), while only a handful of high-energy \unimodal features dominate modality-specific variance. These high-energy \unimodal atoms behave like modality biases and are responsible for much of the observed modality gap.
  \textbf{(Right) Geometric organization of concepts.} Low-dimensional projections reveal three distinct clusters: image-only, text-only, and \bimodal. \textit{Unimodal} atoms align with the modality cones of the embedding space, while \bimodal atoms occupy a modality-agnostic subspace orthogonal to these directions, thereby sustaining cross-modal alignment.
  }
  \label{fig:energy_distribution}
\end{figure}
Having established the effectiveness of Iso-Energy on real VLMs, we now turn to a qualitative characterization of the learned concepts. This analysis focuses on how energy is distributed, how atoms are geometrically organized, and how interpretable they are in practice.  
\paragraph{Energy distribution.} A first observation concerns how energy is distributed across modalities. As illustrated in \Cref{fig:energy_distribution}, the vast majority of features are \bimodal and exhibit moderate energy levels, whereas a small subset of \unimodal features concentrates disproportionately high energy. These high-energy \unimodal atoms dominate modality-specific variance and act as biases, as further detailed in \cref{app:unimodal_biases}.  
\paragraph{Geometric organization.}  The high values of $p_{\mathrm{acc}}$ show a near-perfect alignment between latent structure and concept organization. This can be visualized by projecting the learned atoms into a low-dimensional space, revealing a clear separation into three clusters: image-only, text-only, and \bimodal (\cref{fig:energy_distribution}, right). \textit{Unimodal} atoms align tightly with the cones spanned by image and text embeddings, reproducing the geometry of the modality gap. In contrast, \bimodal atoms occupy a modality-agnostic subspace, orthogonal to the \unimodal directions. This geometry explains why \bimodal (resp. \unimodal) atoms carry cross-modal (resp. modality-specific) information.  
\paragraph{Qualitative inspection.}  Finally, we examine the semantic meaning of individual atoms by inspecting their most activating examples (\cref{app:gallery}). \textit{Bimodal} atoms are semantically stable, consistently capturing the same concept across modalities (e.g., colors, objects, actions). \textit{Unimodal} atoms, on the other hand, often reflect idiosyncratic modality-specific signals (such as poor cropping artifacts in images or ``name patterns'' in text) that contribute little to cross-modal semantics. Together, these three perspectives converge on the same conclusion: \textbf{\unimodal atoms function as modality-specific biases, while \bimodal atoms encode the shared conceptual backbone that supports cross-modal alignment.}

\section{Actionable Interventions on Multimodal Embeddings}
\label{sec:actionability}

Together, these analyses demonstrate that our dictionaries yield a structured and interpretable decomposition of multimodal embeddings. Building on this foundation, we now shift from analysis to intervention: once Iso-Energy isolates the cross-modal backbone, it enables direct manipulation of embeddings in ways that were previously inaccessible. 
In particular, we consider the minimal intervention that removes modality information and examine its effect on two structural aspects. In fact, we show that such transformations are possible without altering ranking-related capabilities even when the modality information is non-trivial (e.g., not a bias, as proposed by \citet{zhang2023diagnosing}'s Proposition 1) and under realistic assumptions of orthogonality.
\begin{proposition}[Modality information removal impact on ranking.]
\label{prop:modality_removal}
Consider $\bm{v} \in \R^d$ with decomposition $\bm{v} = \m(\x) + \c(\x)$ where $\m(\x) \in \M$ encodes modality-specific information, $\c(\x) \in \C$ captures cross-modal content, and $\R^d = \M \oplus \C$. If visual and textual information are orthogonal, then ranking preservation is guaranteed.
\end{proposition}

\begin{proof}
See \cref{app:why_remove_bias}.
\end{proof}

\begin{figure}[t]
    \centering
    \vspace{-9mm}
    \includegraphics[width=1.\linewidth]{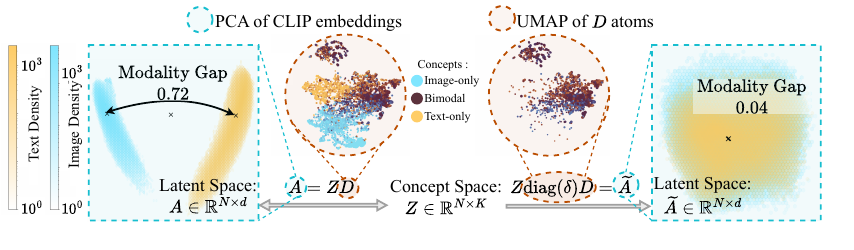}
\caption{
\textbf{The modality gap arises from multiple \unimodal concepts, while \bimodal concepts are sufficient to sustain cross-modal alignment.} 
\textbf{Left:} CLIP embeddings are re-expressed through a learned dictionary. A PCA projection highlights the separation between modalities, and a UMAP layout distinguishes two types of atoms: \unimodal and \bimodal. 
\textbf{Right:} Removing \unimodal atoms with a binary mask $\bm{\delta} \in \{0,1\}^K$ closes the gap. The reconstructed embeddings $\widetilde{\A}$ continue to support retrieval, indicating that \bimodal atoms alone capture the structure necessary for alignment.
}
    \label{fig:removing_modality_gap_banner}
    \vspace{-4mm}
\end{figure}

\begin{figure}[t]
    \centering
    \includegraphics[width=0.98\linewidth]{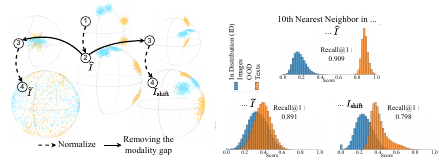}
    \caption{\textbf{Filtering \unimodal atoms closes the modality gap without harming performance.} \textbf{(Left)} Synthetic illustration comparing our method with the embedding shift baseline~\citep{liang2022mind}. Only our approach merges image and text distributions. \textbf{(Right)} Histogram of distances from each image (ID) and caption (OOD) embedding to its 10th nearest image neighbor. The modality gap is measured as the separation between the ID and OOD histograms. Filtering \unimodal atoms aligns the two distributions, whereas shift degrades performance and leaves the gap wide open.
    }
    \label{fig:removing_gap_main}
    \vspace{-4mm}
\end{figure}

\vspace{-3mm}
\paragraph{Closing the modality gap.} First, we find that by filtering out \unimodal atoms with a binary mask (\Cref{fig:removing_modality_gap_banner}), we nearly eliminate the modality gap while preserving retrieval and zero-shot performance.
As illustrated in~\Cref{fig:removing_gap_main}, this intervention merges the image and text distributions, unlike embedding shift baselines~\citep{liang2022mind}, which enforce matching means but still leave distributions well separated. Crucially, our approach preserves contrastive capabilities, showing that our dictionaries faithfully capture both contrastive and modality-specific information separately through \bimodal and \unimodal atoms.
Measuring the modality gap is typically done by measuring the distance between the mean of text and image distributions, or by measuring their linear separability. However, these fail to account for distributional mismatches remaining post intervention. For this reason, we chose to turn to the out-of-distribution (OOD) literature to measure the modality gap, and borrow a method described by \citet{sun2022out}. This method measures the separation between the blue and orange histograms in \Cref{fig:removing_gap_main}.
Our method consists in the following intervention on the activations, indicated with a tilde: $\widetilde{\A}\coloneqq(\Z \odot\bm{\delta})\D$, where $\Z\in\R^{N\times K}$ contains the sparse codes, $\D$ is the concept dictionary, and $\bm{\delta}$ is the binary mask filtering out \unimodal features, broadcast to the size of $\Z$. Reconstructed activations are indicated with a hat: $\widehat{\A}\coloneqq \Z\D$. The embedding shift intervention, and variants described in the appendix, consists in adding a modality-wise constant, essentially moving the mean of each distribution $\bm{\mu}_{I/T}$. \citet{liang2022mind}'s shift method transforms the images by $\bm{I}_{\text{shift}}\coloneqq \bm{I} - \bm{\mu}_I+\frac{\bm{\mu}_I+\bm{\mu}_T}{2}$, similarly for texts $\bm{T}_{\text{shift}}$.

\begin{figure}[t]
  \vspace{-12mm}
  \centering
  \includegraphics[width=0.8\textwidth]{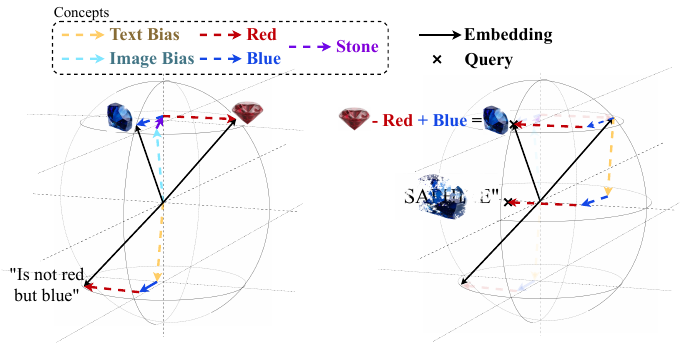}
  \vspace{-8mm}
  \caption{\textbf{Semantic vector arithmetic restricted to cross-modal information.} 
  Starting from a Ruby (red stone), the target is a Sapphire (blue stone).
  The classical edit vector $\bm{\Delta} =$ Text $+$ Blue $-$ Red is polluted by \unimodal directions, producing a query $\bm{Q} = \bm{I}_{\text{src}} + \bm{\Delta}$ that drifts out-of-distribution. 
  In contrast, restricting to \bimodal atoms yields $\bm{Q}_{\text{SAE}}$, which lies on the semantic manifold and reliably retrieves the correct target. 
  This illustrates how \unimodal features inject modality-specific bias into $\bm{\Delta}$, while Iso-Energy isolates the truly shared concepts that support valid semantic arithmetic.
  }
  \label{fig:retrieval_semantics}
  \vspace{-7mm}
\end{figure}

\vspace{-3mm}
\paragraph{Semantic vector arithmetic.} Iso-Energy also grounds semantic manipulations. Let $I_{\mathrm{src}}$ be the source image, and $\Delta$ be the textual description of the difference between the source and target image -- i.e., the relative caption. Restricting vector arithmetic to \bimodal atoms $Q_{\mathrm{SAE}}\coloneqq I_{\mathrm{src}} + \widetilde{\Delta}$ (\Cref{fig:retrieval_semantics}) produces queries that remain in-distribution (\Cref{fig:retrieval_OOD,tab:retrieval_OOD}) while preserving retrieval (\cref{app:arithmetic}). In contrast, classical arithmetic $Q = I_{\mathrm{src}} + \Delta$ incorporates \unimodal noise from the text embedding on top of the interesting cross-modal information, yielding degenerate queries that drift outside the embedding distribution. Our intervention consistently produces queries in-distribution without degrading performance on the FashionIQ benchmark~\cite{wu2021fashion}, demonstrating that the \bimodal backbone revealed by Iso-Energy is practically useful.

\begin{minipage}{0.65\linewidth}
\centering
    \centering
    \vspace{-2mm}
    \includegraphics[width=0.9\linewidth]{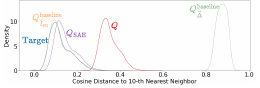}
    \captionof{figure}{\textbf{Out-of-distribution behavior of semantic queries.} Histogram of distances between each query and its 10th nearest neighbor in the target image distribution. Classical arithmetic $\bm{Q} = \bm{I}_{\text{src}} + \bm{\Delta}$ drifts out-of-distribution, while our \bimodal-restricted query $\bm{Q}_{\text{SAE}}$ remains aligned with the target space. Baselines using only the source image ($\bm{Q}^{\text{baseline}}_{\widehat{\bm{I}}_{\text{src}}}$) or only the caption difference ($\bm{Q}^{\text{baseline}}_{\widehat{\bm{\Delta}}}$) confirm the expected extremes: perfectly in-distribution and fully OOD, respectively.}
    \label{fig:retrieval_OOD}
\end{minipage}
\hfill
\begin{minipage}{0.32\linewidth}
    \centering
    \resizebox{\linewidth}{!}{
    \begin{tabular}{lc|c}
        \toprule
        OOD score ($\downarrow$) & $\bm{Q}$ & $\bm{Q}_\text{SAE}$ \\
        \midrule
        \quad CLIP & 0.97 & \textbf{0.77} \\
        \quad {CLIP-L} & 0.95 & \textbf{0.76} \\
        \quad {OpenCLIP} & 0.86 & \textbf{0.68} \\
        \quad {OpenCLIP-L} & 0.87 & \textbf{0.72} \\
        \quad {SigLIP} & 0.99 & \textbf{0.70} \\
        \quad {SigLIP2} & 0.99 & \textbf{0.61} \\
        
        \bottomrule
    \end{tabular}}
    \captionof{table}{OOD scores for classical queries vs. concept-based queries $Q_{\text{SAE}}$. Lower is better, best results per line in bold.}
    \label{tab:retrieval_OOD}
\end{minipage}%

Consider the following example. Let $I_{\mathrm{src}}$ be an image of a ruby, $\Delta$ be the prompt \textit{"is not red but blue"} and the target be an image of a saphire. This is the example illustrated in \Cref{fig:retrieval_semantics}. Adding $\Delta$ to the ruby produces a query that contains both textual-only and visual-only concepts and therefore that do not correspond to any realistic embedding. Adding only the cross-modal concepts of $\Delta$, however, produces a query that actually corresponds to an image of a saphire.

\section{Conclusion}
The Iso-Energy Assumption introduces a simple yet effective inductive bias for analyzing multimodal representations. In synthetic settings, it facilitates the recovery of ground-truth structure; in large-scale vision–language models, it consistently reveals a compact \bimodal basis that supports cross-modal alignment. This basis makes multimodal concepts accessible, isolates \unimodal concepts, closes the modality gap, and enables controlled semantic edits even in foundation-scale VLMs. In contrast, standard sparse autoencoders tend to learn a diffuse mixture of \unimodal and \bimodal atoms despite similar reconstruction quality, which obscures the structure that underpins alignment.

Despite its promising results, our approach has several limitations. First, the alignment penalty is sensitive to the choice of its weighting coefficient $\beta$: when too small, it becomes ineffective; when too large, it can lead to degenerate feature representations. In this work, we select $\beta$ through simple sweeps, but a more principled calibration—e.g., by coupling the penalty to a performance constraint or stability criterion—remains an important direction for future work.
Second, our analysis is conducted on reconstructions produced by the sparse autoencoder, rather than on the original embeddings. This constrains the reported performance to the autoencoder’s reconstruction regime.
Finally, our experiments are limited to dual-encoder vision–language models. Whether the same structural invariants and alignment properties hold in models with cross-attention mechanisms or generative training objectives remains an open question.

More broadly, our results support a hypothesis-driven approach to interpretability: inductive biases should be selected to reflect the structural properties relevant to downstream tasks, rather than applied indiscriminately. When properly aligned with task-relevant objectives, even simple biases can illuminate underlying mechanisms that might otherwise remain hidden—and offer actionable control without compromising core performance.

\bibliography{main}

@string{bmvc={Proceedings of the British Machine Vision Conference (BMVC)}}

@string{cvpr={Proceedings of the IEEE Conference on Computer Vision and Pattern Recognition (CVPR)}}

@string{eccv={Proceedings of the IEEE European Conference on Computer Vision (ECCV)}}

@string{iccv={Proceedings of the IEEE International Conference on Computer Vision (ICCV)}}

@string{iclr={Proceedings of the International Conference on Learning Representations (ICLR)}}

@string{neurips={Advances in Neural Information Processing Systems (NeurIPS)}}

@string{aaai={Proceedings of the AAAI Conference on Artificial Intelligence (AAAI)}}

@string{icml={Proceedings of the International Conference on Machine Learning (ICML)}}

@string{dsaa={Proceedings of the IEEE International Conference on data science and advanced analytics (DSAA)}}

@string{springer={Springer International Publishing}}

@string{arxiv={{A}r{X}iv e-print}}

@article{zeiler2013visualizing,
title={Visualizing and Understanding Convolutional Networks},
author={Matthew D Zeiler and Rob Fergus},
journal=eccv,
pages={818--833},
year={2014}
}

@article{fel2021sobol,
title={Look at the Variance! Efficient Black-box Explanations with Sobol-based Sensitivity Analysis},
author={Fel, Thomas and Cadene, Remi and Chalvidal, Mathieu and Cord, Matthieu and Vigouroux, David and Serre, Thomas},
journal=neurips,
volume={34},
pages={26005--26014},
year={2021}
}

@article{smilkov2017smoothgrad,
title={SmoothGrad: removing noise by adding noise},
author={Daniel Smilkov and Nikhil Thorat and Been Kim and Fernanda Viégas and Martin Wattenberg},
journal=icml,
year={2017}
}

@article{sundararajan2017axiomatic,
title={Axiomatic Attribution for Deep Networks},
author={Mukund Sundararajan and Ankur Taly and Qiqi Yan},
journal=icml,
pages={3319--3328},
year={2017}
}

@article{sixt2022users,
title={Do Users Benefit From Interpretable Vision? A User Study, Baseline, And Dataset},
author={Sixt, Leon and Schuessler, Martin and Popescu, Oana-Iuliana and Wei{\ss}, Philipp and Landgraf, Tim},
journal=iclr,
year={2022}
}

@article{gilpin2018explaining,
title={Explaining explanations: An overview of interpretability of machine learning},
author={Gilpin, Leilani H. and Bau, David and Yuan, Ben Z and Bajwa, Ayesha and Specter, Michael and Kagal, Lalana},
journal=dsaa,
pages={80--89},
year={2018}
}

@article{fel2021cannot,
  title={What i cannot predict, i do not understand: A human-centered evaluation framework for explainability methods},
  author={Colin, Julien and Fel, Thomas and Cad{\`e}ne, R{\'e}mi and Serre, Thomas},
  journal=neurips,
  volume={35},
  pages={2832--2845},
  year={2022}
}

@article{coco2014,
title={Microsoft coco: Common objects in context},
author={Lin, Tsung-Yi and Maire, Michael and Belongie, Serge and Hays, James and Perona, Pietro and Ramanan, Deva and Doll{\'a}r, Piotr and Zitnick, C Lawrence},
journal=eccv,
pages={740--755},
year={2014}
}

@article{imagenet2009,
  title={Imagenet: A large-scale hierarchical image database},
  author={Deng, Jia and Dong, Wei and Socher, Richard and Li, Li-Jia and Li, Kai and Fei-Fei, Li},
  journal=cvpr,
  pages={248--255},
  year={2009},
  organization={Ieee}
}

@article{petsiuk2018rise,
title={Rise: Randomized input sampling for explanation of black-box models},
author={Petsiuk, Vitali and Das, Abir and Saenko, Kate},
journal=bmvc,
  pages        = {151},
  publisher    = {{BMVA} Press},
year={2018}
}

@article{bricken2023monosemanticity,
title={Towards Monosemanticity: Decomposing Language Models With Dictionary Learning},
author={Bricken, Trenton and Templeton, Adly and Batson, Joshua and Chen, Brian and Jermyn, Adam and Conerly, Tom and Turner, Nick and Anil, Cem and Denison, Carson and Askell, Amanda and Lasenby, Robert and Wu, Yifan and Kravec, Shauna and Schiefer, Nicholas and Maxwell, Tim and Joseph, Nicholas and Hatfield-Dodds, Zac and Tamkin, Alex and Nguyen, Karina and McLean, Brayden and Burke, Josiah E and Hume, Tristan and Carter, Shan and Henighan, Tom and Olah, Christopher},
journal={Transformer Circuits Thread},
year={2023}
}

@article{elhage2022superposition,
title={Toy Models of Superposition},
author={Elhage, Nelson and Hume, Tristan and Olsson, Catherine and Schiefer, Nicholas and Henighan, Tom and Kravec, Shauna and Hatfield-Dodds, Zac and Lasenby, Robert and Drain, Dawn and Chen, Carol and Grosse, Roger and McCandlish, Sam and Kaplan, Jared and Amodei, Dario and Wattenberg, Martin and Olah, Christopher},
journal={Transformer Circuits Thread},
year={2022}
}

@inproceedings{kim2021hive,
author = {Kim, Sunnie S. Y. and Meister, Nicole and Ramaswamy, Vikram V. and Fong, Ruth and Russakovsky, Olga},
title = {HIVE: Evaluating the Human Interpretability of Visual Explanations},
year = {2022},
isbn = {978-3-031-19774-1},
publisher = {Springer-Verlag},
address = {Berlin, Heidelberg},
url = {https://doi.org/10.1007/978-3-031-19775-8_17},
doi = {10.1007/978-3-031-19775-8_17},
booktitle = {Computer Vision – ECCV 2022: 17th European Conference, Tel Aviv, Israel, October 23–27, 2022, Proceedings, Part XII},
pages = {280–298},
numpages = {19},
location = {Tel Aviv, Israel}
}

@inproceedings{hase2020evaluating,
    title = "Evaluating Explainable {AI}: Which Algorithmic Explanations Help Users Predict Model Behavior?",
    author = "Hase, Peter  and
      Bansal, Mohit",
    editor = "Jurafsky, Dan  and
      Chai, Joyce  and
      Schluter, Natalie  and
      Tetreault, Joel",
    booktitle = "Proceedings of the 58th Annual Meeting of the Association for Computational Linguistics",
    month = jul,
    year = "2020",
    address = "Online",
    publisher = "Association for Computational Linguistics",
    url = "https://aclanthology.org/2020.acl-main.491/",
    doi = "10.18653/v1/2020.acl-main.491",
    pages = "5540--5552"
}

@article{zhang2021invertible,
title={Invertible concept-based explanations for cnn models with non-negative concept activation vectors},
author={Zhang, Ruihan and Madumal, Prashan and Miller, Tim and Ehinger, Krista A and Rubinstein, Benjamin IP},
journal=aaai,
  volume={35},
  number={13},
  pages={11682--11690},
year={2021}
}

@article{makhzani2013k,
title={K-sparse autoencoders},
author={Makhzani, Alireza and Frey, Brendan},
journal=iclr,
year={2014}
}

@article{sun2022out,
title={Out-of-distribution detection with deep nearest neighbors},
author={Sun, Yiyou and Ming, Yifei and Zhu, Xiaojin and Li, Yixuan},
journal=icml,
pages={20827--20840},
year={2022}
}

@article{nguyen2021effectiveness,
title={The effectiveness of feature attribution methods and its correlation with automatic evaluation scores},
author={Nguyen, Giang and Kim, Daeyoung and Nguyen, Anh},
journal=neurips,
volume={34},
pages={26422--26436},
year={2021}
}

@article{ghorbani2019towards,
title={Towards automatic concept-based explanations},
author={Ghorbani, Amirata and Wexler, James and Zou, James Y and Kim, Been},
journal=neurips,
volume={32},
year={2019}
}

@article{fel2023craft,
title={CRAFT: Concept Recursive Activation FacTorization for Explainability},
author={Thomas Fel and Agustin Picard and Louis Bethune and Thibaut Boissin and David Vigouroux and Julien Colin and Rémi Cadène and Thomas Serre},
journal=cvpr,
pages={2711--2721},
year={2023}
}

@article{fel2023holistic,
title={A Holistic Approach to Unifying Automatic Concept Extraction and Concept Importance Estimation},
author={Fel, Thomas and Boutin, Victor and Moayeri, Mazda and Cadene, Remi and Bethune, Louis and Chalvidal, Mathieu and Serre, Thomas},
journal=neurips,
volume={36},
pages={54805--54818},
year={2023}
}

@article{bhalla2024interpreting,
  title={Interpreting clip with sparse linear concept embeddings (splice)},
  author={Bhalla, Usha and Oesterling, Alex and Srinivas, Suraj and Calmon, Flavio and Lakkaraju, Himabindu},
  journal=neurips,
  volume={37},
  pages={84298--84328},
  year={2024}
}

@article{zhai2023sigmoid,
title={Sigmoid loss for language image pre-training},
author={Zhai, Xiaohua and Mustafa, Basil and Kolesnikov, Alexander and Beyer, Lucas},
journal=iccv,
pages={11975--11986},
year={2023}
}

@article{bussmann2024batchtopk,
  title={Batchtopk sparse autoencoders},
  author={Bussmann, Bart and Leask, Patrick and Nanda, Neel},
  journal=arxiv,
  year={2024}
}

@article{rajamanoharan2024jumping,
title={Jumping ahead: Improving reconstruction fidelity with jumprelu sparse autoencoders},
author={Rajamanoharan, Senthooran and Lieberum, Tom and Sonnerat, Nicolas and Conmy, Arthur and Varma, Vikrant and Kramar, Janos and Nanda, Neel},
journal=arxiv,
year={2024}
}

@article{gao2024scaling,
title={Scaling and evaluating sparse autoencoders},
author={Gao, Leo and la Tour, Tom Dupre and Tillman, Henk and Goh, Gabriel and Troll, Rajan and Radford, Alec and Sutskever, Ilya and Leike, Jan and Wu, Jeffrey},
journal=iclr,
year={2025}
}

@inproceedings{
wattenberg2024relational,
title={Relational Composition in Neural Networks: A Survey and Call to Action},
author={Martin Wattenberg and Fernanda Vi{\'e}gas},
booktitle={ICML 2024 Workshop on Mechanistic Interpretability},
year={2024},
}

@article{bau2017network,
  title={Network dissection: Quantifying interpretability of deep visual representations},
  author={Bau, David and Zhou, Bolei and Khosla, Aditya and Oliva, Aude and Torralba, Antonio},
  journal=cvpr,
  pages={6541--6549},
  year={2017}
}

@article{cunningham2023sparse,
  title={Sparse autoencoders find highly interpretable features in language models},
  author={Cunningham, Hoagy and Ewart, Aidan and Riggs, Logan and Huben, Robert and Sharkey, Lee},
  journal=arxiv,
  year={2023}
}

@article{gorton2024missing,
title={The Missing Curve Detectors of InceptionV1: Applying Sparse Autoencoders to InceptionV1 Early Vision},
author={Gorton, Liv},
journal=arxiv,
year={2024}
}

@inproceedings{surkov2024unpacking,
  title={Unpacking sdxl turbo: Interpreting text-to-image models with sparse autoencoders},
  author={Surkov, Viacheslav and Wendler, Chris and Terekhov, Mikhail and Deschenaux, Justin and West, Robert and Gulcehre, Caglar},
  booktitle={Mechanistic Interpretability for Vision at CVPR 2025 (Non-proceedings Track)},
  year={2025}
}

@article{locatello2019challenging,
title={Challenging common assumptions in the unsupervised learning of disentangled representations},
author={Locatello, Francesco and Bauer, Stefan and Lucic, Mario and Raetsch, Gunnar and Gelly, Sylvain and Scholkopf, Bernhard and Bachem, Olivier},
journal=icml,
pages={4114--4124},
year={2019}
}

@article{gresele2020incomplete,
title={The incomplete rosetta stone problem: Identifiability results for multi-view nonlinear ica},
author={Gresele, Luigi and Rubenstein, Paul K and Mehrjou, Arash and Locatello, Francesco and Scholkopf, Bernhard},
journal={Uncertainty in Artificial Intelligence},
year={2020}
}

@inproceedings{dravid2023rosetta,
  title={Rosetta neurons: Mining the common units in a model zoo},
  author={Dravid, Amil and Gandelsman, Yossi and Efros, Alexei A and Shocher, Assaf},
  booktitle={Proceedings of the IEEE/CVF International Conference on Computer Vision},
  pages={1934--1943},
  year={2023}
}

@InProceedings{pmlr-v235-huh24a,
  title = 	 {Position: The Platonic Representation Hypothesis},
  author =       {Huh, Minyoung and Cheung, Brian and Wang, Tongzhou and Isola, Phillip},
  booktitle = 	 {Proceedings of the 41st International Conference on Machine Learning},
  pages = 	 {20617--20642},
  year = 	 {2024},
  editor = 	 {Salakhutdinov, Ruslan and Kolter, Zico and Heller, Katherine and Weller, Adrian and Oliver, Nuria and Scarlett, Jonathan and Berkenkamp, Felix},
  volume = 	 {235},
  series = 	 {Proceedings of Machine Learning Research},
  month = 	 {21--27 Jul},
  publisher =    {PMLR},
  url = 	 {https://proceedings.mlr.press/v235/huh24a.html}
}

@article{khemakhem2020variational,
title={Variational autoencoders and nonlinear ica: A unifying framework},
author={Khemakhem, Ilyes and Kingma, Diederik and Monti, Ricardo and Hyvarinen, Aapo},
journal=icml,
year={2020}
}

@article{olshausen1996emergence,
  title={Emergence of simple-cell receptive field properties by learning a sparse code for natural images},
  author={Olshausen, Bruno A and Field, David J},
  journal={Nature},
  volume={381},
  number={6583},
  pages={607--609},
  year={1996},
  publisher={Nature Publishing Group UK London}
}

@article{olshausen1997sparse,
  title={Sparse coding with an overcomplete basis set: A strategy employed by V1?},
  author={Olshausen, Bruno A and Field, David J},
  journal={Vision research},
  volume={37},
  number={23},
  pages={3311--3325},
  year={1997},
  publisher={Elsevier}
}

@book{dumitrescu2018dictionary,
  title={Dictionary learning algorithms and applications},
  author={Dumitrescu, Bogdan and Irofti, Paul},
  year={2018},
  publisher={Springer}
}

@article{rubinstein2010dictionaries,
  title={Dictionaries for sparse representation modeling},
  author={Rubinstein, Ron and Bruckstein, Alfred M and Elad, Michael},
  journal={Proceedings of the IEEE},
  volume={98},
  number={6},
  pages={1045--1057},
  year={2010},
  publisher={IEEE}
}

@article{mairal2014sparse,
  title={Sparse modeling for image and vision processing},
  author={Mairal, Julien and Bach, Francis and Ponce, Jean and others},
  journal={Foundations and Trends{\textregistered} in Computer Graphics and Vision},
  volume={8},
  number={2-3},
  pages={85--283},
  year={2014},
  publisher={Now Publishers, Inc.}
}

@article{lee2006efficient,
title={Efficient sparse coding algorithms},
author={Lee, Honglak and Battle, Alexis and Raina, Rajat and Ng, Andrew},
journal=neurips,
volume={19},
year={2006}
}

@article{foldiak2008sparse,
title={Sparse coding},
author={Foldiak, Peter and Endres, Dominik Maria},
journal=arxiv,
year={2008}
}

@book{elad2010sparse,
title={Sparse and redundant representations: from theory to applications in signal and image processing},
author={Elad, Michael},
publisher=springer,
year={2010}
}

@article{rentzeperis2023beyond,
title={Beyond l1 sparse coding in V1},
author={Rentzeperis, Ilias and Calatroni, Luca and Perrinet, Laurent U and Prandi, Dario},
journal={PLoS Computational Biology},
  volume={19},
  number={9},
  pages={e1011459},
year={2023}
}

@article{tovsic2011dictionary,
  title={Dictionary learning},
  author={To{\v{s}}i{\'c}, Ivana and Frossard, Pascal},
  journal={IEEE Signal Processing Magazine},
  volume={28},
  number={2},
  pages={27--38},
  year={2011},
  publisher={IEEE}
}

@article{fel2025archetypal,
title={Archetypal sae: Adaptive and stable dictionary learning for concept extraction in large vision models},
author={Fel, Thomas and Lubana, Ekdeep Singh and Prince, Jacob S and Kowal, Matthew and Boutin, Victor and Papadimitriou, Isabel and Wang, Binxu and Wattenberg, Martin and Ba, Demba and Konkle, Talia},
journal=icml,
year={2025}
}

@article{hindupur2025projecting,
title={Projecting assumptions: The duality between sparse autoencoders and concept geometry},
author={Hindupur, Sai Sumedh R and Lubana, Ekdeep Singh and Fel, Thomas and Ba, Demba},
journal=arxiv,
year={2025}
}

@article{johnson1984extensions,
  title={Extensions of Lipschitz mappings into a Hilbert space},
  author={Johnson, William B and Lindenstrauss, Joram and others},
  journal={Contemporary mathematics},
  volume={26},
  number={189-206},
  pages={1},
  year={1984}
}

@inproceedings{larsen2017optimality,
  title={Optimality of the Johnson-Lindenstrauss lemma},
  author={Larsen, Kasper Green and Nelson, Jelani},
  booktitle={2017 IEEE 58th annual symposium on foundations of computer science (FOCS)},
  pages={633--638},
  year={2017},
  organization={IEEE}
}

@article{tschannen2025siglip2,
title={SigLIP 2: Multilingual Vision-Language Encoders with Improved Semantic Understanding, Localization, and Dense Features}, 
author={Michael Tschannen and Alexey Gritsenko and Xiao Wang and Muhammad Ferjad Naeem and Ibrahim Alabdulmohsin and Nikhil Parthasarathy and Talfan Evans and Lucas Beyer and Ye Xia and Basil Mustafa and Olivier Hénaff and Jeremiah Harmsen and Andreas Steiner and Xiaohua Zhai},
journal=arxiv,
year={2025},
}

@article{fini2024multimodal,
  title={Multimodal autoregressive pre-training of large vision encoders},
  author={Fini, Enrico and Shukor, Mustafa and Li, Xiujun and Dufter, Philipp and Klein, Michal and Haldimann, David and Aitharaju, Sai and da Costa, Victor G Turrisi and B{\'e}thune, Louis and Gan, Zhe and others},
  journal=cvpr,
  pages={9641--9654},
  year={2025}
}

@article{liang2022mind,
title={Mind the gap: Understanding the modality gap in multi-modal contrastive representation learning},
author={Liang, Victor Weixin and Zhang, Yuhui and Kwon, Yongchan and Yeung, Serena and Zou, James Y},
journal=neurips,
volume={35},
pages={17612--17625},
year={2022}
}

@article{chen2024spatialvlm,
  title={Spatialvlm: Endowing vision-language models with spatial reasoning capabilities},
  author={Chen, Boyuan and Xu, Zhuo and Kirmani, Sean and Ichter, Brain and Sadigh, Dorsa and Guibas, Leonidas and Xia, Fei},
  booktitle={Proceedings of the IEEE/CVF Conference on Computer Vision and Pattern Recognition},
  journal=cvpr,
  pages={14455--14465},
  year={2024}
}

@article{singhal2023large,
title={Large language models encode clinical knowledge},
author={Karan Singhal and Shekoofeh Azizi and Tao Tu and S. Sara Mahdavi and Jason Wei and Hyung Won Chung and Nathan Scales and Ajay Tanwani and Heather Cole-Lewis and Stephen Pfohl and Perry Payne and Martin Seneviratne and Paul Gamble and Chris Kelly and Abubakr Babiker and Nathanael Schärli and Aakanksha Chowdhery and Philip Mansfield and Dina Demner-Fushman and Blaise Agüera y Arcas and Dale Webster and Greg S. Corrado and Yossi Matias and Katherine Chou and Juraj Gottweis and Nenad Tomasev and Yun Liu and Alvin Rajkomar and Joelle Barral and Christopher Semturs and Alan Karthikesalingam and Vivek Natarajan},
  journal={Nature},
  volume={620},
  number={7972},
  pages={172--180},
  year={2023},
  publisher={Nature Publishing Group}
}

@article{zhou2024vision,
title={Vision language models in autonomous driving: A survey and outlook},
author={Zhou, Xingcheng and Liu, Mingyu and Yurtsever, Ekim and Zagar, Bare Luka and Zimmer, Walter and Cao, Hu and Knoll, Alois C},
journal={IEEE Transactions on Intelligent Vehicles},
year={2024},
publisher={IEEE}
}

@article{shukor2025smolvla,
title={SmolVLA: A Vision-Language-Action Model for Affordable and Efficient Robotics},
author={Shukor, Mustafa and Aubakirova, Dana and Capuano, Francesco and Kooijmans, Pepijn and Palma, Steven and Zouitine, Adil and Aractingi, Michel and Pascal, Caroline and Russi, Martino and Marafioti, Andres and Alibert, Simon and Cord, Matthieu and Wolf, Thomas and Cadene, Remi},
journal=arxiv,
year={2025},
}

@article{hyvarinen2016unsupervised,
  title={Unsupervised feature extraction by time-contrastive learning and nonlinear ica},
  author={Hyvarinen, Aapo and Morioka, Hiroshi},
  journal={Advances in neural information processing systems},
  year={2016}
}

@article{dreyer2025mechanistic,
  title={Mechanistic understanding and validation of large AI models with SemanticLens},
  author={Dreyer, Maximilian and Berend, Jim and Labarta, Tobias and Vielhaben, Johanna and Wiegand, Thomas and Lapuschkin, Sebastian and Samek, Wojciech},
  journal={Nature Machine Intelligence},
  pages={1--14},
  year={2025},
  publisher={Nature Publishing Group UK London}
}

@article{rao2024discover,
author="Rao, Sukrut
and Mahajan, Sweta
and B{\"o}hle, Moritz
and Schiele, Bernt",
editor="Leonardis, Ale{\v{s}}
and Ricci, Elisa
and Roth, Stefan
and Russakovsky, Olga
and Sattler, Torsten
and Varol, G{\"u}l",
title="Discover-then-Name: Task-Agnostic Concept Bottlenecks via Automated Concept Discovery",
journal=eccv,
year="2024",
publisher="Springer Nature Switzerland",
address="Cham",
pages="444--461",
isbn="978-3-031-72980-5"
}

@article{pach2025sparse,
title={Sparse autoencoders learn monosemantic features in vision-language models},
author={Pach, Mateusz and Karthik, Shyamgopal and Bouniot, Quentin and Belongie, Serge and Akata, Zeynep},
journal=arxiv,
year={2025}
}

@article{cherti2023reproducible,
  title={Reproducible scaling laws for contrastive language-image learning},
  author={Cherti, Mehdi and Beaumont, Romain and Wightman, Ross and Wortsman, Mitchell and Ilharco, Gabriel and Gordon, Cade and Schuhmann, Christoph and Schmidt, Ludwig and Jitsev, Jenia},
  journal=cvpr,
  pages={2818--2829},
  year={2023}
}

@article{yang2022openood,
  title={Openood: Benchmarking generalized out-of-distribution detection},
  author={Yang, Jingkang and Wang, Pengyun and Zou, Dejian and Zhou, Zitang and Ding, Kunyuan and Peng, Wenxuan and Wang, Haoqi and Chen, Guangyao and Li, Bo and Sun, Yiyou and others},
  journal=neurips,
  volume={35},
  pages={32598--32611},
  year={2022}
}

@article{wu2021fashion,
  title={Fashion iq: A new dataset towards retrieving images by natural language feedback},
  author={Wu, Hui and Gao, Yupeng and Guo, Xiaoxiao and Al-Halah, Ziad and Rennie, Steven and Grauman, Kristen and Feris, Rogerio},
  journal=cvpr,
  pages={11307--11317},
  year={2021}
}

@article{papadimitriou2025,
      title={Interpreting the linear structure of vision-language model embedding spaces}, 
      author={Isabel Papadimitriou and Huangyuan Su and Thomas Fel and Sham Kakade and Stephanie Gil},
      journal=arxiv,
      year={2025},
}

@article{schuhmann2021laion,
  title={Laion-400m: Open dataset of clip-filtered 400 million image-text pairs},
  author={Schuhmann, Christoph and Vencu, Richard and Beaumont, Romain and Kaczmarczyk, Robert and Mullis, Clayton and Katta, Aarush and Coombes, Theo and Jitsev, Jenia and Komatsuzaki, Aran},
  journal=arxiv,
  year={2021}
}

@inproceedings{shi2023,
  title={Towards understanding the modality gap in clip},
  author={Shi, Peiyang and Welle, Michael C and Bj{\"o}rkman, M{\aa}rten and Kragic, Danica},
  booktitle={ICLR 2023 workshop on multimodal representation learning: perks and pitfalls},
  year={2023}
}

@article{yaras2024,
  title={Explaining and Mitigating the Modality Gap in Contrastive Multimodal Learning},
  author={Yaras, Can and Chen, Siyi and Wang, Peng and Qu, Qing},
  journal=arxiv,
  year={2024}
}

@inproceedings{ethayarajh2019,
    title = "How Contextual are Contextualized Word Representations? {C}omparing the Geometry of {BERT}, {ELM}o, and {GPT}-2 Embeddings",
    author = "Ethayarajh, Kawin",
    editor = "Inui, Kentaro  and
      Jiang, Jing  and
      Ng, Vincent  and
      Wan, Xiaojun",
    booktitle = "Proceedings of the 2019 Conference on Empirical Methods in Natural Language Processing and the 9th International Joint Conference on Natural Language Processing (EMNLP-IJCNLP)",
    month = nov,
    year = "2019",
    address = "Hong Kong, China",
    publisher = "Association for Computational Linguistics",
    url = "https://aclanthology.org/D19-1006/",
    doi = "10.18653/v1/D19-1006",
    pages = "55--65"
}

@article{fahim2024s,
  title={It's Not a Modality Gap: Characterizing and Addressing the Contrastive Gap},
  author={Fahim, Abrar and Murphy, Alex and Fyshe, Alona},
  journal={arXiv preprint arXiv:2405.18570},
  year={2024}
}

@inproceedings{radford2021,
  title={Learning transferable visual models from natural language supervision},
  author={Radford, Alec and Kim, Jong Wook and Hallacy, Chris and Ramesh, Aditya and Goh, Gabriel and Agarwal, Sandhini and Sastry, Girish and Askell, Amanda and Mishkin, Pamela and Clark, Jack and others},
  booktitle={International conference on machine learning},
  pages={8748--8763},
  year={2021},
  organization={PmLR}
}

@article{mallat1993matching,
  title={Matching pursuits with time-frequency dictionaries},
  author={Mallat, St{\'e}phane G and Zhang, Zhifeng},
  journal={IEEE Transactions on signal processing},
  volume={41},
  number={12},
  pages={3397--3415},
  year={1993},
  publisher={IEEE}
}

@inproceedings{jiang2023understanding,
  title={Understanding and constructing latent modality structures in multi-modal representation learning},
  author={Jiang, Qian and Chen, Changyou and Zhao, Han and Chen, Liqun and Ping, Qing and Tran, Son Dinh and Xu, Yi and Zeng, Belinda and Chilimbi, Trishul},
  booktitle={Proceedings of the IEEE/CVF Conference on Computer Vision and Pattern Recognition},
  pages={7661--7671},
  year={2023}
}

@article{costa2025flat,
  title={From Flat to Hierarchical: Extracting Sparse Representations with Matching Pursuit},
  author={Costa, Val{\'e}rie and Fel, Thomas and Lubana, Ekdeep Singh and Tolooshams, Bahareh and Ba, Demba},
  journal={arXiv preprint arXiv:2506.03093},
  year={2025}
}

@inproceedings{
  zhang2023diagnosing,
  title={Diagnosing and Rectifying Vision Models using Language},
  author={Zhang, Yuhui and HaoChen, Jeff Z and Huang, Shih-Cheng and Wang, Kuan-Chieh and Zou, James and Yeung, Serena},
  booktitle={International Conference on Learning Representations (ICLR)},
  year={2023},
  url={https://openreview.net/pdf?id=D-zfUK7BR6c}
}

@inproceedings{
schrodi2025two,
title={Two Effects, One Trigger: On the Modality Gap, Object Bias, and Information Imbalance in Contrastive Vision-Language Models},
author={Simon Schrodi and David T. Hoffmann and Max Argus and Volker Fischer and Thomas Brox},
booktitle={The Thirteenth International Conference on Learning Representations},
year={2025},
url={https://openreview.net/forum?id=uAFHCZRmXk}
}

@inproceedings{
levi2025the,
title={The Double-Ellipsoid Geometry of {CLIP}},
author={Meir Yossef Levi and Guy Gilboa},
booktitle={Forty-second International Conference on Machine Learning},
year={2025},
url={https://openreview.net/forum?id=QGUju9B68Z}
}

@article{udandarao2022understanding,
  title={Understanding and fixing the modality gap in vision-language models},
  author={Udandarao, Vishaal},
  journal={Master's thesis, University of Cambridge},
  volume={32},
  year={2022}
}
\bibliographystyle{iclr2026_conference}

\appendix

\newpage
\addcontentsline{toc}{section}{Appendix} %
\part{Appendix} %
\parttoc %

\section{SAE-dex}
\label{app:SAE_architecture}
In this appendix, we explore several SAE algorithms popular in the community. We provide definitions along with interpretations and implicit inductive biases for each algorithm. For similar insights into architectures, refer to \citet{hindupur2025projecting}.

\paragraph{Notations.} Let's first redefine all necessary notations here. Let $A \in \R^{n \times d}$ denote the collection of $n$ activations of dimension $d$ to be disentangled. Let $D \in \R^{K \times d}$ be the matrix containing the $K$ dictionary atoms, $Z \in \R^{n \times K}$ the collection of sparse codes.

Let $\phi: x \in \R^d \to \Pi\{W x + b\} \in \R^\K$ be an encoder attempting to solve the \emph{sparse coding} optimization problem, with $W \in \R^{K \times d}$ and $b \in \R^K$ the parameters of an affine transformation on the activations, and $\Pi: \R^\K \to \R^\K$ a non linear projection function.

\paragraph{SAE.} Recent work used SAEs to disentangle latent spaces, within the framework of the \emph{Linear Representation Hypothesis} (LRH) \citep{elhage2022superposition,wattenberg2024relational}. This framework hypothesizes that latent space's activations arise from a sparse linear combination of semantic atoms. This linear combination should be linearly accessible, as all operations on these spaces are linear. SAEs attempt to find this combination. In practice, purely linear SAEs perform very poorly, and some non-linearity has to be introduced to the encoder to increase its expressiveness.

\paragraph{ReLU.} This is the most common nonlinearity, where $\Pi = \mathrm{ReLU} : \bm{}x \to \max(\bm{0}, \bm{x})$, making the SAE effectively a one layer MLP. Sparsity is enforced through a soft constraint, encoded in the loss typically as an $L_1$ penalty. This soft constraint results in desirable and undesirable inductive biases, like sparsity and activation shrinkage, respectively.

This projection assumes linear separability of concepts, with neurons' receptive fields being half spaces, independent of one another. This independence creates an inductive bias preventing the discovery of concepts too close from one another, as they become harder to separate.

\paragraph{JumpReLU.} \citet{rajamanoharan2024jumping} proposed to use a gating mechanism on top of the $\mathrm{ReLU}$, as defined by \Cref{eq:jumprelu}. This gate intuitively prevents noisy small activations, with concepts being active only if they are significantly present. This is not allowed by a simple $\mathrm{ReLU}$, as changing a concept's activation threshold can only happen through the bias, which also shifts its activation magnitude. Sparsity is often achieved by optimizing $\theta$ instead of using the $L_1$ penalty.
This projection also assumes linear separability of concepts, with receptive fields being again independent half-spaces.

\begin{equation}
    \label{eq:jumprelu}
    \Pi_\theta : x \to x H(x-\theta)
\end{equation}
where $H$ is the Heaviside step function.

\paragraph{TopK.} These SAEs were introduced by \citet{gao2024scaling} to facilitate training by enforcing a hard constraint on sparsity, rather than the soft constraint, as an inductive bias in the loss of the $\mathrm{ReLU}$ SAE. Indeed, TopK selects the $k$ most active concepts for each sample, enforcing an $L_0$ of at most $k$. This also prevents known, undesirable effects of the inductive bias.
This projection assumes angular separability of concepts, with receptive fields being polyhedral cones such that at each point in space, at most $k$ of them overlap.

This projection introduces competition between concepts---they are no longer independent of one another---which also prevents the discovery of features too close to one another, though competition introduces different dynamics from the independent case.

\begin{equation}
    \label{eq:topk}
    \Pi : x \to x \odot \bm{1}_{\{i\in\mathcal{T}_k(x)\}}\text{, where }\mathcal{T}_k(x)=\underset{k}{\arg\mathrm{top}}(x)
\end{equation}

\paragraph{BatchTopK.} This is a variant of the TopK algorithm, where competition is introduced in the whole batch rather than in each individual sample. With a batch size of $b$, this function selects the $b \times k$ most active concepts in the batch. This allows for a relaxed constraint on the $L_0$, which is now enforced to be $k$ on average, but allowing for sample-wise variations.
This projection also assumes angular separability.

\paragraph{Matching Pursuit.} As described in \cref{alg:mp}, the Matching Pursuit algorithm iteratively adds the best atom $d_i$ to the code $z$, removes it from the residual $r$, and starts again. Two main variants exist to decide the number of iterations: one where it is fixed, and one where the algorithm continues until some threshold reconstruction is reached. We opted for the former to have better control over sparsity.

\begin{algorithm}[H]
    \caption{Matching Pursuit SAE Encoder}
    \label{alg:mp}
    \begin{algorithmic}[1]
        \State \textbf{Input:} $x \in \R^d$, $D \in \R^{K \times d}$, $\lambda > 0$
        \State $z \gets 0 \in \R^K$ \Comment{Initialize sparse code}
        \State $r \gets x$ \Comment{Initialize residual}
        \For{$k \in \intint{1}{\kappa}$} \Comment{Iterate for $\kappa$ steps - enforcing an $L_0$ sparsity constraint of $\kappa$}
            \State $i \gets \underset{i}{\arg\max} \abs{D_i^\top r}$ \Comment{Find the atom that best matches the residual}
            \State $z_i \gets D_i^\top r$ \Comment{Update sparse code for this atom}
            \State $r \gets r - z_i D_i$ \Comment{Update residual}
        \EndFor
    \end{algorithmic}
\end{algorithm}

The main advantage of this algorithm over TopK is that each atom selected can only decrease the error term $r$. Thus, even though competition still exists, this algorithm cannot select two competing atoms with large activations that would result in an increased error. The main downside of this algorithm is its increased complexity.

Receptive fields are again polyhedral cones, such that exactly $k$ intersect at every point in space, but both their shape and the activation patterns inside them are now much more flexible than for TopK cones. This removes the inductive bias that constrains the geometric similarity between two atoms, which is present in all other methods. It might cause a bias to learn duplicates of a feature, especially high-energy ones, to fit noise at no additional cost.

\begin{table}[H]
    \centering
    \resizebox{\linewidth}{!}{
    \begin{tabular}{lcccccc}
        \toprule
        \textbf{Metric} & \textbf{ReLU} & \textbf{JumpReLU} & \textbf{TopK} & \textbf{BatchTopK} & \textbf{\SAE{SAE}} & \textbf{\SAEA{SAE-A}} \\
        \midrule
        \multicolumn{5}{l}{\textbf{Sparse Reconstruction}} \\
        \cmidrule(lr){1-1}
        \quad MSE ($\downarrow$)     & 0.31 & 0.26 & 0.22 & 0.22 & 0.14 & 0.16 \\
        \quad $R^2$ ($\uparrow$)   & 0.69 & 0.74 & 0.78 & 0.78 & 0.86 & 0.84 \\
        \quad $\ell_0$ ($\downarrow$)   & 19.8 & 17.4 & / & / & / & / \\
        \quad $\ell_1$ ($\downarrow$)     & 1.22 & 2.48 & 2.59 & 2.65 & 2.94 & 3.09 \\
        
        \midrule
        \multicolumn{5}{l}{\textbf{Consistency}} \\
        \cmidrule(lr){1-1}
        \quad $C\mathrm{-insertion}$ ($\uparrow$)   & 0.81 & 0.79 & 0.75 & 0.76 & 0.74 & 0.71 \\
        \quad $C\mathrm{-deletion}$ ($\downarrow$)   & 0.16 & 0.16 & 0.23 & 0.20 & 0.10 & 0.10 \\
        \quad Stability ($\downarrow$)   & 0.067 & 0.23 & 0.31 & 0.33 & 0.15 & 0.22 \\
        
        \midrule
        \multicolumn{5}{l}{\textbf{Structure in D}} \\
        \cmidrule(lr){1-1}
        \quad Stable Rank ($\uparrow$)   & 88.3 & 231 & 29.6 & 161 & 28.3 & 14.0 \\
        \quad Eff. Rank ($\uparrow$)   & 500 & 501 & 498 & 501 & 455 & 435 \\
        \quad Stable Rank (w) ($\uparrow$)   & 1.57 & 1.46 & 1.92 & 1.32 & 1.59 & 1.17 \\
        \quad Eff. Rank (w) ($\uparrow$)   & 18.0 & 37.8 & 55.7 & 46.4 & 160 & 90.3 \\
        \quad Coherence ($\downarrow$)   & 0.99 & 0.50 & 0.51 & 0.55 & 0.99 & 0.99 \\
        
        \midrule
        \multicolumn{5}{l}{\textbf{Structure in Z}} \\
        \cmidrule(lr){1-1}
        \quad Stable Rank ($\downarrow$)   & 1.97 & 1.69 & 1.76 & 1.61 & 9.03 & 1.26 \\
        \quad Eff. Rank ($\downarrow$)   & 46.0 & 45.5 & 85.4 & 80 & 829 & 532 \\
        \quad Connectivity ($\uparrow$)   & 0.24 & 0.25 & 0.45 & 0.021 & 0.007 & 0.058 \\
        \quad Neg. Inter. ($\downarrow$)   & 0.0008 & 0.0019 & 0.0003 & 0.0006 & 0.0005 & 0.0020 \\
        
        \midrule
        \multicolumn{5}{l}{\textbf{Modality}} \\
        \cmidrule(lr){1-1}
        \quad $\rho$ ($\uparrow$) & 0.46 & 0.96 & 1.98 & 0.97 & 0.33 & 4.23 \\
        \quad $\delta_{\mathrm{r@1}}$ ($\downarrow$) & 0.134 & 0.06 & 0.05 & 0.05 & 0.22 & 0.12 \\
        \quad $\mathrm{FDA}$ ($\uparrow$) & 2.52 & 3.44 & 8.77 & 3.83 & 2.63 & 4.56 \\
        \quad $p_{\mathrm{acc}}$ ($\uparrow$) & 0.87 & 0.87 & 0.88 & 0.89 & 0.85 & 0.92 \\
        
        \bottomrule
    \end{tabular}}
    \caption{Metric comparison across SAE architectures under fixed hyperparameters (expansion ratio: 8, target $\ell_0$: $20$). We only report results for SAEs trained on CLIP here.}
    \label{tab:sae_architecture_small}
\end{table}

\begin{table}[H]
    \centering
    \resizebox{\linewidth}{!}{
    \begin{tabular}{lcccccc}
        \toprule
        \textbf{Metric} & \textbf{ReLU} & \textbf{JumpReLU} & \textbf{TopK} & \textbf{BatchTopK} & \textbf{\SAE{SAE}} & \textbf{\SAEA{SAE-A}} \\
        \midrule
        \multicolumn{5}{l}{\textbf{Sparse Reconstruction}} \\
        \cmidrule(lr){1-1}
        \quad MSE ($\downarrow$)     & 0.28 & 0.19 & 0.16 & 0.16 & 0.07 & 0.08 \\
        \quad $R^2$ ($\uparrow$)   & 0.72 & 0.81 & 0.84 & 0.84 & 0.93 & 0.92 \\
        \quad $\ell_0$ ($\downarrow$)   & 50 & 49.1 & / & / & / & / \\
        \quad $\ell_1$ ($\downarrow$)     & 1.60 & 4.38 & 3.99 & 4.43 & 4.58 & 4.69 \\
        
        \midrule
        \multicolumn{5}{l}{\textbf{Consistency}} \\
        \cmidrule(lr){1-1}
        \quad $C\mathrm{-insertion}$ ($\uparrow$)   & 0.75 & 0.61 & 0.61 & 0.59 & 0.59 & 0.58 \\
        \quad $C\mathrm{-deletion}$ ($\downarrow$)   & 0.24 & 0.31 & 0.33 & 0.33 & 0.15 & 0.16 \\
        \quad Stability ($\downarrow$)   & 0.11 & 0.35 & 0.33 & 0.36 & 0.26 & 0.32 \\
        
        \midrule
        \multicolumn{5}{l}{\textbf{Structure in D}} \\
        \cmidrule(lr){1-1}
        \quad Stable Rank ($\uparrow$)   & 346 & 376 & 347 & 383 & 207 & 191 \\
        \quad Eff. Rank ($\uparrow$)   & 511 & 510 & 511 & 511 & 494 & 493 \\
        \quad Stable Rank (w) ($\uparrow$)   & 1.65 & 1.92 & 1.85 & 2.82 & 1.60 & 1.54 \\
        \quad Eff. Rank (w) ($\uparrow$)   & 26.5 & 146 & 101 & 101 & 74.8 & 55.1 \\
        \quad Coherence ($\downarrow$)   & 0.59 & 0.44 & 0.51 & 0.41 & 0.99 & 0.99 \\
        
        \midrule
        \multicolumn{5}{l}{\textbf{Structure in Z}} \\
        \cmidrule(lr){1-1}
        \quad Stable Rank ($\downarrow$)   & 1.86 & 1.58 & 1.68 & 1.43 & 4.06 & 2.80 \\
        \quad Eff. Rank ($\downarrow$)   & 48.9 & 101 & 107 & 105 & 762 & 597 \\
        \quad Connectivity ($\uparrow$)   & 0.56 & 0.91 & 0.81 & 0.56 & 0.12 & 0.13 \\
        \quad Neg. Inter. ($\downarrow$)   & 0.0012 & 0.0012 & 0.0017 & 0.0045 & 0.0013 & 0.0013 \\
        
        \midrule
        \multicolumn{5}{l}{\textbf{Modality}} \\
        \cmidrule(lr){1-1}
        \quad $\rho$ ($\uparrow$) & 1.71 & 1.41 & 1.77 & 2.71 & 1.13 & 2.69 \\
        \quad $\delta_{\mathrm{r@1}}$ ($\downarrow$) & -0.024 & 0.028 & 0.024 & 0.089 & 0.13 & 0.052 \\
        \quad $\mathrm{FDA}$ ($\uparrow$) & 8.12 & 5.17 & 5.12 & 17.3 & 2.79 & 2.57 \\
        \quad $p_{\mathrm{acc}}$ ($\uparrow$) & 0.80 & 0.81 & 0.83 & 0.89 & 0.84 & 0.93 \\
        
        \bottomrule
    \end{tabular}}
    \caption{Metric comparison across SAE architectures under fixed hyperparameters (expansion ratio: 64, target $\ell_0$: $50$). We only report results for SAEs trained on CLIP here.}
    \label{tab:sae_architecture_big}
\end{table}

\section{Tuning the strength of $\mathcal{L}_{\mathrm{align}}$}
\label{app:beta_tuning}

In this appendix, we look at the impact of the alignment loss, and how to tune its strength---which we call $\beta$ in this section. We consider the CLIP model, and train a family of SAEs with $\beta$ ranging from $10^{-5}$ to $10^{-2}$ on a logarithmic scale. The alignment loss is defined by the average cosine similarity between matching image--caption codes~:

\begin{equation}
\mathcal{L}_{\mathrm{align}}(\underbrace{\widetilde{Z}^{(\domain)}}_{\in \R^{b \times K}}, \underbrace{\widetilde{Z}^{(\domain')}}_{\in \R^{b \times K}}) = -\frac{1}{b} \mathrm{Tr}(\widetilde{Z}^{(\domain)} \cdot \widetilde{Z}^{(\domain')\top})
\end{equation}
where $b$ is the batch size, $K$ is the number of latent concept dimensions and $\widetilde{Z}^{(\domain)}$ and $\widetilde{Z}^{(\domain')}$ are normalized codes.

We observe a sharp transition, with almost no effect below $\beta=10^{-4}$, and systematic degenerate solutions above $\beta=10^{-3}$. \Cref{fig:r2vsbeta} shows the evolution of reconstruction error against $\beta$. At the transition, constant features start appearing with frequencies of 1. These features are degenerate solutions to the alignment loss. \Cref{fig:Einbvsbeta} shows the proportion of energy contained in \bimodal and degenerate features against $\beta$.

\begin{figure}[h]
    \centering
    \includegraphics[width=0.5\linewidth]{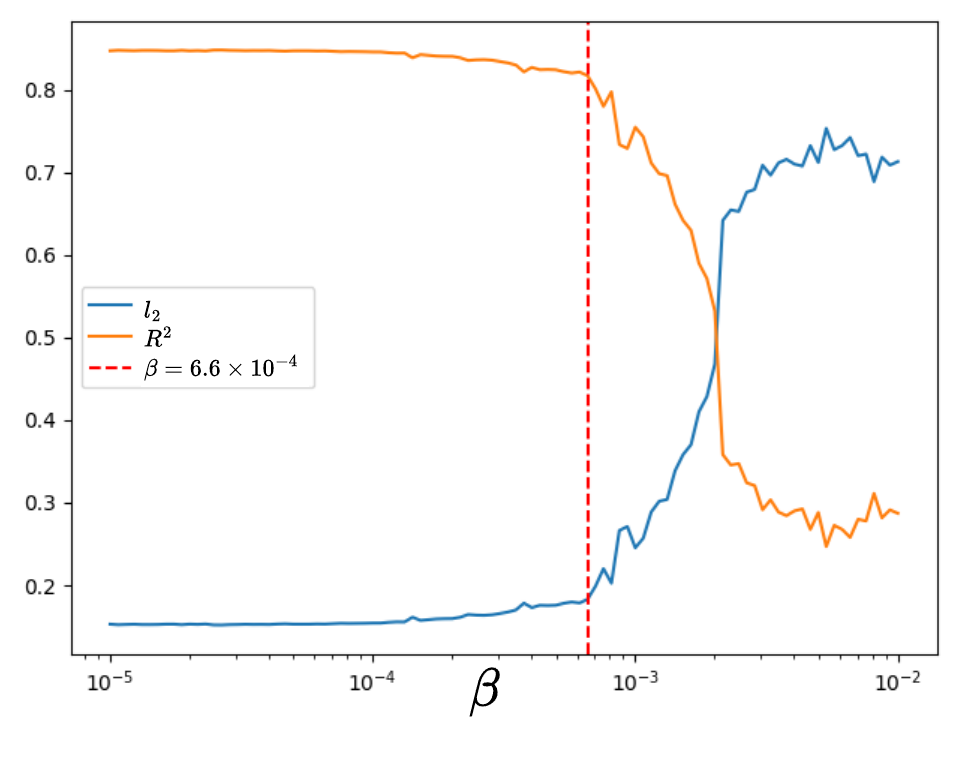}
    \caption{Reconstruction error ($MSE = l_2$) and explained variance ($R^2$) against the strength of the alignment penalty ($\beta$).}
    \label{fig:r2vsbeta}
\end{figure}

\begin{figure}[h]
    \centering
    \includegraphics[width=0.5\linewidth]{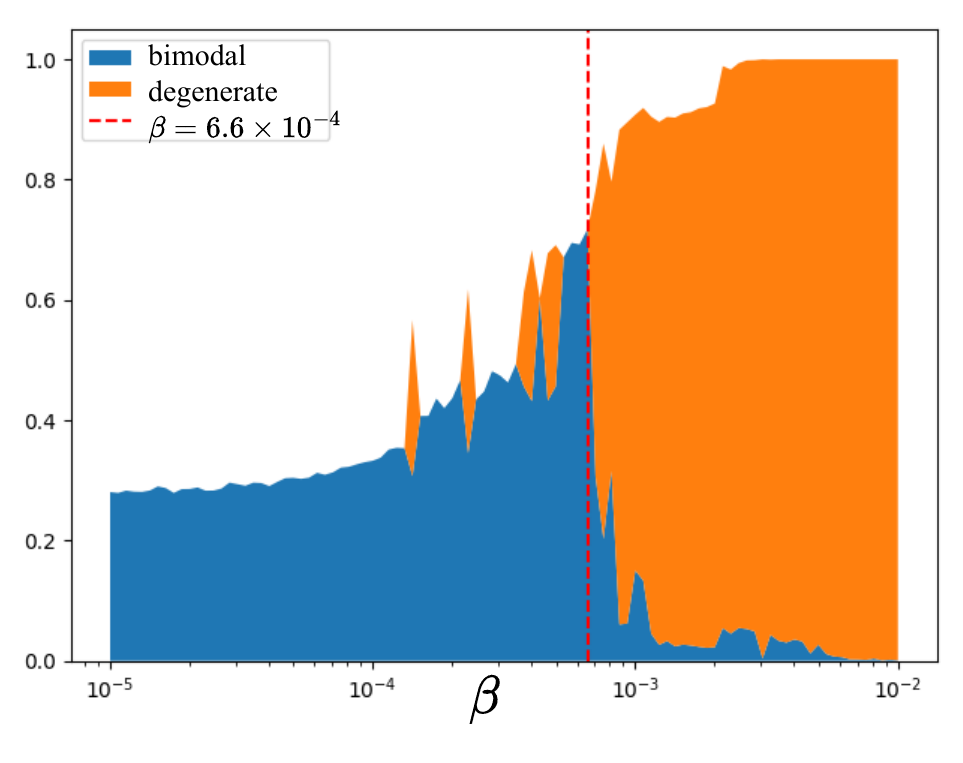}
    \caption{Proportion of energy contained in \bimodal features, and in degenerate features, against the strength of the alignment penalty.}
    \label{fig:Einbvsbeta}
\end{figure}

\section{Toy Data Generating Process}
\label{app:DGP}

\subsection{DGP definition}
\label{app:DGP_def}

For our toy Data Generating Process (DGP), we created two ground truth dictionaries $D_1, D_2$ and sparse codes $Z_1, Z_2$ such that $Z_1D_1=Z_2D_2$. The goal is to have a realistic-enough distribution on which we have full control of which features are "\unimodal" and which are "\bimodal". $(D_1, Z_1)$ will have "ground truth" \bimodal features while $(D_2, Z_2)$ will contain "ground truth" \unimodal features. One parameter, $\tau_1$ controls smoothly the separation between these, with no ambiguity over which of the two is the "true" ground truth when this parameter is extreme. When $\tau_1$ is close to $1$---resp. $\tau_2$---$(D_1, Z_1)$---resp. $(D_2, Z_2)$--- is the ground truth dictionary, without ambiguity. Below is the list of parameters :

\begin{itemize}
    \item $d$, the dimension of the latent space.
    \item $k$, controlling the number of concepts, such that $(D_1, Z_1)$ has $10k$ concepts and $(D_2, Z_2)$ has $14k$ concepts.
    \item \emph{$L$} : the sparsity of the codes, such that $\forall i, L_0(Z_{2,i}) = L$.
    \item $\tau_1$ : controls the cosine similarity between $D_2^I$ and $D_2^T$, such that $\forall i, D_{2,i}^{I, \top}\cdot D_{2,i}^T = \tau_1$
    \item $\tau_2$ : controls the cosine similarity between a matching image-caption pair, such that $\forall i, (Z_{2, i}^I D_2)^\top(Z_{2, i}^T D_2) = \tau_2$
\end{itemize}

where the superscript $\cdot^I$ indicates that we take the submatrices corresponding to image codes in $Z$ or to image-only features in $D$. Similarly for $\cdot^T$ and text.

\paragraph{\bm{$(D_1, Z_1)$}.} Let $D^I$ and $D^T$ be subsets of concepts of size $k$ each and corresponding to high-energy high-frequency \unimodal concepts. Let $D^{B}$ be $4k$ \bimodal concepts, associated to small (depending on $\tau_1$) additional terms $D^{B, I}$ and $D^{B, T}$ of size $4k$ each. These additional terms are modality-specific and control for the fact that features can have some amount of \unimodal information.

We enforce that all features in $D^I$ are orthogonal to all features in $D^T$, $D^B$ and $D^{B, T}$, similarly for $D^{B, I}$, $D^T$ and $D^{B, T}$. All features in $D^B$ are thus orthogonal to all features not in $D^B$.

To generate one sample $Z_{1, i}$, we proceed as follows. We select exactly one feature in $D^I$ and one in $D^T$ that will have an activation of $1$ (temporarily). Then, we select $L-1$ features in $D^B$, and their corresponding additional terms in $D^{B, I}$ and $D^{B, T}$. The activations associated to $D^B$ are constant, as are those in $D^{B, I}$ and $D^{B, T}$, and equal to $\tau_1 \beta$ and $(1-\tau_1)\beta$ respectively. To compute $\beta$, we select the only positive real-valued root of the underlying 4th-degree polynomial to satisfy the constraint given by $\tau_2$. Finally, we scale all activations by $\frac{1}{\norm{Z_{1, i} D_1}}$ such that the final embeddings have unit norm.

\paragraph{\bm{$(D_2, Z_2)$}.} $D_2$ is the collection of features in $D^I$, $D^T$ as well as $\tau_1D^{B} + (1-\tau_1)D^{B, I}$ and $\tau_1D^{B} + (1-\tau_1)D^{B, T}$. Activation strengths in $Z_2$ associated to $B^I$ and $B^T$ are unchanged, and those associated to $\tau_1D^{B} + (1-\tau_1)D^{B, I/T}$ are equal to $\beta$.

\paragraph{} In \cref{app:DGP_experiments}, we use $\tau_2$ to have cosine similarity histograms similar to those obtained on CLIP embeddings of LAION, and for $\tau_1$, we use $\tau_2 + 0.1$ in experiment 1, and $0.999$ in experiment 2.

In \Cref{fig:toy_cosim} we show the histogram of cosine similarities between embeddings generated with CLIP and LAION, as well as those between embeddings generated by our DGP.

\begin{figure}[h]
    \centering
    \includegraphics[width=0.45\linewidth]{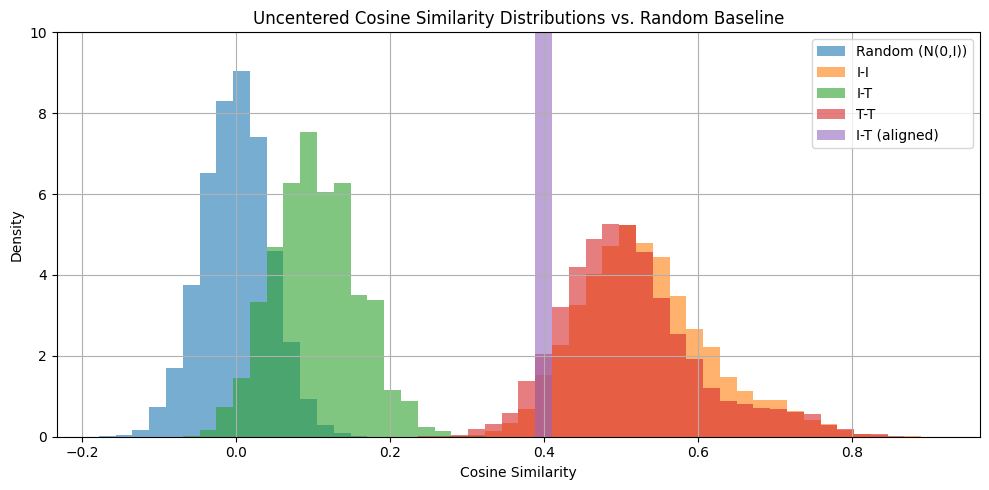}
    \hfill
    \includegraphics[width=0.45\linewidth]{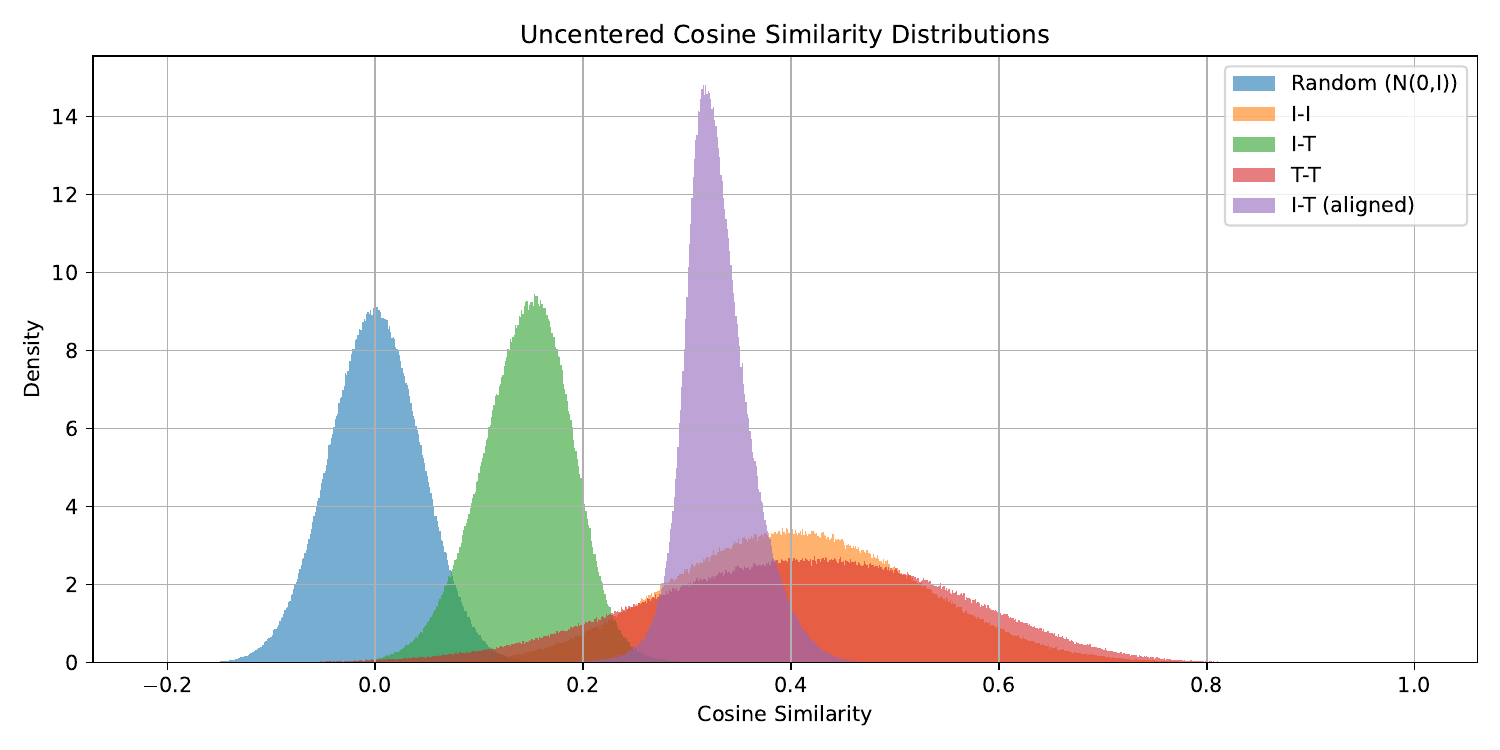}
    \caption{Histograms of cosine similarities between CLIP embeddings (\textbf{right}). "II": image-image embedding pairs, "TT": text-text, "IT": random image-caption, "IT (aligned)": matching image caption, Random: here for reference. \textbf{Left:} cosine similarity histogram for the embeddings generated by our toy DGP.}
    \label{fig:toy_cosim}
\end{figure}

\subsection{Experiments}
\label{app:DGP_experiments}

In this section, we investigate the impact of enforcing our alignment penalty $\mathcal{L}_{\mathrm{align}}$ on the learned dictionaries. For that, we design two toy datasets sampled from two different data-generating processes (DGPs), one satisfying the \Isoenergy~Assumption and one violating it. In both cases, parameters of the DGPs are chosen such that the data distribution is similar to the actual data distribution of LAION-400M embedded by CLIP ViT-B/32. Details on the DGPs and the datasets are given in \cref{app:DGP_def}.

We train our SAEs using the same number of atoms and sparsity as those of the DGPs. We do not report sparse reconstruction metrics, as they are all equally good, with $R^2 \geq 0.99$ and fixed sparsity.

\paragraph{Experiment 1.} In our first experiment, we show that both \SAE{SAE} and \SAEA{SAE-A} are equally able to find the ground truth dictionary in a setup where the \Isoenergy~Assumption is violated. In other words, we check that enforcing the alignment penalty does not lead to learning spurious \bimodal features.

Let $D$ and $Z$ be the learned dictionary and sparse codes, respectively. Let $W$ and $C$ be the ground truth dictionary weights and corresponding sparse codes, respectively. We measure the proximity between both dictionaries using the Wasserstein distance $\mathcal{W}(D, W)$ between the ground truth and the learned dictionary's atoms. We also measure the mean matching accuracy $\mathrm{mma}(Z, C)$, for a permutation matrix $\Pi \in \R^{K \times K}$ that best matches the atoms in $Z$ and $C$. $\mathrm{mma}(Z, C)$ is the average cosine similarity between matched atoms' activation patterns in $Z$ and $C$.

Results are summarized in \cref{tab:experiment_1_results}. We observe no distinction between the two SAEs. In particular, we show that adding the alignment penalty \SAEA{$\mathcal{L}_{\mathrm{align}}$} does not incentivize \SAEA{SAE-A} to learn spurious \bimodal features.

\paragraph{Experiment 2.} In our second experiment, the DGP from which the data is sampled satisfies the \Isoenergy~Assumption. In this case, we observe that training a \SAE{SAE} fails to uncover the ground truth dictionary and learns spurious, \unimodal features, while training a \SAEA{SAE-A} succeeds at doing so.
Results are summarized in \cref{tab:experiment_2_results}.

\begin{table}[H]
    \centering
    \begin{subtable}[t]{0.45\textwidth}
        \centering
        \resizebox{\linewidth}{!}{
        \begin{tabular}{l|cc}
            \toprule
            & $\mathcal{W}(D, W)\,\downarrow$ & $\mathrm{mma}(Z, C)\,\uparrow$\\
            \midrule
            \SAE{SAE}   & $0.197$ & $0.83$ \\
            \SAEA{SAE-A} & $0.185$ & $0.81$ \\
            \bottomrule
        \end{tabular}}
        \caption{Both \SAE{SAE} and \SAEA{SAE-A} learn an equally good approximation of the ground truth dictionary. \textbf{Adding the alignment penalty $\mathcal{L}_{\mathrm{align}}$ does not lead to learning spurious \bimodal features}.}
        \label{tab:experiment_1_results}
    \end{subtable}
    \hfill
    \begin{subtable}[t]{0.45\textwidth}
        \centering
        \resizebox{\linewidth}{!}{
        \begin{tabular}{l|cc}
            \toprule
            & $\mathcal{W}(D, W)\,\downarrow$ & $\mathrm{mma}(Z, C)\,\uparrow$\\
            \midrule
            \SAE{SAE}   & $0.396$ & $0.29$ \\
            \SAEA{SAE-A} & $\bm{0.184}$ & $\bm{0.52}$ \\
            \bottomrule
        \end{tabular}}
        \caption{\textbf{\SAE{SAE} fail to learn ground truth \bimodal atoms}. \SAEA{SAE-A} learns the ground truth dictionary.}
        \label{tab:experiment_2_results}
    \end{subtable}
    \caption{Comparision between \SAE{SAE} and \SAEA{SAE-A} (trained with $\mathcal{L}_{\mathrm{align}}$). $\mathcal{W}(D, W)$ : Wasserstein distance between learned and ground truth dictionaries (lower is better). $\mathrm{mma}(Z, C)$ : mean matching accuracy between learned and ground truth sparse codes (higher is better).
    \textbf{(a)} Experiment 1: the data-generating process violates the \Isoenergy~Assumption. \textbf{(b)} Experiment 2: the data-generating process satisfies the \Isoenergy~Assumption.}
    \label{tab:experiment_results}
\end{table}

\section{Model-dex}
\label{app:models}

We studied a collection of models, varying in size and training procedure from the CLIP \citep{radford2021}, SigLIP \citep{zhai2023sigmoid}, and OpenCLIP \citep{cherti2023reproducible} implementations.

\paragraph{CLIP and OpenCLIP.} We used the ViT-B/32 and ViT-L/14 variants of both CLIP and OpenCLIP. These were trained with the contrastive softmax loss to align matching pairs and separate non-matching ones. Specifically, we used {\texttt{openai/clip-vit-base-patch32}}, {\texttt{openai/clip-vit-large-patch14}}, {\texttt{laion/CLIP-ViT-B-32-laion2B-s34B-b79K}} and {\texttt{laion/CLIP-ViT-L-14-laion2B-s32B-b82K}} from hugging face's {\texttt{transformers.CLIPModel}} implementation. They are referred to as CLIP, CLIP-L, OpenCLIP, and OpenCLIP-L.

\paragraph{SigLIP.} We conducted our experiments using both SigLIP, trained with a sigmoid-based contrastive loss, and SigLIP2, trained with a heterogeneous mix of objectives. Specifically, we used {\texttt{google/siglip-base-patch16-224}} and {\texttt{google/siglip2-base-patch16-224}} from hugging face's {\texttt{transformers.SiglipModel}} implementation.\\

We report in \Cref{fig:pcas} linear projections of embeddings, colored by modality, illustrating the modality gap. We measure the modality gap both as the Difference in Mean (DiM) between each modality and as the difference in distribution, with the Wasserstein distance ($\mathcal{W}$). We repeat this on three datasets (\cref{app:datasets}). We report the modality gap inside of each dataset in \cref{tab:model_bestiary}.

Finally, we report recall@1 and recall@5, along with histograms of cosine similarities between embeddings of the same modality, matching embeddings of different modalities, and random pairs of embeddings of different modalities in \Cref{fig:histograms,tab:model_bestiary}. We add a histogram of cosine similarity between samples from white noise as a reference.

For the ImageNet dataset, we also measure zero-shot recall---slightly different from recall@1, see \cref{app:metrics}---and classifier accuracy.

\begin{figure}[h]
  \centering
  \includegraphics[width=.98\textwidth]{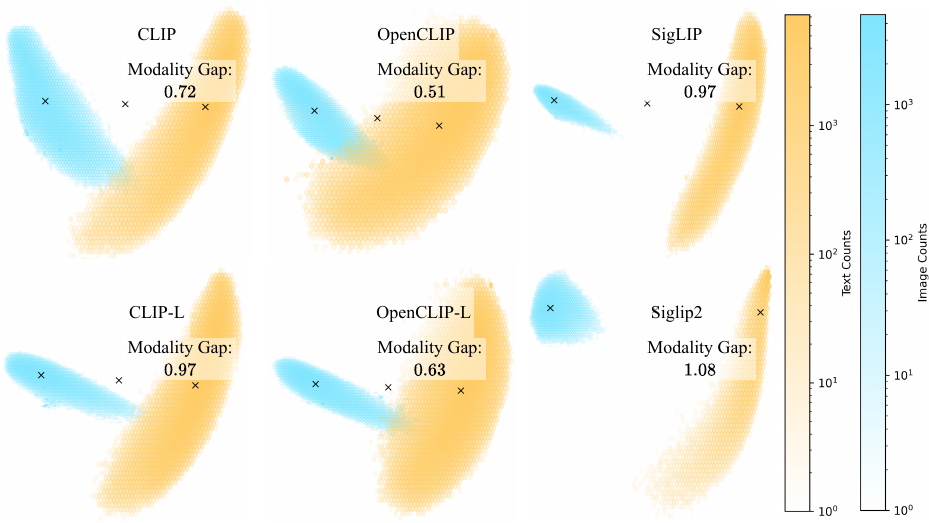}
  \caption{PCA of image (blue) and text (yellow) embeddings from LAION for different VLM encoders. The modality gap is measured as the difference in mean (DiM) and as the Wasserstein distance ($\mathcal{W}$) between the two distributions. A cross indicates the modality-wise and global means.}
  \label{fig:pcas}
\end{figure}

\begin{figure}[h]
  \centering
  \includegraphics[width=.98\textwidth]{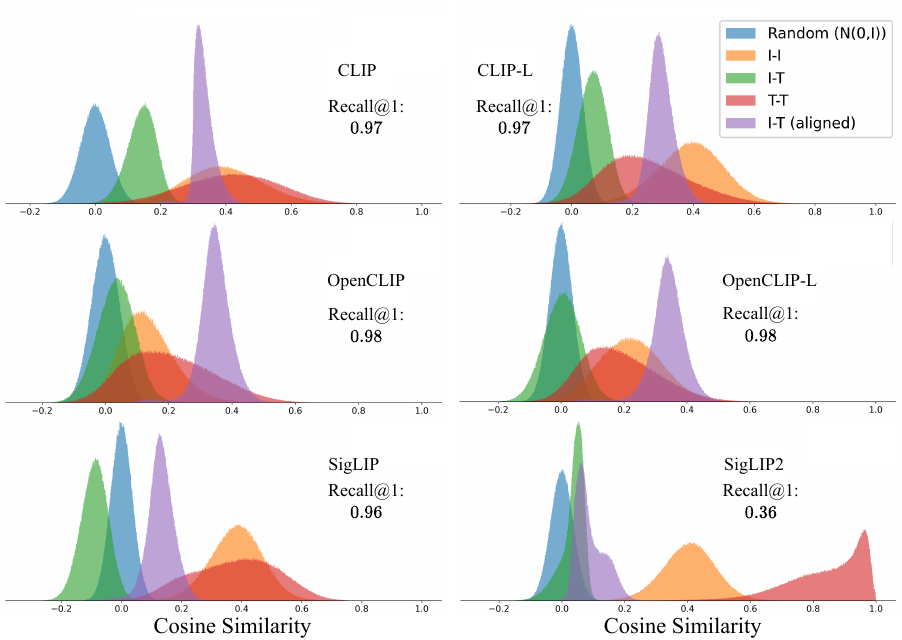}
  \caption{Histograms of cosine similarities between pairs of embeddings from the same modality (I : image, T : text), random pairs of opposing modality (I-T), and matching image-caption pairs (I-T (aligned)). Cosine similarities between samples from a white noise are indicated as a reference. We only consider LAION embeddings here. We indicate the corresponding recall@1 performance on each plot.}
  \label{fig:histograms}
\end{figure}

\begin{table}[H]
    \centering
    \resizebox{\linewidth}{!}{
    \begin{tabular}{lcccccc}
        \toprule
        \textbf{} & \textbf{CLIP} & \textbf{CLIP-L} & \textbf{OpenCLIP} & \textbf{OpenCLIP-L} & \textbf{SigLIP} & \textbf{SigLIP2} \\
        
        \midrule
        \multicolumn{6}{l}{\textbf{LAION}} \\
        \cmidrule(lr){1-1}
        \quad $\mathrm{DiM} \mid \mathcal{W}$ & $0.72~\mid~{\gray0.68}$ & $0.70~\mid~{\gray0.73}$ & $0.51~\mid~{\gray0.70}$ & $0.63~\mid~{\gray0.75}$ & $0.97~\mid~{\gray0.94}$ & $1.08~\mid~{\gray0.90}$ \\
        \quad $\mathrm{recall@1}\mid@5$   & $0.97~\mid~{\gray0.99}$ & $0.97~\mid~{\gray0.99}$ & $0.98~\mid~{\gray0.99}$ & $0.98~\mid~{\gray0.99}$ & $0.96~\mid~{\gray0.99}$ & $0.36~\mid~{\gray0.46}$ \\
        
        \midrule
        \multicolumn{6}{l}{\textbf{COCO}} \\
        \cmidrule(lr){1-1}
        \quad $\mathrm{DiM} \mid \mathcal{W}$ & $0.82~\mid~{\gray0.67}$ & $0.81~\mid~{\gray0.72}$ & $0.72~\mid~{\gray0.65}$ & $0.75~\mid~{\gray0.68}$ & $1.04~\mid~{\gray0.90}$ & $1.05~\mid~{\gray0.86}$ \\
        \quad $\mathrm{recall@1}\mid@5$   & $0.71~\mid~{\gray0.95}$ & $0.74~\mid~{\gray0.96}$ & $0.77~\mid~{\gray0.97}$ & $0.81~\mid~{\gray0.98}$ & $0.81~\mid~{\gray0.98}$ & $0.81~\mid~ {\gray0.97}$ \\
        
        \midrule
        \multicolumn{6}{l}{\textbf{ImageNet}} \\
        \cmidrule(lr){1-1}
        \quad $\mathrm{DiM} \mid \mathcal{W}$ & $0.86~\mid~{\gray0.68}$ & $0.86~\mid~{\gray0.73}$ & $0.77~\mid~{\gray0.67}$ & $0.84~\mid~{\gray0.69}$ & $1.13~\mid~{\gray0.93}$ & $1.13~\mid~{\gray0.89}$ \\
        \quad $\mathrm{recall@1}\mid@5$   & $0.65~\mid~{\gray0.91}$ & $0.74~\mid~{\gray0.95}$ & $0.67~\mid~{\gray0.92}$ & $0.73~\mid~{\gray0.95}$ & $0.75~\mid~{\gray0.96}$ & $0.76~\mid~{\gray0.97}$ \\
        \quad zero-shot~accuracy   & $0.60~\mid~{\gray0.85}$ & $0.73~\mid~{\gray0.92}$ & $0.63~\mid~{\gray0.86}$ & $0.72~\mid~{\gray0.91}$ & $0.73~\mid~{\gray0.92}$ & $0.75~\mid~{\gray0.93}$ \\
        \quad Classifier~acc   & $0.75~\mid~{\gray0.94}$ & $0.84~\mid~{\gray0.97}$ & $0.77~\mid~{\gray0.94}$ & $0.83~\mid~{\gray0.97}$ & $0.83~\mid~{\gray0.97}$ & $0.84~\mid~{\gray0.97}$ \\
        
        \bottomrule
    \end{tabular}}
    \caption{Modality Gap and recall on three datasets for the six models studied. For ImageNet, zero-shot and classifier accuracy show downstream task performance.}
    \label{tab:model_bestiary}
\end{table}

\section{Metrics}
\label{app:metrics}
\subsection{VLM encoder metrics}

\paragraph{Recall@k.} Let $I, T \in R^{b \times d}$ be normalized embeddings of matching image and captions. Let $C = I \cdot T^\top \in [-1, 1]^{b \times b}$ be the matrix of cosine similarities between these embeddings. Let $\bm l$ be the vector of ground truth matching labels : $\bm l=\intint{0}{b-1} \in \N^b$, and $\mathrm{topk}^{(\mathrm{img})} \in \N^{b\times k}$ be such that $\mathrm{topk}^{(\mathrm{img})}_{i, j}$ is the $j$-th best match for the image embedding $i$, and likewise for $\mathrm{topk}^{(\mathrm{txt})}$. We can now define $\mathrm{recall@k} = \frac{1}{2b} \sum_{i} \left(\bm{1}_{i \in \mathrm{topk}^{(\mathrm{img})}_i} + \bm{1}_{i \in \mathrm{topk}^{(\mathrm{txt})}_i}\right)$. In other words, it measures the accuracy of the model at matching image-caption pairs.
As this metric depends on the batch size $b$, we will fix it to $256$ throughout this study.

\paragraph{Zero-shot accuracy.} The goal of this metric is to measure the quality of multimodal embeddings through downstream task performance. In the context of image classification, text embeddings are constrained to a small number of points, corresponding to simple descriptions of the class---e.g., "an image of a \{class\}". Following \citet{radford2021}, we generate about 80 text descriptions for each ImageNet class, embed them, take their average, and finally normalize this average to get the class classifier. By concatenating all of these embeddings, we get a linear classifier, and zero-shot accuracy is the accuracy of this classifier. As for recall, "@k" means that the prediction is correct if the correct class is among the $k$ best matches.

\paragraph{Classifier Accuracy.} Alternatively, we train a linear classifier on the image embeddings instead of taking the one given by the embedding of these basic textual class descriptions. This metric doesn't consider the multimodal aspect of the embeddings anymore; it only considers whether the embeddings contain the information to classify.

\paragraph{Retrieval.} We evaluate retrieval performance, another downstream task, on the FashionIQ dataset (\cref{app:datasets}). To do so, we consider the average of recall@10 and recall@50, as in \citet{wu2021fashion}. Instead of matching images $I$ and texts $T$, we now have queries $Q$ and targets $T$. Queries are functions of the source image's embedding and the relative caption's embedding, typically the normalized sum. Targets are the target image's embedding.

The goal is to quantify the quality of semantic vector arithmetic. Qualitatively, the embedding of an image of a red shoe $\bm{i}_{\mathrm{src}}$ plus the embedding of the relative description \emph{"is not red but blue"} $\bm{\Delta}$ should match the embedding of an image of a blue shoe $\bm{i}_{\mathrm{src}} + \bm{\Delta} \sim \bm{i}_{\mathrm{trg}}$. Quantitatively, we evaluate both if our queries have high recall---meaning they can retrieve images of blue shoes---and whether they belong to the distribution of images---meaning they are actually embeddings of blue shoes---using the distance to the K-th Nearest Neighbor, a standard out-of-distribution (OOD) detection method \citep{sun2022out,yang2022openood}.

\subsection{SAE metrics}

\subsubsection{General metrics}

\paragraph{Sparse Reconstruction.} These metrics measure reconstruction faithfulness and the sparsity of the codes. While a perfect reconstruction indicates the SAE likely fits noise, a faithful reconstruction of the activations is necessary to have captured the latent conceptual structure.

\begin{itemize}
    \item $\mathrm{MSE} = \E_i\left[\|A_i - \widehat{A}_i\|_2^2\right]$, where $\widehat{A}_i$ is the reconstruction of $A_i$ by the SAE. This is a direct measure of the reconstruction error, quantifying how closely the autoencoder reproduces the original activations.
    
    \item $R^2 = \E_i\left[1 - \frac{\|A_i - \widehat{A}_i\|_2^2}{\|A_i - \bar{A}\|_2^2}\right]$, where $\bar{A}$ is the average activation in $A$. The $R^2$ score is a measure of the explained variance. It is independent of the norm of the activations, focusing on the proportion of variance captured by the model on each instance.
    
    \item $L_0 = \E_i\left[\|z_i\|_0\right]$, where $z_i$ is the encoding of $A_i$ by the SAE, and $\|\cdot\|_0$ counts the number of nonzero elements. This metric captures the average sparsity level in the encoded representation.
    
    \item $L_1 = \E_i\left[\|z_i\|_1\right]$, where $z_i$ is the encoding of $A_i$. This is the average $L_1$ norm of the codes and serves as a differentiable proxy for $L_0$ sparsity. It penalizes large activations and encourages sparse representations, and is commonly used as a regularization term. In some architectures, a term similar to $L_1$ is included in the loss to promote sparsity, while in others, where sparsity is directly controlled, it is not necessary. We report it for all architectures for comparison.
\end{itemize}

\paragraph{Consistency.} Measures whether the dictionary is well-grounded and functionally relevant.

\begin{itemize}
    \item Interventions: $C\mathrm{-insertion}$, $C\mathrm{-deletion}$ \citep{fel2023holistic} We include these metrics for completeness, even though they are not very useful in our case. For all data points, they sort active features by attribution and either insert or delete them one by one, measuring the downstream effect.
    Attribution is relative to a scalar, and we used the reconstruction error in this case. This is not very informative, as activation is then a very good proxy for attribution due to the near orthogonality of features. What these metrics tell us is essentially how fast the activation decreases if you sort features by activation.
    An arguably more useful way of computing these metrics would have been w.r.t. the cosine similarity between matching pairs or w.r.t. the recall. But again, it would not be very informative, and they would essentially become weaker versions of $\delta_{\mathrm{r@1}}$.
    \item Stability measures the average distance between learned dictionaries under different random seeds. It is traditionally believed that lower is better, as it means the algorithm and its inductive biases solve the unidentifiability problem. However, it does not, by any means, imply that the found solution is "good" or is representative of the "true" conceptual structure.
\end{itemize}

\paragraph{Structure in $\bm D$.} Organization of the learned dictionary atoms.

\begin{itemize}
    \item $\mathrm{Stable Rank}(M) = \frac{\norm{M}_F^2}{\norm{M}_2^2}$. Here, $\norm{M}_F^2 = \sum\sigma_i^2$ and $\norm{M}_2^2 = \sigma_{\max}^2$, where $\sigma_i$ are the singular values of $M$. It gives both the number of significant singular values in $M$ and a robust rank estimation. Due to numerical and floating-point errors, almost all high-dimensional matrices are full rank.
    \item $\mathrm{EffectiveRank}(M)=e^{H(p)}$, where $H(p)=-\sum p_i \log p_i$ is the entropy of the probability distribution $p$, and $p_i=\frac{\sigma_i}{\sum\sigma_j}$. It captures information-theoretic diversity in $M$, measuring the spread of the spectrum. It is the number of informative singular directions in the matrix. Both a $\mathrm{Stable Rank}(M)$ and an $\mathrm{Effective Rank}(M)$ close to 1 suggest a few dominant directions, while values close to $\mathrm{rank}(M)$ indicate all directions are equally important.
\end{itemize}

When analyzing $D$, higher stable and effective ranks are better, which suggests high diversity of concepts.

\begin{itemize}
    \item $\mathrm{Coherence}=\max_{i\neq j} D_i^\top D_j$ is the maximum cosine similarity between two dictionary atoms. A score close to $1$ indicates possibly redundant concepts, while lower scores tend to reflect a better spread of concepts. In practice, this metric seems to be a measure of the inductive bias against learning close, undistinguishable features, with all SAEs but MP in \cref{app:SAE_architecture} having this bias and achieving similar scores, and MP not having this bias and having a much higher score. See \cref{app:unimodal_biases} for more details on MP.
\end{itemize}

\paragraph{Structure in $\bm{Z^\top Z}$.} In this case, we consider $Z\in\R^{N\times K}$ with normalized columns, to focus on co-activation structure, not just activation magnitude.

\begin{itemize}
    \item $\mathrm{Connectivity}=1 - \frac{1}{K^2}\norm{Z^\top Z}_0$ counts the number of unique co-activations among all possible pairs. Higher connectivity scores suggest a wide range of possible meaningful concept combinations, while low connectivity implies a sparse graph of concept combinations, with possibly more structure in this second case.
    \item $\mathrm{Stable Rank}$ and $\mathrm{Effective Rank}$ measure the structure in $Z^\top Z$, with low scores indicating highly structured co-activations, with the presence of groups or blocks of concepts behaving in similar manners. High scores suggest little to no structure.
    \item $\mathrm{NegativeInterference} = \norm{\mathrm{ReLU}\left[-(Z^\top Z) \odot(DD^\top)\right]}_2$ quantifies how much concepts tend to cancel each other out.
\end{itemize}

\subsubsection{Modality specific metrics}
\label{app:modality_metrics}

In addition to these general metrics, we introduce multimodality-specific metrics, motivated by the \Isoenergy~Assumption and the phenomenology of VLM embeddings. We report:

\begin{enumerate}
    \item \emph{Probing accuracy} $p_{\mathrm{acc}}$, which quantifies how the geometric organization of dictionary atoms aligns with known structures in latent space.
    \item \emph{Functional alignment} $\rho$, an instance-level ratio indicating whether cross-modal alignment is carried by \bimodal rather than \unimodal features.
    \item \emph{Functional and Distributional Agreement} (FDA), a distribution-level analogue of $\rho$ that checks consistency of this functional role across populations.
    \item $\delta_{\mathrm{r@1}}$, an interventional measure of how much retrieval recall degrades when \unimodal atoms are ablated.
\end{enumerate}

Together, these four views (geometric, functional, distributional, and causal) consistently probe complementary aspects of how features support cross-modal alignment.
While $p_{\mathrm{acc}}$ gives a purely geometric measure of consistency between structures in concepts and embeddings, $\rho$ and FDA link the geometric aspect of alignment to the functional behavior of each feature.
$\rho$ and FDA are two sides of the same coin: while $\rho$ operates at the instance level, capturing how cross-modal alignment is distributed across features functionally, FDA operates at the distribution level and shows whether this functional role is consistent with the global geometry of the two modalities, in a way that is more consistent with the \Isoenergy~Assumption.
$\delta_{\mathrm{r@1}}$ complements these functional views by providing a causal perspective on whether \bimodal features alone are sufficient to maintain alignment.

\paragraph{Energy.} Energy $E$ gives a measure of the importance of each feature in the dictionary. See \Cref{fig:scatter_high} for a detailed distribution of energy across features and modality.
\vspace{-2mm}
\begin{align}
\label{eq:energy}
\forall \domain \in \Domains \quad E^{(\domain)} = \E \left( \encoder^{(\domain)}(\Xvec)^2 \mid \Xvec \in \mathcal{X}^{(\domain)} \right)
 \qquad\qquad E = \frac{1}{|\Domains|} \sum_{\domain \in \Domains} E^{(\domain)} \in \R^{K}
\end{align}
\vspace{-4mm}

where $\encoder$ is the composition of the VLM and the SAE. Let $\mathbf{a} \coloneqq \frac{E^{(\img)}}{\norm{E^{(\img)}}_1}$ and $\mathbf{b} \coloneqq \frac{E^{(\txt)}}{\norm{E^{(\txt)}}_1}$ be probability distributions for the image and text domains, respectively.

\paragraph{Modality Score.}
Before turning to the metrics themselves, we introduce the modality score, and a binary mask indicating whether a feature is \bimodal. The \emph{modality score} of feature $i$ is a direct measure of the component-wise \Isoenergy~Assumption:
$\mu_i = \frac{E_i^{(\img)}}{E_i^{(\img)} + E_i^{(\txt)}} \in [0,1]$.
Values $\mu_i \approx 1$ or $\mu_i \approx 0$ correspond to image-only and text-only features, while $\mu_i \approx 0.5$ indicates balanced activation energy across modalities. We refer to the former as \unimodal features and the latter as \bimodal features. To make this distinction operational, we introduce a binary mask $\bm{\delta}_i$ indicating whether feature $i$ is classified as \bimodal. In practice, $\bm{\delta}_i$ is obtained by thresholding $\mu_i$ within an interval $[\tau, 1-\tau]$ around $0.5$, with $\tau=0.05$. We use a bridge-based consistency criterion to define $\tau$, which we introduce later.

\paragraph{Linear Structure.}\label{metric:pacc}We evaluate how the geometric organization of dictionary atoms aligns with the known structure of the modality gap. In \Cref{fig:energy_distribution}, a linear projection of feature vectors reveals three distinct clusters corresponding to image-only, text-only, and \bimodal features. This geometry of concept organization closely aligns with the geometry of embeddings and the modality gap~: image--only features are aligned with the image cone, similarly for text, while \bimodal features are orthogonal to modality information. We define \emph{probing accuracy} $p_{\mathrm{acc}}$ to quantify how well the geometric structure of concepts aligns with embedding structure and the modality gap. High $p_{\mathrm{acc}}$ indicates strong alignment.

Each feature $D_i$ is considered as a linear probe classifying embeddings. \textit{Unimodal} features are expected to be highly predictive, yielding high accuracy, whereas \bimodal features should be domain-agnostic, yielding an accuracy near $0.5$.
The score $s_i$ of \unimodal feature $i$ is simply its accuracy $acc$. For \bimodal features, $s_i=1 - 2 * (\max(\mathrm{acc}, 1 - \mathrm{acc}) - 0.5)$ is close to 1 if they are consistently anti-informative (i.e. good at being bad), and close to $0$ otherwise.
Finally, $p_{\mathrm{acc}}$ is computed as the energy-weighted average of the $s_i$ over all features:
$p_{\mathrm{acc}} = \frac{\sum_i E_i \cdot s_i}{\sum_i E_i}$.

\paragraph{Bridge.} Some features naturally act as \emph{bridges} between modalities, carrying information that contributes to cross-modal alignment. Intuitively, these are the \bimodal features: they are simultaneously active in both image and text embeddings and facilitate matching across modalities. Following \citet{papadimitriou2025}, we define the \emph{bridge} matrix $B_W$ as $B_W = W \odot (D \cdot D^\top) \in \R^{K \times K}$
where $W \in \R^{K \times d}$ is a weight matrix. They used $B_\Sigma$, where $\Sigma$ is the cross-covariance of activations across image-text pairs from a joint empirical distribution $\gamma$ (\Cref{eq:cross_covariance}).

\begin{equation}
    \label{eq:cross_covariance}
    \Sigma = \underset{{(\Xvec, \Xvec') \sim \gamma}}{\E}\left(\encoder^{(\mathrm{img})}(\Xvec)^\top\encoder^{(\mathrm{txt})}(\Xvec')\right)
\end{equation}

\begin{align}
    \label{eq:ot_plan}
    \Gamma = \underset{\gamma}{\arg\min} \|\gamma \odot C\|_F\qquad\text{s.t.}\quad \gamma \mathbf{1} = \mathbf{a},\\\qquad\qquad\qquad\qquad\quad\mathbf{1}^\top \gamma = \mathbf{b}^\top,~\gamma \ge 0
\end{align}

We also consider $B_\Gamma$, where $\Gamma$ is the optimal transport plan (\Cref{eq:ot_plan}) between the image- and text-weighted dictionaries, using weights $\mathbf{a}$ and $\mathbf{b}$, and a cost matrix $C = 1 - D \cdot D^\top$ (equivalently, both cosine distance and half squared Euclidean distance since atoms are unit norm).
While $B_\Sigma$ captures instance-level co-activation patterns under the contrastive objective, $B_\Gamma$ captures alignment at the distribution level, making it more consistent with the \Isoenergy~assumption and robust to instance-level mismatches. We show how to chose $\tau$ based on $B_\Sigma$ in \cref{app:modality_metrics}.

\paragraph{Functional Alignment.}\label{metric:rho}$B_{\Sigma}$ captures \emph{functional alignment}, the combined functional similarity and geometric alignment of features across modalities. It reflects each feature pair's contribution to cross-modal alignment under the contrastive objective, with its entries providing pairwise alignment score.
Thus we define $\rho$ the ratio between total $B_\Sigma$ bridge score adjacent to \bimodal features ---i.e. for all pairs $(i, j)$ where at least one of $i$ or $j$ is a \bimodal feature--- and that adjacent only to \unimodal features ---i.e. both $i$ and $j$ are \unimodal. $\rho>1$ means functional alignment is primarily done by \bimodal features; $\rho<1$ means \unimodal pairs spuriously carry it.

\paragraph{Distributional Alignment.}\label{metric:FDA}In contrast, $B_{\Gamma}$ captures \emph{distributional alignment} between modalities. Its total mass is $1-c$, where $c$ is the transport cost. The entries of $B_{\Gamma}$ reveal which features contribute to balancing or imbalancing the alignment.
A high $B_{\Gamma, i, j}$ indicates features $i$ and $j$ jointly contribute to aligning the two distributions. A low score is harder to interpret: it may reflect features contributing to imbalance (if $\Gamma_{i,j}\neq0$) or unmatched pairs by the transport plan.
\textit{Bimodal} features should belong to both modalities' distribution (particularly matching mostly to themselves) and thus drive distributional alignment. On the other hand, \unimodal features should be specific to it's corresponding modality's distribution and thus should not participate in the overall distributional alignment.
This motivates the \textbf{\boldmath Functional and Distributional Agreement ($\mathrm{FDA}$)} ratio:
\begin{equation}
\label{eq:rho}
\mathrm{FDA} = \frac{\alpha}{\varepsilon}\frac{1-\varepsilon}{(1-c) - \alpha}\text{, where }\alpha = \sum_{i, j \mid \bm{\delta}_i \lor \bm{\delta}_j} B_{\Gamma, i, j}\text{ and }\varepsilon = \frac{\sum_{i \mid \bm{\delta}_i} E_i}{\sum_{i} E_i}
\end{equation}
Here, $\alpha$ is the sum of bridge scores adjacent to a \bimodal feature, and $\varepsilon$ is the proportion of energy in \bimodal features.
$\mathrm{FDA}$ is a mirror of the instance--level $\rho$. It summarizes the impact of \bimodal and \unimodal features on the distributional alignment between the two modalities. The higher the $\mathrm{FDA}$, the better the dictionary is at capturing features whose functional behavior consistently matches their geometric distribution across modalities. Indeed, \bimodal features should always carry most of the alignment \emph{relative to their overall energy}.
A low FDA means either \bimodal features fail to align distributions properly or \unimodal features are spuriously filling this role, indicating a breakdown in the shared conceptual structure.

\paragraph{Intervention.}\label{metric:deltar}After these two functional measures, we turn to the causal behavior of \bimodal feature, asking if they alone are sufficient to explain the recall performance of the embeddings. To do so, we define $\delta_{\mathrm{r@1}}$ as the difference in recall between embeddings reconstructed with the full dictionary and those reconstructed by filtering out \unimodal features.
Let $A \in \R^{N\times d}$ be a set of activations, $Z \in \R^{N \times K}$ the set of sparse codes obtained by encoding $A$ with our learned encoder, $\bm{\delta} \in \{0, 1\}^{N \times K}$ be a binary mask filtering out \unimodal features, broadcasted across the rows of $Z$. Then, let $\widehat{A} \coloneqq ZD$ be the set of reconstructed activations, and $\widetilde{A} \coloneqq (Z \odot \bm{\delta})D$ be the set of \bimodal activations. We finally measure $\delta_\mathrm{r@1}$ as the difference in recall between the embeddings given by $\widehat{A}$ and~$\widetilde{A}$.\\

\paragraph{Thresholding $\bm{\mu}$.}These four metrics depend on the threshold chosen to decide which features are \bimodal, text-only, and image-only. See \Cref{fig:bridge_vs_eps} for a visualization of $\rho$ across the full range of possible thresholds. This figure also illustrates how to choose the threshold, depending on the region's size, involving bridges between two image-only or two text-only features. As soon as this region becomes significant, \unimodal features can't be considered unimodal anymore. As we can see, our results are not very sensitive to the threshold. This is explained by the fact that most modality scores are extreme, either close to 0, 0.5 or 1. This is best visualised through the domain-wise distribution of energy (\Cref{fig:scatter_high}), where we can see clearly defined modes that are not sensitive to the threshold.\\

\begin{figure}[h]
    \centering
    \includegraphics[width=1.\linewidth]{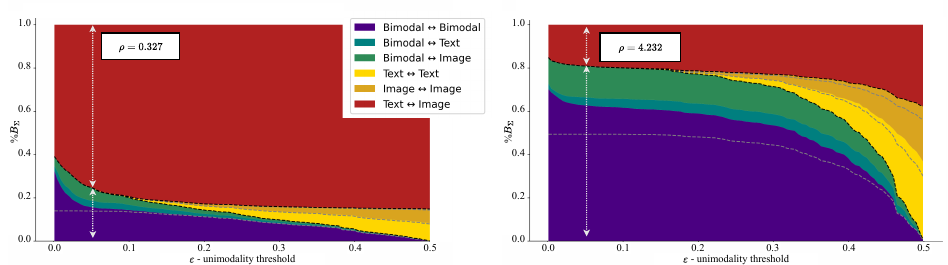}
    \caption{Bridge $B_\Sigma$ mass repartition between \bimodal, image-only, and text-only features, depending on the unimodality threshold.}
    \label{fig:bridge_vs_eps}
\end{figure}

\Cref{fig:distributional_alignment} shows an illustration of the $\mathrm{FDA}$ score, with both an example of a high score on a \SAEA{SAE-A} where we can clearly see that \bimodal features are very stable while image-only are transported to text-only features at considerable cost, and a low score on a classical \SAE{SAE}, where there is no clear dynamic between all three types of features, especially in the center cluster where all three types seem to be mixed.

\begin{figure}[h]
  \centering
  \includegraphics[width=.9\textwidth]{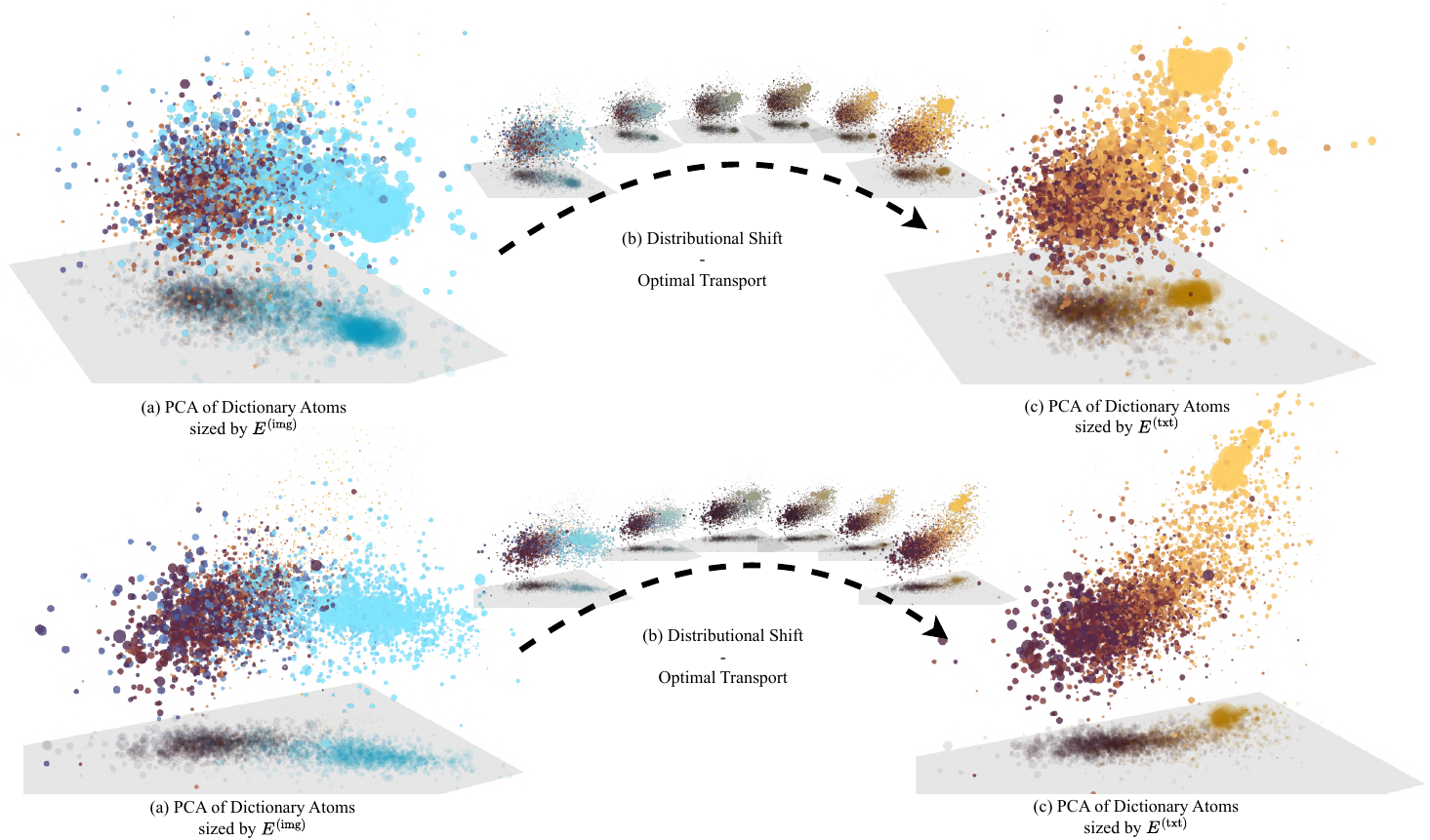}
  \caption{\textbf{Features of \SAEA{SAE-A} (Bottom, FDA = 4.559) functionally behave in accordance with their geometric repartition, which is less the case in \SAE{SAE} (Top, FDA = 2.630)}. The left and right plots show 3D projections of the feature distributions, colored by modality score, sized by energy on the image (\emph{left}) and text (\emph{right}) domains. The center part displays optimal transport between the two domains' distribution, illustrating the FDA score.}
  \label{fig:distributional_alignment}
\end{figure}

We can see in \Cref{fig:linear_structure} the details of the $p_{\mathrm{acc}}$ score distribution across each feature, both for an \SAE{SAE} and a \SAEA{SAE-A}.

\begin{figure}[h]
  \centering
  \includegraphics[width=1.\textwidth]{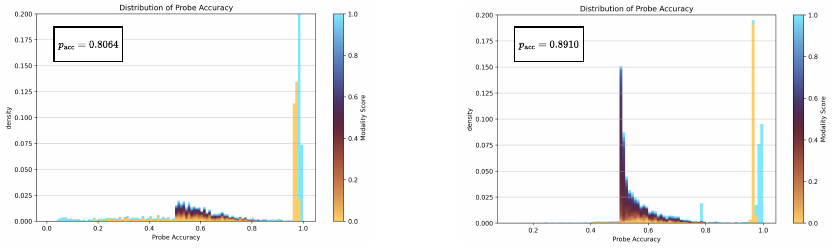}
  \caption{\textbf{Features in \SAEA{SAE-A} (right) outperform those in \SAE{SAE} (left) at geometrically aligning with modality-specific or agnostic information}. Distribution of probing accuracy for each feature based on its modality score. Each feature is evaluated as a probe classifying the modality of embeddings. \textit{Unimodal} features should have an accuracy of $1$, while \bimodal features should have an accuracy of $0.5$. $p_{\mathrm{acc}}$ summarizes that in a single metric---higher is better.}
  \label{fig:linear_structure}
\end{figure}

\subsection{Optimizing for $p_{\mathrm{acc}}$}
\label{app:pacc}

Let's dive into why optimizing for $\mathcal{L}_{\mathrm{align}}$ does not optimize for $p_{\mathrm{acc}}$. Before presenting the technical reasoning, it is important to clarify the motivation for making this distinction.

We consider each feature vector $D_i$ as a probe on embeddings $A$ by computing the dot product between $D_i$ and every embedding $A_j$~: $\mathrm{scores} = A \cdot D_i \in \R^N$, and then taking binary classification predictions as $\mathrm{preds}=\bm{1}_{scores>0} \in \{0, 1\}^{N}$. Given some desired labels $l \in \{0, 1\}^N$, we can thus compute the accuracy of $D_i$. Precisely, the scores are corrected to have a mean of $0$, which is equivalent to, and done by, centering $A$. The reason for this is that we know exactly half of the labels are $0$s and half are $1$s.

Let $l^{(\mathrm{img})}$ be labels indicating which embeddings are from the image domain. For an image-only feature, with $\mu > 1-\tau$, we take $l = l^{(\mathrm{img})}$. For a text-only feature, with $\mu < \tau$, we take $l = 1-l^{(\mathrm{img})}$. Finally, as explained above, for \bimodal features, we take $l=l^{(\mathrm{img})}$ and apply a transformation to the probe's accuracy such that the resulting score is close to 0 when the feature acts as a reliable (text or image) classifier, and approaches 1 when the classification accuracy is near chance (i.e., 0.5): $\mathrm{acc_{bimodal}}=1-2*(\max(\mathrm{acc}, 1-\mathrm{acc})-0.5)$.

If we consider the distributions of $I \cdot D_i$ and $T \cdot D_i$, it can appear that $\mathcal{L}_{\mathrm{align}}$ pushes these distributions to be equal, by enforcing an equal activation on both elements of a matching image-caption pair. However, this is only true when $z_i > 0$ for at least one of the image or the caption. In other words, if you consider only a subset of $I$ and $T$ where $D_i$ is active, then it is true that $\mathcal{L}_{\mathrm{align}}$ penalizes distribution mismatches.

Nevertheless, as features are sparse and have very low frequency, this effect is negligible on the total distributions, as illustrated by \Cref{fig:app_pacc}.

Moreover, all of that is applicable only for \bimodal features. \textit{Unimodal} features are not affected at all by this effect, and yet, we observe that in \SAE{SAE} a significant amount of \unimodal features act as classifiers for the opposite modality, which almost completely disappears in \SAEA{SAE-A}. This suggests that when training with $\mathcal{L}_{\mathrm{align}}$, features become more organized overall.

\begin{figure}[h]
    \centering
    \includegraphics[width=0.6\textwidth]{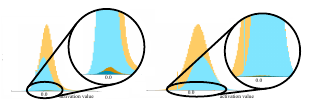}
    \caption{Distribution of $\mathrm{scores} = A \cdot D_i$ for a high energy \bimodal feature (\textbf{left}) and a low energy \bimodal feature (\textbf{right}). Blue: image scores $I \cdot D_i$. Yellow: text scores $T \cdot D_i$. Dark: subset of image-caption pairs where $D_i$ is active in at least one of them. Light: all other pairs. The implicit optimization for $p_{\mathrm{acc}}$ happens only in the dark regions.}
    \label{fig:app_pacc}
\end{figure}

\section{MP-SAE high energy "biases"}
\label{app:unimodal_biases}

In this section, we investigate high-energy \unimodal features, circled in \Cref{fig:scatter_high}. We find that they are all extremely close to one another, with an average cosine similarity above 0.9 (\Cref{fig:cosim_matrices}), never co-activate, and have a total frequency of 0.5. As they are \unimodal, this means that they are, in fact, small variations around a \unimodal bias. These results hold for all the models we tested, and for all the MP-SAEs we trained.

We find that, while some of these features do not seem to be trivially interpretable by looking at their most activating examples, some are surprisingly interpretable (\Cref{fig:visu_i_high,fig:visu_t_high}).

\Cref{fig:E_and_f} additionally shows the distribution of energy and frequency across features.

\begin{figure}[h]
    \centering
    
    \begin{subfigure}[b]{0.45\linewidth}
        \centering
        \includegraphics[width=\linewidth]{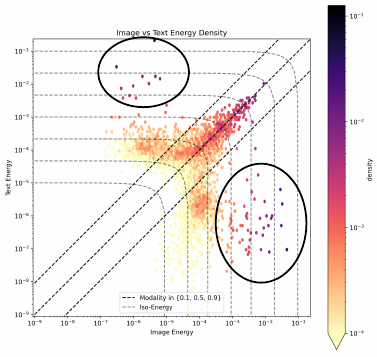}
        \caption{2D histogram of feature energy across both the image and text domains. We circled high-energy \unimodal features for the analysis of \cref{app:unimodal_biases}.}
        \label{fig:scatter_high}
    \end{subfigure}
    \hfill
    \begin{subfigure}[b]{0.45\linewidth}
        \centering
        \includegraphics[width=\linewidth]{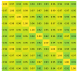}
        \caption{Matrix of pairwise cosine similarity between high-energy text-only features. Average cosine similarity: 0.9555 Total frequency: 0.461}
        \label{fig:cosim_matrices}
    \end{subfigure}
    \caption{\textbf{Left:} energy distribution across domain, \textbf{right:} pairwise cosine similarity between circled features.}
\end{figure}

\section{Feature Gallery}
\label{app:gallery}

\begin{figure}[h]
    \centering
    \includegraphics[width=0.5\linewidth]{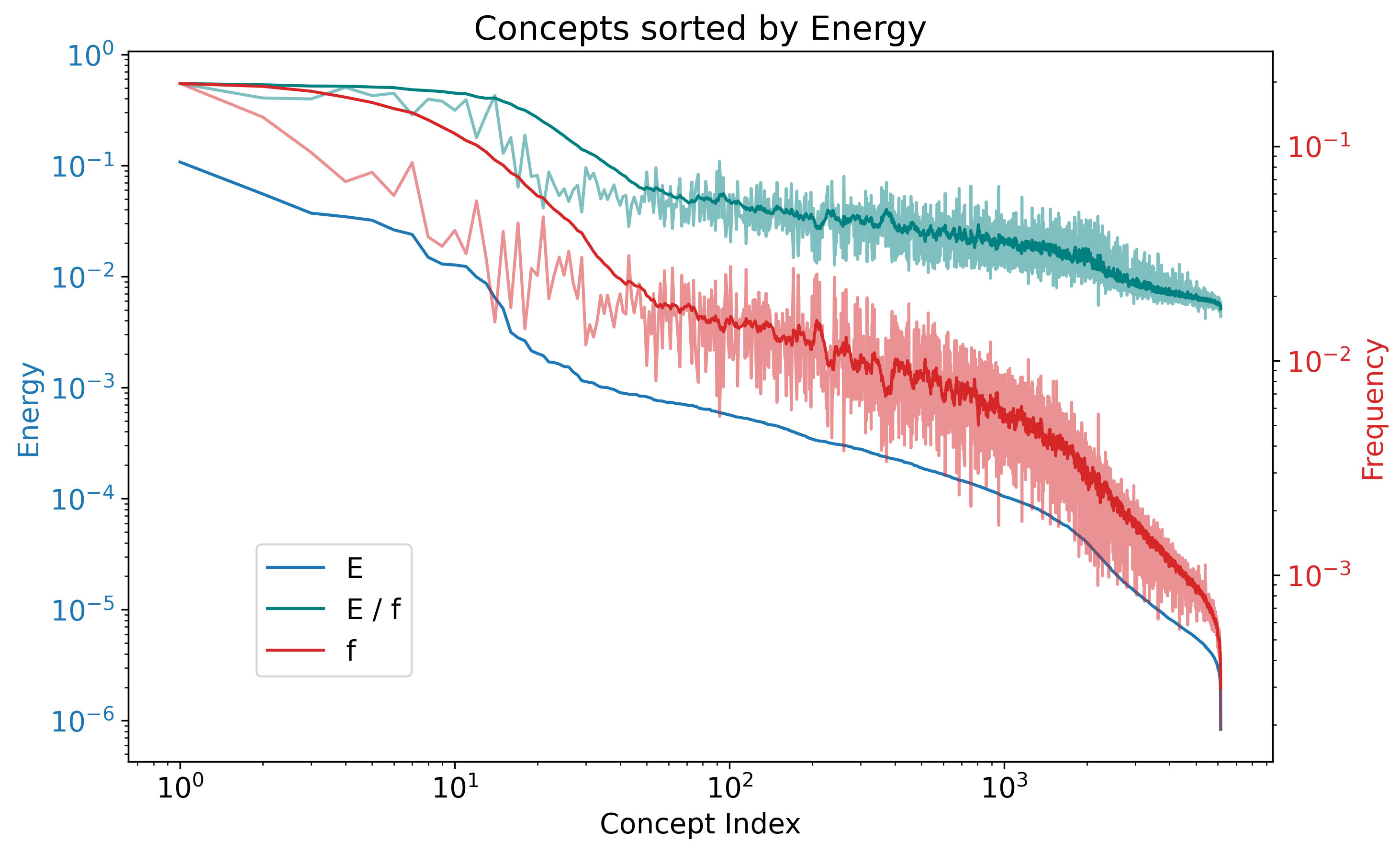}
    \caption{Energy and frequency distribution across features, sorted by energy. The $E/f$ curve gives an indication of the activation magnitude when the feature is active. Apart from the first few features, it appears that this quantity is roughly constant across all features.}
    \label{fig:E_and_f}
\end{figure}

\begin{figure}[h]
    \centering
    \includegraphics[width=1.\linewidth]{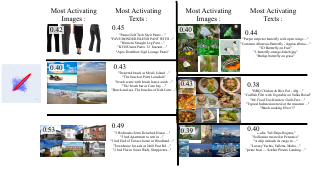}
    \caption{Examples of high and low-energy \bimodal features' most activating samples. For each feature, we indicate its average activation on the selected samples.}
    \label{fig:visu_bi_high}
\end{figure}

\begin{figure}[h]
    \centering
    \includegraphics[width=1.\linewidth]{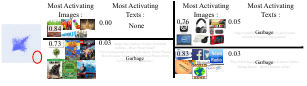}
    \caption{Examples of image-only features' most activating samples. For each feature, we indicate its average activation on the selected samples. We can see an a priori uninterpretable feature, then a feature for electronic devices, covers, and logos. These features are almost never active in texts, and when they are, they have very small activation and no apparent pattern.}
    \label{fig:visu_i_high}
\end{figure}

\begin{figure}[h]
    \centering
    \includegraphics[width=1.\linewidth]{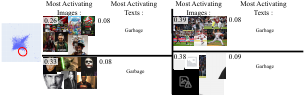}
    \caption{Examples of image-only features' most activating samples. For each feature, we indicate its average activation on the selected samples. We can see a feature for covers, sports, famous figures, and badly cropped images. These features are almost never active on texts, and when they are, it's with very small activation and without any apparent pattern.}
    \label{fig:visu_i_low}
\end{figure}

\begin{figure}[h]
    \centering
    \includegraphics[width=1.\linewidth]{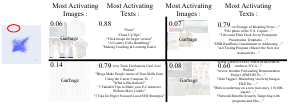}
    \caption{Examples of text-only features' most activating samples. For each feature, we indicate its average activation on the selected samples. We can see an a priori non-interpretable feature, then a feature for news, scam influencers, and news again. These features are almost never active on images, and when they are, it's with very small activation and without any apparent pattern.}
    \label{fig:visu_t_high}
\end{figure}

\begin{figure}[h]
    \centering
    \includegraphics[width=1.\linewidth]{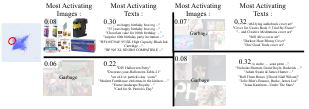}
    \caption{Examples of text-only features' most activating samples. For each feature, we indicate its average activation on the selected samples. We can see a feature for ink cartridges and birthdays, cover art, traditional holidays, and name-dropping. These features are almost never active on images, and when they are, they have very small activation and no apparent pattern, except for the cartridges and birthday one, which also activate with low energy on similar images.}
    \label{fig:visu_t_low}
\end{figure}

\clearpage

\section{Semantic vector arithmetic.}
\label{app:arithmetic}

After training SAEs with desirable properties, based on our hypotheses on the representations of VLM dual encoders, we now ask whether they can be used to improve their performance on downstream tasks. To do so, we select both zero-shot accuracy, measured on ImageNet in \cref{app:removing_gap} as part of the modality gap experiment, and retrieval, measured in this section on FashionIQ. See \cref{app:datasets} for details on this dataset.\\

\paragraph{Notations.} Let $Q \in R^{N\times d}$ be query vectors and $T \in \R^{N \times d}$ be target vectors. The task consists in maximizing average recall on batches $b$ of cosine similarity matrix $Q_{b} \cdot T_b^\top$.
$T$ are embeddings of target images, and $Q$ will involve $I_{\mathrm{src}}$, source image embeddings, and $\Delta$, embeddings of relative captions.
Let $Z_T$, $Z_{I_{\mathrm{src}}}$ and $Z_\Delta$ be the sparse codes of $T, I_{\mathrm{src}}$ and $\Delta$ given by our SAE.

Let $\widehat{T} \coloneqq Z_TD$ be the set of reconstructed activations, and $\widetilde{T} \coloneqq (Z_T \odot \bm{\delta})D$ be the set of \bimodal activations (recall that $D \in \R^{K\times d}$ is the dictionary and $\bm{\delta} \in \{0, 1\}^{K}$ is the binary mask selecting \bimodal features, and broadcasted to the size of the codes). Similarly, let $\widehat{I}_{\mathrm{src}}$, $\widetilde{I}_{\mathrm{src}}$, $\widehat{\Delta}$ and $\widetilde{\Delta}$ be reconstructions and \bimodal versions of $I_{\mathrm{src}}$ and $\Delta$.\\

\paragraph{Queries.} We define the following queries. First, we use $Q^\mathrm{baseline}_{\widehat{I}_{\mathrm{src}}}\coloneqq \widehat{I}_{\mathrm{src}}$ and $Q^\mathrm{baseline}_{\widehat{\Delta}} \coloneqq \widehat{\Delta}$ as lower bounds for any other queries. Then, let $Q \coloneqq \widehat{I}_{\mathrm{src}} + \widehat{\Delta}$ be the classical way of doing semantic vector arithmetic, and $Q_{\mathrm{SAE}} \coloneqq \widehat{I}_{\mathrm{src}} + \widetilde{\Delta}$ be arithmetic restricted on \bimodal features.

See \Cref{fig:retrieval_semantics} for a visual representation of all vectors involved. Note that queries are scaled to end up on the unit sphere on this representation, the actual result of the addition is indicated with a cross.

\paragraph{Retrieval.} We find that $Q_{\mathrm{SAE}}$ systematically outperforms $Q$ in terms of retrieval recall, across the six models we tested (\cref{tab:retrieval_recall}).

\paragraph{Out of Distribution.} We also measure the out of distribution (OOD) score of each of our four queries distribution to the target distribution. The idea is to measure whether our queries actually match targets or are just vectors that happen to be able to do retrieval.

We use the distance between $Q_i$ and it's K-th nearest neighbor in $\widehat{T}$--- similarly for $T_i$---then we find the optimal threshold on the distance to classify $Q$ vs $\widehat{T}$. \textbf{The final OOD score of Q is the accuracy of this classifier}. This approach is popular in the OOD literature \citep{sun2022out, yang2022openood}. We typically use $K = 10$. We repeat that procedure for $Q_{\mathrm{SAE}}$ and the two baselines.

We expect $Q^\mathrm{baseline}_{\widehat{I}_{\mathrm{src}}}$ to have the exact same distribution as $\widehat{T}$ and $Q^\mathrm{baseline}_{\widehat{\Delta}}$ to be extremely out of distribution, due to the modality gap. Also due to the modality gap, we expect $Q$ to be completely OOD. Intuitively, if you only consider \unimodal biases, $Q$ will have the mean of both. As biases are the strongest components of all embeddings, this means that $Q$ will end up in a cone disjoint from both the image and text embedding distribution. In practice, there are a lot of other \unimodal features that are not biases. Thus, by intervening only on shared representations, we should have $Q_\mathrm{SAE}$ end up in the same distribution as $\widehat{T}$.

Results are available in \Cref{fig:retrieval_OOD_appendix} and \cref{tab:retrieval_OOD_appendix}. We can see that, as expected, restricting the intervention to shared representations reduces drastically the difference in query and target distribution.

\begin{minipage}{0.32\linewidth}
    \centering
    \resizebox{\linewidth}{!}{
    \begin{tabular}{lc|c}
        \toprule
        OOD score ($\downarrow$) & $Q$ & $Q_\mathrm{SAE}$ \\
        \midrule
        \quad \textbf{CLIP} & 0.97 & \textbf{0.63} \\
        \quad \textbf{CLIP-L} & 0.95 & \textbf{0.76} \\
        \quad \textbf{OpenCLIP} & 0.86 & \textbf{0.68} \\
        \quad \textbf{OpenCLIP-L} & 0.87 & \textbf{0.72} \\
        \quad \textbf{SigLIP} & 0.99 & \textbf{0.70} \\
        \quad \textbf{SigLIP2} & 0.99 & \textbf{0.61} \\
        
        \bottomrule
    \end{tabular}}
    \captionof{table}{OOD score between the query distribution and target distribution. Across all six models, our concept based queries $Q_{\mathbf{SAE}}$ outperforms the classical queries $Q$.}
    \label{tab:retrieval_OOD_appendix}
\end{minipage}%
\hfill
\begin{minipage}{0.63\linewidth}
\centering
    \centering
    \includegraphics[width=0.98\linewidth]{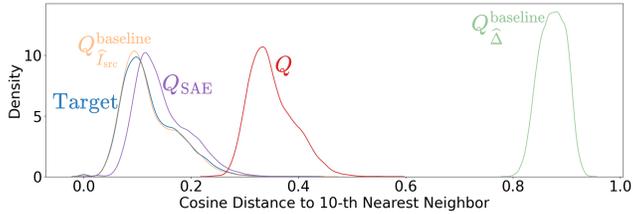}
    \captionof{figure}{Histograms of distances between queries and their 10th nearest neighbor in the target distribution. Our $Q_{\mathrm{SAE}}$ are close to the target distribution while classical $Q$ are not. $Q^\mathrm{baseline}_{\widehat{\Delta}}$ and $Q^\mathrm{baseline}_{\widehat{I}_{\mathrm{src}}}$ are, as expected, perfectly OOD due to the modality gap, and perfectly in distribution as source and target image distribution are the same.}
    \label{fig:retrieval_OOD_appendix}
\end{minipage}%

\begin{table}[H]
    \centering
    \resizebox{\linewidth}{!}{
    \begin{tabular}{lcccccc}
        \toprule
        $\frac{1}{2}(\mathrm{r@10}+\mathrm{r@50})~(\uparrow)$ & \textbf{CLIP} & \textbf{CLIP-L} & \textbf{OpenCLIP} & \textbf{OpenCLIP-L} & \textbf{SigLIP} & \textbf{SigLIP2} \\
        
        \midrule
        \quad $Q^\mathrm{baseline}_{\widehat{\Delta}}$ & 0.499 & 0.536 & 0.610 & 0.630 & 0.601 & 0.611 \\
        \quad $Q^\mathrm{baseline}_{\widehat{I}_{\mathrm{src}}}$ & 0.567 & 0.630 & 0.647 & 0.669 & 0.686 & 0.683 \\
        
        \midrule
        \quad $Q$ & 0.681 & 0.743 & 0.788 & 0.818 & 0.794 & 0.787 \\
        \quad $Q_\mathrm{SAE}$ & \textbf{0.683} & \textbf{0.750} & \textbf{0.790} & \textbf{0.820} & \textbf{0.812} & \textbf{0.798} \\
        
        \bottomrule
    \end{tabular}}
    \caption{Retrieval recall on FashionIQ for our two baselines and two queries. Across all six models, our concept based queries $Q_{\mathbf{SAE}}$ outperforms the classical queries $Q$.}
    \label{tab:retrieval_recall}
\end{table}

\section{Removing the modality gap}
\label{app:removing_gap}

\paragraph{}We investigate two methods and two baselines for removing the modality gap between $I$ and $T$. We measure the modality gap using the difference in mean (DiM) and the Wasserstein distance $\mathcal{W}$, and consider the performance in recall on LAION. We also include COCO and ImageNet to study the transferability of our methods on unseen data distributions, and on unseen tasks.

\paragraph{Notations.} Let's first recall some notations. Let $I$ and $T$ denote image and text embeddings, with respective mean $\mu_I$ and $\mu_T$. Let $Z_I$, $Z_T$ be the codes of $I$, $T$ given by our SAE. Let $D$ be the dictionary, and $\bm{\delta} \in \{0, 1\}^{K}$ be the binary mask selecting \bimodal features, and broadcasted to the size of the codes. 
Let $\widehat{I} \coloneqq Z_ID$ be the set of reconstructed activations, and $\widetilde{I} \coloneqq (Z_I \odot \bm{\delta})D$, and similarly for $\widehat{T}$ and $\widetilde{T}$.\\

\paragraph{}Our first method comes from the analysis of our SAEs : we remove \unimodal concepts, ending up with $\widetilde{I}$ and $\widetilde{T}$. This is the only method that doesn't enforce the mean of the two distribution to exactly match. Instead, we rely on our SAEs to capture all shared representations. As a baseline we propose to consider the centered versions of the data $I_{\mathrm{center}} \coloneqq \widehat{I} - \bm{\mu_{\widehat{I}}}$ and $T_{\mathrm{center}} \coloneqq \widehat{T} - \bm{\mu_{\widehat{T}}}$. The motivation for this baseline is the observation in \cref{app:unimodal_biases} that \unimodal features are predominantly small variations around a unique direction, except for small energy ones which are orders of magnitude less energetic.

For the above baseline and the next two methods, we use reconstructed embeddings instead of the original ones for fairness of comparison, as SAEs never perfectly reconstruct activations. Additionally, for the means, we use the means on the LAION dataset even on COCO and on ImageNet, also for fairness, as SAEs were trained on the LAION dataset.

\paragraph{}The second method is called \emph{embedding shift} in \citet{liang2022mind}, and it essentially consist in the following transformation : $I_{\mathrm{shift}} \coloneqq \widehat{I} - \bm{\mu_{\widehat{I}}} + \bm{\frac{\mu_{\widehat{I}} + \mu_{\widehat{T}}}{2}}$ and $T_{\mathrm{shift}} \coloneqq \widehat{T} - \bm{\mu_{\widehat{T}}} + \bm{\frac{\mu_{\widehat{I}} + \mu_{\widehat{T}}}{2}}$. We propose the following baseline for this method : we replace the average of the means by a random direction $\bm{r}$, sampled uniformly from the unit sphere---$I_{\mathrm{rand}} \coloneqq \widehat{I} - \bm{\mu_{\widehat{I}}} + \bm{r}$ and $T_{\mathrm{rand}} \coloneqq \widehat{T} - \bm{\mu_{\widehat{T}}} + \bm{r}$.

These two methods are not comparable to the first two in terms of modality gap, as even if the DiM is similar, the Wasserstein distance $\mathcal{W}$ behaves very differently. Indeed, in the first two methods, data is centered around $0$ and is roughly uniformly spread on the unit sphere. On the other hand, by manually adding a mean, these methods concentrate the data on a narrow cone, significantly reducing the $\mathcal{W}$. We illustrate these four methods on synthetic 3D points in \Cref{fig:toy_modality_gap}.

\begin{figure}[h]
    \centering
    \includegraphics[width=0.8\linewidth]{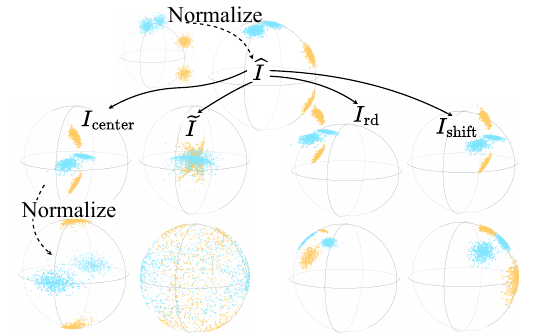}
    \caption{Illustration of the four methods to remove the modality gap on synthetic data points.}
    \label{fig:toy_modality_gap}
\end{figure}

\paragraph{} Results are compiled in \cref{tab:removing_gap_laion,tab:removing_gap_coco,tab:removing_gap_imagenet}. We can see that the shift method introduced in previous work is by far the worst, significantly reducing recall performances, then comes the random baseline. The mean baseline outperforms the concept-based approach in terms of recall, but it does not close the modality gap.

We also indicate in \Cref{fig:modality_gap_OOD} the KNN distance histograms for our four methods, following the OOD detection method described in \cref{app:arithmetic}, for the CLIP-L model on the LAION dataset. We take the image distribution as the reference one, and the text distribution as the "OOD" one. This illustrates a bit more precisely the effect of all methods on the distribution of embeddings. We can see that our method is the only one that closes the gap by merging the two distributions.

\begin{figure}[h]
    \centering
    \includegraphics[width=0.95\linewidth]{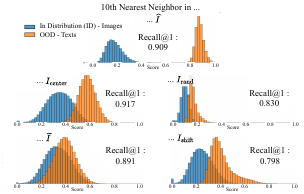}
    \caption{Histogram of distances from each image (ID) and caption (OOD) embedding to it's 10th nearest neighbor in the corresponding distribution of image embeddings. We use CLIP-L and the LAION dataset. As explained in \cref{app:removing_gap}, the shift method and our concept based method can't be compared in terms of distances, they can only be compared to their respective baselines. Indeed, distances in the random baseline and shift method are artificially small, but we can see from these histograms that the two distributions can still be very well separated. They are however comparable in histogram separation and recall, which we indicate. \textbf{Our concept based method is the only one to truly close the gap by matching the image and text distribution.}}
    \label{fig:modality_gap_OOD}
\end{figure}

\begin{table}[H]
    \centering
    \resizebox{\linewidth}{!}{
    \begin{tabular}{lcccccc|c}
        \toprule
         & \textbf{CLIP} & \textbf{CLIP-L} & \textbf{OpenCLIP} & \textbf{OpenCLIP-L} & \textbf{SigLIP} & \textbf{SigLIP2} & \textbf{$\Delta$ from $(\widehat{I}, \widehat{T})$} \\
        
        \midrule
        \multicolumn{6}{l}{$\mathrm{DiM}$} \\
        \cmidrule(lr){1-1}
        \quad $(\widehat{I}, \widehat{T})$ & 0.723 & 0.707 & 0.520 & 0.633 & 0.969 & 1.079 & / \\
        \quad $(I_{\mathrm{center}}, T_{\mathrm{center}})$ & 0.009 & 0.010 & 0.012 & 0.011 & 0.009 & 0.007 & -0.762 \\
        \quad $(\widetilde{I}, \widetilde{T})$ & 0.040 & 0.054 & 0.062 & 0.071 & 0.028 & 0.054 & -0.720 \\
        \quad $(I_{\mathrm{rand}}, I_{\mathrm{rand}})$ & 0.017 & 0.033 & 0.020 & 0.015 & 0.011 & 0.117 & -0.736 \\
        \quad $(I_{\mathrm{shift}}, I_{\mathrm{shift}})$ & 0.011 & 0.010 & 0.011 & 0.011 & 0.009 & 0.006 & -0.762 \\
        
        \midrule
        \multicolumn{6}{l}{$\mathcal{W}$} \\
        \cmidrule(lr){1-1}
        \quad $(\widehat{I}, \widehat{T})$ & 0.679 & 0.729 & 0.708 & 0.754 & 0.940 & 0.901 & / \\
        \quad $(I_{\mathrm{center}}, T_{\mathrm{center}})$ & 0.686 & 0.684 & 0.680 & 0.669 & 0.669 & 0.767 & -0.093 \\
        \quad $(\widetilde{I}, \widetilde{T})$ & 0.355 & 0.508 & 0.533 & 0.560 & 0.533 & 0.649 & -0.262 \\
        \quad $(I_{\mathrm{rand}}, I_{\mathrm{rand}})$ & 0.220 & 0.232 & 0.261 & 0.249 & 0.206 & 0.168 & -0.562 \\
        \quad $(I_{\mathrm{shift}}, I_{\mathrm{shift}})$ & 0.447 & 0.505 & 0.590 & 0.573 & 0.514 & 0.349 & -0.289 \\
        
        \midrule
        \multicolumn{6}{l}{$\mathrm{recall@1}$} \\
        \cmidrule(lr){1-1}
        \quad $(\widehat{I}, \widehat{T})$ & 0.882 & 0.909 & 0.952 & 0.964 & 0.907 & 0.307 & / \\
        \quad $(I_{\mathrm{center}}, T_{\mathrm{center}})$ & 0.898 & 0.917 & 0.951 & 0.966 & 0.921 & 0.330 & +0.010 \\
        \quad $(\widetilde{I}, \widetilde{T})$ & 0.760 & 0.891 & 0.933 & 0.964 & 0.901 & 0.313 & -0.026 \\
        \quad $(I_{\mathrm{rand}}, T_{\mathrm{rand}})$ & 0.745 & 0.830 & 0.924 & 0.943 & 0.856 & 0.251 & -0.062 \\
        \quad $(I_{\mathrm{shift}}, T_{\mathrm{shift}})$ & 0.704 & 0.798 & 0.913 & 0.935 & 0.827 & 0.238 & -0.084 \\
        \bottomrule
    \end{tabular}}
    \caption{Recall performance and modality gap in LAION's reconstructed embeddings, after the \emph{shift} method and our concept based method to remove the gap, and after two baselines adapted to each method. As explained in \cref{app:removing_gap}, the modality gap between the shift and concept based methods are not comparable, they can only be compared to their respective baselines. We can see that the shift method performs by far the worst in terms of recall. The concept based method seems to be consistently worse in terms of recall than it's baseline in terms of recall on smaller models, and equivalent in larger ones. However, in terms of modality gap, it is much better, with the mean baseline closing the gap in DiM but leaving it almost untouched in terms of distribution ($\mathcal{W}$).}
    \label{tab:removing_gap_laion}
\end{table}

\begin{table}[H]
    \centering
    \resizebox{\linewidth}{!}{
    \begin{tabular}{lcccccc|c}
        \toprule
         & \textbf{CLIP} & \textbf{CLIP-L} & \textbf{OpenCLIP} & \textbf{OpenCLIP-L} & \textbf{SigLIP} & \textbf{SigLIP2} & \textbf{$\Delta$ from $(\widehat{I}, \widehat{T})$} \\
        
        \midrule
        \multicolumn{6}{l}{$\mathrm{DiM}$} \\
        \cmidrule(lr){1-1}
        \quad $(\widehat{I}, \widehat{T})$ & 0.812 & 0.808 & 0.714 & 0.747 & 1.043 & 1.051 & / \\
        \quad $(I_{\mathrm{center}}, T_{\mathrm{center}})$ & 0.354 & 0.379 & 0.472 & 0.469 & 0.375 & 0.507 & -0.437 \\
        \quad $(\widetilde{I}, \widetilde{T})$ & 0.153 & 0.186 & 0.192 & 0.216 & 0.123 & 0.100 & -0.701 \\
        \quad $(I_{\mathrm{rand}}, T_{\mathrm{rand}})$ & 0.292 & 0.311 & 0.366 & 0.360 & 0.308 & 0.420 & -0.520 \\
        \quad $(I_{\mathrm{shift}}, T_{\mathrm{shift}})$ & 0.354 & 0.379 & 0.471 & 0.467 & 0.374 & 0.507 & -0.437 \\
        
        \midrule
        \multicolumn{6}{l}{$\mathcal{W}$} \\
        \cmidrule(lr){1-1}
        \quad $(\widehat{I}, \widehat{T})$ & 0.676 & 0.723 & 0.648 & 0.675 & 0.901 & 0.859 & / \\
        \quad $(I_{\mathrm{center}}, T_{\mathrm{center}})$ & 0.653 & 0.670 & 0.604 & 0.588 & 0.613 & 0.659 & -0.116 \\
        \quad $(\widetilde{I}, \widetilde{T})$ & 0.359 & 0.524 & 0.426 & 0.463 & 0.451 & 0.436 & -0.304 \\
        \quad $(I_{\mathrm{rand}}, T_{\mathrm{rand}})$ & 0.210 & 0.225 & 0.241 & 0.233 & 0.195 & 0.211 & -0.528 \\
        \quad $(I_{\mathrm{shift}}, T_{\mathrm{shift}})$ & 0.431 & 0.487 & 0.520 & 0.523 & 0.470 & 0.460 & -0.265 \\
        
        \midrule
        \multicolumn{6}{l}{$\mathrm{recall@1}$} \\
        \cmidrule(lr){1-1}
        \quad $(\widehat{I}, \widehat{T})$ & 0.534 & 0.566 & 0.617 & 0.678 & 0.636 & 0.636 & / \\
        \quad $(I_{\mathrm{center}}, T_{\mathrm{center}})$ & 0.551 & 0.585 & 0.612 & 0.667 & 0.650 & 0.628 & +0.004 \\
        \quad $(\widetilde{I}, \widetilde{T})$ & 0.430 & 0.536 & 0.575 & 0.633 & 0.581 & 0.600 & -0.052 \\
        \quad $(I_{\mathrm{rand}}, T_{\mathrm{rand}})$ & 0.442 & 0.496 & 0.557 & 0.617 & 0.552 & 0.410 & -0.099 \\
        \quad $(I_{\mathrm{shift}}, T_{\mathrm{shift}})$ & 0.420 & 0.485 & 0.554 & 0.603 & 0.542 & 0.395 & -0.111 \\
        \bottomrule
    \end{tabular}}
    \caption{Recall performance and modality gap in COCO's reconstructed embeddings, after the \emph{shift} method and our concept based method to remove the gap, and after two baselines adapted to each method. Compared to \cref{tab:removing_gap_laion}, this table tells us that our concept based method transfers much better than all other ones to a different data distribution. Indeed, the performance in recall are comparable with those on LAION, yet our concept based method is the only one that seems to transfer it's ability to close the modality gap.}
    \label{tab:removing_gap_coco}
\end{table}

\begin{table}[H]
    \centering
    \resizebox{\linewidth}{!}{
    \begin{tabular}{lcccccc|c}
        \toprule
         & \textbf{CLIP} & \textbf{CLIP-L} & \textbf{OpenCLIP} & \textbf{OpenCLIP-L} & \textbf{SigLIP} & \textbf{SigLIP2} & \textbf{$\Delta$ from $(\widehat{I}, \widehat{T})$} \\
        
        \midrule
        \multicolumn{6}{l}{$\mathrm{DiM}$} \\
        \cmidrule(lr){1-1}
        \quad $(\widehat{I}, \widehat{T})$ & 0.859 & 0.864 & 0.769 & 0.837 & 1.130 & 1.128 & / \\
        \quad $(I_{\mathrm{center}}, T_{\mathrm{center}})$ & 0.339 & 0.380 & 0.443 & 0.458 & 0.373 & 0.380 & -0.536 \\
        \quad $(\widetilde{I}, \widetilde{T})$ & 0.112 & 0.148 & 0.188 & 0.224 & 0.119 & 0.0881 & -0.785 \\
        \quad $(I_{\mathrm{rand}}, I_{\mathrm{rand}})$ & 0.288 & 0.317 & 0.361 & 0.459 & 0.321 & 0.340 & -0.583 \\
        \quad $(I_{\mathrm{shift}}, I_{\mathrm{shift}})$ & 0.340 & 0.381 & 0.443 & 0.365 & 0.373 & 0.381 & -0.551 \\
        
        \midrule
        \multicolumn{6}{l}{$\mathcal{W}$} \\
        \cmidrule(lr){1-1}
        \quad $(\widehat{I}, \widehat{T})$ & 0.684 & 0.726 & 0.666 & 0.687 & 0.931 & 0.893 & / \\
        \quad $(I_{\mathrm{center}}, T_{\mathrm{center}})$ & 0.686 & 0.685 & 0.651 & 0.616 & 0.655 & 0.679 & -0.103 \\
        \quad $(\widetilde{I}, \widetilde{T})$ & 0.362 & 0.534 & 0.462 & 0.468 & 0.467 & 0.467 & -0.304 \\
        \quad $(I_{\mathrm{rand}}, I_{\mathrm{rand}})$ & 0.197 & 0.210 & 0.237 & 0.219 & 0.181 & 0.176 & -0.561 \\
        \quad $(I_{\mathrm{shift}}, I_{\mathrm{shift}})$ & 0.397 & 0.443 & 0.490 & 0.465 & 0.419 & 0.358 & -0.336 \\
        
        \midrule
        \multicolumn{6}{l}{Zero shot accuracy} \\
        \cmidrule(lr){1-1}
        \quad $(\widehat{I}, \widehat{T})$ & 0.350 & 0.472 & 0.400 & 0.504 & 0.466 & 0.448 & / \\
        \quad $(I_{\mathrm{center}}, T_{\mathrm{center}})$ & 0.367 & 0.484 & 0.406 & 0.522 & 0.490 & 0.490 & +0.020 \\
        \quad $(\widetilde{I}, \widetilde{T})$ & 0.340 & 0.477 & 0.401 & 0.524 & 0.481 & 0.460 & +0.007 \\
        \quad $(I_{\mathrm{rand}}, I_{\mathrm{rand}})$ & 0.310 & 0.438 & 0.398 & 0.511 & 0.481 & 0.409 & -0.016 \\
        \quad $(I_{\mathrm{shift}}, I_{\mathrm{shift}})$ & 0.311 & 0.434 & 0.387 & 0.497 & 0.479 & 0.381 & -0.025 \\
        \bottomrule
    \end{tabular}}
    \caption{Zero-shot accuracy and modality gap in ImageNet's reconstructed embeddings, after the \emph{shift} method and our concept based method to remove the gap, and after two baselines adapted to each method. Our conclusion here is the same as the conclusion for \cref{tab:removing_gap_coco}.}
    \label{tab:removing_gap_imagenet}
\end{table}

\section{Datasets}
\label{app:datasets}

Throughout this study, we used four datasets~: LAION-400M \citep{schuhmann2021laion}, COCO \citep{coco2014}, ImageNet \citep{imagenet2009} and FashionIQ \citep{wu2021fashion}.

\paragraph{LAION.} We randomly selected a subset of 1M image-caption samples from the LAION-400M dataset. These samples were used to generate matching image-text embeddings. All reported SAEs were trained on this dataset.

\paragraph{COCO.} We used the train split of the 2017 update of the COCO dataset. Each image being paired with multiple captions, we generated pairs by duplicating the images for each of their captions. This dataset was used to further test the VLM encoders as well as the transferability of the learned features.

\paragraph{ImageNet.} This dataset was used to study the transferability of our features on a downstream classification task, commonly used in the literature. We generated about 80 captions for each class, embedded them, and averaged each class's embeddings to get a linear classifier, following \citet{radford2021}'s method.

\begin{figure}[h]
    \centering
    \includegraphics[width=0.8\linewidth]{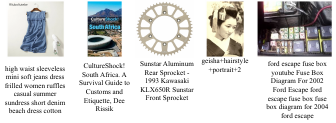}
    \includegraphics[width=0.8\linewidth]{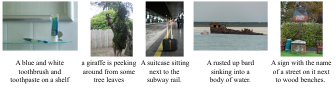}
    \includegraphics[width=0.8\linewidth]{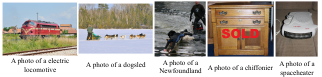}
    \caption{Examples of image-caption pairs in the LAION (\textbf{top}), COCO (\textbf{center}) and ImageNet (\textbf{bottom}).}
    \label{fig:dataset_examples}
\end{figure}

\paragraph{FashionIQ.} This dataset is used to evaluate retrieval performance. It contains source images, called candidate, target images, and \emph{relative} captions capturing the difference between the candidate and target. When several relative captions are available for a single candidate-target pair, we proceed as in ImageNet: we embed each caption and mean the results.

\begin{figure}[h]
    \centering
    \includegraphics[width=0.9\linewidth]{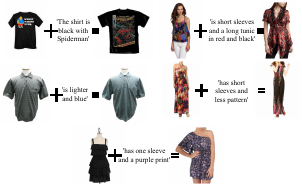}
    \caption{Examples of source and target images in FashionIQ, along with their relative caption.}
    \label{fig:fashionIQ_examples}
\end{figure}

\clearpage
\section{Why remove bias before comparing content}
\label{app:why_remove_bias}

One could ask why it is important to remove unimodal information and what impact this has on retrieval and other downstream tasks. Can it help improve performance by cleaning the representation and revealing the useful semantic information? Building on the fact that many works have found that modality-specific components live in orthogonal subspaces, we are left with two choices: either the modality information acts as a constant offset, in which case removing or adding it does not change rankings, or the modality information is adaptive, as we find in this work. In the adaptive case, retrieval rankings can change significantly, making our decomposition a useful tool for precise control and steering of retrieval systems.
We detail this here, and start with our setup.

\paragraph{Setup} 
We work in a representation space $\mathbb{R}^d$ that decomposes orthogonally as $\mathbb{R}^d = \C \oplus \M$ into a content subspace $\C$ and a modality subspace $\M$ with $\C \perp \M$. For any underlying item $\x$, the observed embedding factorizes as
$$\bm{v} = \underbrace{\bm{\m}(\x)}_{\text{modality } \in \M} + \underbrace{\bm{\c}(\x)}_{\text{content } \in \C} \quad \text{with} \quad \mathbb{R}^d = \C \oplus \M$$
where $\bm{\m}(\x) \in \M$ captures modality-specific information and $\bm{\c}(\x) \in \C$ represents the semantic content estimated by our SAE. The orthogonal direct sum decomposition ensures that content and modality components are geometrically independent.

For retrieval, we score against a set of unit-normalized candidates $\{\bm{y}_i\}_{i=1}^n$ using cosine similarity. To analyze pairwise rankings, we define the content margin $\Delta_c(i,j) \coloneqq \langle \bm{\c}(\x), \bm{y}_i - \bm{y}_j \rangle$ and the modality margin $\Delta_m(i,j) \coloneqq \langle \bm{\m}(\x), \bm{y}_i - \bm{y}_j \rangle$. We also write $\alpha_i \coloneqq \langle \bm{\m}(\x), \bm{y}_i \rangle$ for the modality projection onto candidate $i$.

\subsection{Constant modality information}
We first consider the case where modality information affects all candidates equally. This corresponds to the intuitive scenario where modality acts as a simple offset that shifts all similarities by the same amount, leaving relative rankings unchanged.

When the modality component $\bm{\m}(\x)$ projects with identical strength onto every candidate in the retrieval set, removing this component preserves all pairwise orderings. This formalizes the intuition that a candidate-independent offset cannot affect relative rankings.

\begin{proposition}[Ranking invariance under constant projection]
\label{prop:constant-modality-invariance}
Assume unit normalized candidates $\{\bm{y}_i\}_{i=1}^n$. If there exists $\alpha \in \mathbb{R}$ such that $\langle \bm{\m}(\x), \bm{y}_i \rangle = \alpha$ for all $i$, then for any $i$ and $j$
$$\cos(\bm{v}, \bm{y}_i) > \cos(\bm{v}, \bm{y}_j)
\quad \iff \quad
\cos(\bm{\c}(\x), \bm{y}_i) > \cos(\bm{\c}(\x), \bm{y}_j)$$
The same equivalence holds with equality. The statement also holds for dot product scoring.
\end{proposition}

\begin{proof}
With $\|\bm{y}_i\| = 1$ we have
$$\cos(\bm{v}, \bm{y}_i) = \frac{\langle \bm{v}, \bm{y}_i \rangle}{\|\bm{v}\|}
= \frac{\langle \bm{\m}(\x), \bm{y}_i \rangle + \langle \bm{\c}(\x), \bm{y}_i \rangle}{\|\bm{v}\|}$$

Under the hypothesis $\langle \bm{\m}(\x), \bm{y}_i \rangle = \alpha$ for all $i$ we get
$$\cos(\bm{v}, \bm{y}_i) - \cos(\bm{v}, \bm{y}_j)
= \frac{\alpha - \alpha}{\|\bm{v}\|} + \frac{\langle \bm{\c}(\x), \bm{y}_i \rangle - \langle \bm{\c}(\x), \bm{y}_j \rangle}{\|\bm{v}\|}
= \frac{\Delta_c(i,j)}{\|\bm{v}\|}$$

Hence $\cos(\bm{v}, \bm{y}_i) > \cos(\bm{v}, \bm{y}_j)$ if and only if $\Delta_c(i,j) > 0$. We also have
$$\cos(\bm{\c}(\x), \bm{y}_i) - \cos(\bm{\c}(\x), \bm{y}_j) = \frac{\Delta_c(i,j)}{\|\bm{\c}(\x)\|}$$

The denominators $\|\bm{v}\|$ and $\|\bm{\c}(\x)\|$ are positive constants independent of $i$ and $j$, so both differences have the same sign. If $\|\bm{\c}(\x)\| = 0$ then $\bm{\c}(\x) = \bm{0}$ and all candidates tie under content while the first difference is zero as well, so rankings are trivially invariant. The dot product case is identical without denominators.
\end{proof}

In the constant modality regime, removing modality information yields only the subspace of interest, and potentially more interpretable multimodal embeddings, but without any performance gains or losses. 
Thus, in that case, the decomposition provides the content representation by eliminating modality-specific artifacts, but since rankings remain identical, one should not expect improvements in retrieval metrics. This case represents a neutral scenario where the primary benefit is representational clarity rather than functional enhancement.

We now consider a near-constant case where modality projections vary slightly across candidates. This provides a useful practical criterion for when small variations cannot disrupt content-based rankings.

\begin{corollary}[Approximate invariance under bounded spread]
\label{cor:bounded-spread}
Assume unit normalized candidates. Let $\varepsilon \coloneqq \max_{i,j} |\alpha_i - \alpha_j|$. Then for any pair $(i,j)$ with $|\Delta_c(i,j)| > \varepsilon$, the inequality $\cos(\bm{v}, \bm{y}_i) > \cos(\bm{v}, \bm{y}_j)$ holds if and only if $\cos(\bm{\c}(\x), \bm{y}_i) > \cos(\bm{\c}(\x), \bm{y}_j)$.
\end{corollary}

\begin{proof}
We have
$$\cos(\bm{v}, \bm{y}_i) - \cos(\bm{v}, \bm{y}_j)
= \frac{\Delta_c(i,j) + \Delta_m(i,j)}{\|\bm{v}\|}
\quad \text{with} \quad
\Delta_m(i,j) = \alpha_i - \alpha_j$$

If $|\Delta_c(i,j)| > \varepsilon \geq |\Delta_m(i,j)|$ then $\Delta_c(i,j)$ and $\Delta_c(i,j) + \Delta_m(i,j)$ have the same sign. Division by the positive constant $\|\bm{v}\|$ preserves sign. The sign also matches the sign of $\Delta_c(i,j)$ which determines the sign of $\cos(\bm{\c}(\x), \bm{y}_i) - \cos(\bm{\c}(\x), \bm{y}_j)$.
\end{proof}

This corollary establishes a robustness condition: when modality variations are small relative to content differences—even in the presence of multiple adaptive modality components -- the content signal prevails and rankings remain unchanged. Like the constant case, decomposition offers no performance gains here since content already drives the original rankings.

However, as we will show next and as we empirically find in this work, real modality information is typically adaptive rather than constant, leading to significant ranking changes when removed.

\subsection{Adaptive modality information}
This is what we observe empirically in this work. The modality signal decomposes into multiple unimodal atoms and the modality projection varies significantly across candidates. In this regime, removing $\bm{\m}(\x)$ can dramatically change rankings. The next proposition provides a necessary and sufficient condition for ranking flips and clarifies that flips occur when the candidate-dependent modality margin opposes and dominates the content margin.

\begin{proposition}[Characterization of ranking flips under adaptive modality]
\label{prop:adaptive-modality-flip}
Assume unit normalized candidates. For any pair $(i,j)$, the two differences
$$\cos(\bm{v}, \bm{y}_i) - \cos(\bm{v}, \bm{y}_j)
\quad \text{and} \quad
\cos(\bm{\c}(\x), \bm{y}_i) - \cos(\bm{\c}(\x), \bm{y}_j)$$
have opposite signs if and only if
$$\Delta_c(i,j) \big( \Delta_c(i,j) + \Delta_m(i,j) \big) < 0$$
Equivalently, the signs of $\Delta_m(i,j)$ and $\Delta_c(i,j)$ are opposite and $|\Delta_m(i,j)| > |\Delta_c(i,j)|$. The same characterization holds for dot product scoring.
\end{proposition}

\begin{proof}
Using the identities above we write
$$\cos(\bm{v}, \bm{y}_i) - \cos(\bm{v}, \bm{y}_j)
= \frac{\Delta_c(i,j) + \Delta_m(i,j)}{\|\bm{v}\|}
\quad \text{and} \quad
\cos(\bm{\c}(\x), \bm{y}_i) - \cos(\bm{\c}(\x), \bm{y}_j)
= \frac{\Delta_c(i,j)}{\|\bm{\c}(\x)\|}$$

The denominators are positive constants independent of the candidates. Therefore the two differences have opposite signs if and only if $\Delta_c(i,j)$ and $\Delta_c(i,j) + \Delta_m(i,j)$ have opposite signs. This is equivalent to the strict inequality $\Delta_c(i,j) \big( \Delta_c(i,j) + \Delta_m(i,j) \big) < 0$. Expanding gives the equivalent condition that $\Delta_m(i,j)$ and $\Delta_c(i,j)$ have opposite signs and that the magnitude of $\Delta_m(i,j)$ exceeds the magnitude of $\Delta_c(i,j)$. The dot product case follows by the same steps with no denominators.
\end{proof}

In practice, visual information or bias is almost orthogonal to textual information or bias. This makes $\Delta_m(i,j) \approx 0$, and therefore provides a guarantee that rankings should stay the same under sole manipulation of modality-specific information.

\begin{remark}[Concrete two-dimensional example]
Let $\M = \mathrm{span}\{\bm{e}_1\}$ and $\C = \mathrm{span}\{\bm{e}_2\}$ with $\bm{e}_1$ orthogonal to $\bm{e}_2$. Take $\bm{\m}(\x) = \bm{e}_1$ and $\bm{\c}(\x) = \tfrac{1}{2}\bm{e}_2$. Choose candidates $\bm{y}_1 = \tfrac{1}{\sqrt{2}}(\bm{e}_1 + \bm{e}_2)$ and $\bm{y}_2 = \tfrac{1}{\sqrt{2}}(-\bm{e}_1 + \bm{e}_2)$ which are unit normalized. Then $\Delta_c(1,2) = 0$ and $\Delta_m(1,2) = \sqrt{2}$ so $\cos(\bm{v}, \bm{y}_1) > \cos(\bm{v}, \bm{y}_2)$ while $\cos(\bm{\c}(\x), \bm{y}_1) = \cos(\bm{\c}(\x), \bm{y}_2)$. If we perturb $\bm{\c}(\x)$ slightly toward $-\bm{e}_2$ so that $\Delta_c(1,2) < 0$, the inequality reverses under content while it stays the same under observed scores, demonstrating a strict ranking flip.
\end{remark}

In summary, if modality information were candidate-independent, removing it would leave retrieval unchanged as shown in ~\cref{prop:constant-modality-invariance}. In practice, however, modality information is multi-component and adaptive. Its projections onto candidates differ significantly, and when those candidate-dependent differences oppose and dominate the content margin, rankings change dramatically as characterized in ~\cref{prop:adaptive-modality-flip}. This is precisely why we identify and remove unimodal atoms before comparing content -- doing so can reveal the true content-based similarities that were previously masked by conflicting modality signals. The bounded spread criterion in ~\cref{cor:bounded-spread} provides a practical stability guarantee, ensuring that decomposition is beneficial primarily when modality variations are large relative to content margins.

\section{On previous work's attempt to characterise $\Omega_I\oplus\Omega_T$}
\label{app:ilestsinuljpp}

We compare the method to remove the modality gap proposed by \citet{schrodi2025two} to ours and to three baselines in its capacity to close the modality gap. We report measures the modality gap as (\emph{i}) $\norm{\Delta}$ the norm of the difference in means, (\emph{ii}) $\mathrm{sep}$ the accuracy score of a linear logistic regression probe trained to classify text vs image embeddings, measuring the linear separability of both distribution, (\emph{iii}) RMG, a metric introduced by \citet{schrodi2025two} and (\emph{iv}) our OOD score.

Let $I, T \in \R^{N\times d}$ be image-caption pairs embedded using CLIP. Let $\mu_{I} \coloneqq \E[I_i]$, $\mu_{T} \coloneqq \E[T_i]$ and $\Delta \coloneqq \mu_T-\mu_T$. Let $e_1, ..., e_k$, $k\ll d$ be the directions identified by \citet{schrodi2025two}.

\textbf{Baseline A:} center embeddings modality-wise, $I\gets I-\mu_I$, similarly for $T$. \textbf{Baseline B:} project out modality wise means, $I \gets I-(I\cdot\mu_I) \mu_I$, with $\mu_I$ being additionally normalised, and similarly for $T$. \textbf{Baseline C:} project out $\Delta$, $I\gets I-(I\cdot\Delta) \Delta$, with $\Delta$ being normalized, and similarly for $T$. \textbf{Method:} project $I$ and $T$ on the orthogonal complement of $\mathrm{Span}(e_1, ..., e_k)$. We also compare to our method and to the original embedding (line "unchanged" in \cref{tab:comparison_schrodi}).

\begin{table}[h]
\centering
\begin{tabular}{lcccc}
\toprule
\textbf{Method} & $\|\Delta\|$ ($\downarrow$) & $\mathrm{sep}$ ($\downarrow$) & RMG ($\downarrow$) & OOD ($\downarrow$) \\
\midrule
Unchanged & 0.72 & 1.00 & 0.52 & 1.00 \\
\midrule
Baseline A & 0.00 & 0.50 & 0.45 & 0.99 \\
Baseline B & 0.00 & 0.50 & 0.45 & 0.98 \\
Baseline C & 0.00 & 0.50 & 0.45 & 0.89 \\
\midrule
Schrodi et al. & 0.22 & 1.00 & 0.46 & 0.98 \\
Ours & 0.01 & 0.50 & 0.41 & 0.63 \\
\bottomrule
\end{tabular}
\caption{Comparison of 2 methods and 3 baselines in their ability to remove the modality gap, as measured by 4 metrics.}
\label{tab:comparison_schrodi}
\end{table}

\end{document}